\documentclass[11pt]{article}
\usepackage[utf8]{inputenc}
\usepackage[T1]{fontenc}
\usepackage[margin=1in]{geometry}
\usepackage{times}

\usepackage{amsmath,amssymb,amsthm,mathtools}
\usepackage{booktabs,tabularx,float,microtype,xcolor,needspace}
\usepackage{graphicx,rotating,caption,subcaption,enumitem}

\usepackage[authoryear,square]{natbib}
\definecolor{DarkRed}{rgb}{0.368,0.097,0.078}
\definecolor{DarkBlue}{rgb}{0,0,0.5}

\usepackage[
    colorlinks=true,
    linkcolor=DarkBlue,
    citecolor=DarkBlue,
    urlcolor=DarkRed
]{hyperref}

\usepackage[nameinlink,capitalize,noabbrev]{cleveref}

\newtheorem{theorem}{Theorem}
\newtheorem{lemma}[theorem]{Lemma}
\newtheorem{corollary}[theorem]{Corollary}
\newtheorem{proposition}[theorem]{Proposition}
\theoremstyle{definition}
\newtheorem{assumption}[theorem]{Assumption}
\newtheorem{definition}[theorem]{Definition}

\crefname{assumption}{Assumption}{Assumptions}
\Crefname{assumption}{Assumption}{Assumptions}

\newcommand{\A}{\mathcal A}
\newcommand{\C}{\mathcal C}
\newcommand{\D}{\mathcal D}
\newcommand{\F}{\mathcal F}
\newcommand{\E}{\mathbb E}
\newcommand{\Pp}{\mathbb P}
\newcommand{\Reg}{\overline R_T}

\newcommand{\algname}[1]{\texttt{#1}}

\DeclareMathOperator{\Err}{Err}

\newfloat{algorithm}{tbp}{loa}
\floatname{algorithm}{Algorithm}
\floatstyle{ruled}
\restylefloat{algorithm}
\crefname{algorithm}{Algorithm}{Algorithms}
\Crefname{algorithm}{Algorithm}{Algorithms}
\newcounter{AlgoLine}
\newenvironment{algorithmic}{\begin{list}{\arabic{AlgoLine}:}{\usecounter{AlgoLine}
\setlength{\leftmargin}{2em}\setlength{\labelwidth}{1.5em}\setlength{\labelsep}{.5em}
\setlength{\itemsep}{.15em}\setlength{\parsep}{0pt}}}{\end{list}}
\newcommand{\State}{\item}
\newcommand{\AlgIndent}{\hspace*{1em}}

\title{Vanilla Policy Optimization Is Both Optimal and Differentially Private for Stochastic Contextual Bandits}

\author{%
Idan Attias\textsuperscript{*,1,2}\qquad
Orin Levy\textsuperscript{*,3}\qquad
Alexander Ryabchenko\textsuperscript{*,4,5}\\[0.6ex]
Yishay Mansour\textsuperscript{3,6}\qquad
Uri Stemmer\textsuperscript{3,6}
}

\begin{document}
\maketitle

\begingroup
\begin{NoHyper}
\renewcommand{\thefootnote}{}
\footnotetext{%
\textsuperscript{*}Equal contribution; alphabetical order.\\
{
\textsuperscript{1}University of Illinois Chicago;
\textsuperscript{2}Toyota Technological Institute at Chicago;
\textsuperscript{3}Tel Aviv University;
\textsuperscript{4}University of Toronto;
\textsuperscript{5}Vector Institute;
\textsuperscript{6}Google Research.
}
}
\end{NoHyper}
\endgroup

\begin{abstract}
Can vanilla policy optimization explore enough to achieve near-optimal
regret in stochastic contextual bandits?
We show that standard exponential policy updates driven by offline
regression do so under realizability, without exploration bonuses or
importance weighting. For $A$ actions, $T$ rounds, and a finite
prediction class $\mathcal{F}$, vanilla PO achieves
$\widetilde O(\sqrt{AT\log(|\mathcal{F}|)})$ regret with high probability. Our analysis reveals an implicit exploration mechanism of independent
interest: gradual policy updates prevent actions from losing probability
too quickly, allowing the regression oracle to learn their expected
losses.
We further develop a batched version using only $O(\log T)$ regression calls and
policy switches, and show how private regression oracles yield
differentially private contextual bandit algorithms without composition
across batches. For a finite class, this gives pure
$\varepsilon_{\rm priv}$-DP and regret
$\widetilde O\left(
\sqrt{AT
\log(|\mathcal{F}|/\delta)}(1+\varepsilon_{\rm priv}^{-1/2})
\right)$.
Finally, experiments across oracle-based contextual bandit algorithms,
with and without privacy, demonstrate the practical effectiveness of
policy optimization and the value of explicit exploration under stronger
privacy constraints.
\end{abstract}


\section{Introduction }
\label{sec:introduction}

Stochastic contextual multi-armed bandits (CMAB) are a fundamental
and extensively studied model of sequential decision making under
partial feedback \citep{agarwal2012contextual,agarwal2014,foster2018regression,falcon,levy2026}.
At each round, a learner observes a context, selects an action, and
receives feedback only for that action. The goal is to achieve low
cumulative loss while learning from this limited feedback. In the
stochastic setting, contexts and their associated action losses are
drawn independently across rounds from a fixed joint distribution.

Unlike in ordinary multi-armed bandits, the best action can vary with
the context, so the learner must learn a policy mapping contexts to
action distributions. Viewed as a reinforcement-learning problem with
one decision per episode, CMAB retains the challenge of a potentially
enormous or infinite state space. Since contexts may rarely repeat,
learning separately at each context is ineffective and the learner must
generalize across them. This makes CMAB substantially richer than
learning which of a fixed set of arms has the lowest mean loss.

This contextual information is also central to the model's practical
appeal, particularly for personalization. A context can describe a
user's preferences, the available items, or the current task, allowing
the learner to adapt its decision to the situation. Personalized news
recommendation is a prominent example: user and article features guide
which article to display, while feedback is observed only for the
displayed article \citep{li2010contextual}. More broadly, contextual
bandits provide a practical framework for adaptive decision systems
\citep{agarwal2017decisions,bietti2021bakeoff}.

A central goal is to develop computationally efficient algorithms with strong statistical guarantees. Oracle-based reductions address this goal through offline supervised learning \citep{agarwal2014}, particularly regression \citep{foster2018regression,falcon,xu2020uccb}. Related approaches reduce contextual bandit learning to online regression \citep{foster2020squarecb}.

Policy optimization (PO) is a widely used approach to reinforcement
learning, with practical successes exemplified by proximal policy
optimization \citep{schulman2017ppo}. Its basic idea is to improve a
policy directly by increasing the probability of actions predicted
to perform well. Exponential policy updates also have a
well-established theoretical foundation
\citep{abbasi2019politex,agarwal2021theory}.
Recently, \citet{levy2026} showed that PO equipped with
an exploration bonus achieves near-optimal regret for stochastic CMAB
under realizability and offline regression assumptions.
Here, function approximation provides the structure needed to
generalize across contexts: the expected loss belongs to a known
function class, and an offline regression procedure fits a predictor
from the observed feedback \citep{foster2018regression,falcon}.

The exploration bonus encourages actions for which the learner has
insufficient information, but computing it can be costly. In the method
of \citet{levy2026}, evaluating the bonus at a new context requires
reconstructing historical action probabilities from past predictors,
which can become prohibitively expensive over long horizons. This
raises the question of whether explicit bonuses are necessary at all:
perhaps vanilla exponential updates, which simply favor actions with
lower predicted losses, already provide sufficient exploration.

Returning to the personalization motivation, an effective learning algorithm must also protect the people whose data it uses. Contexts and feedback can contain sensitive information whose influence may persist in later predictors and policies.
Differential privacy provides a formal way to limit this influence
\citep{dwork2006calibrating,dwork-roth-2014}, and has been studied in contextual linear
bandits \citep{shariff2018private,chen2025private} and episodic
reinforcement learning \citep{vietri2020private}.
Closely related work studies the sample complexity of private PO,
including methods based on private regression
\citep{he2025privatepo}.
Recent work also establishes regret guarantees for private
reinforcement learning with general function approximation
\citep{he2026general}.
We discuss how our approach differs from these works in its
feedback model, regression requirements, and theoretical
guarantees in \Cref{app:related-work}.

Our goal is to design efficient contextual bandit algorithms that simultaneously achieve strong regret guarantees, support general function approximation, and protect the data on which they learn. We seek algorithms that access the function class through a suitable private regression oracle. Vanilla policy optimization is a particularly appealing candidate for this approach: it is simple, computationally practical, and broadly compatible with modern regression methods.

Together, these considerations lead to our central questions:
\emph{(1) Can vanilla policy optimization achieve near-optimal regret
in stochastic contextual bandits without explicit exploration bonuses?
(2) Can it be made differentially private through private offline
regression while retaining near-optimal regret?}

\paragraph{Our contributions.}
\begin{enumerate}

\item \textbf{Near-optimal regret for vanilla PO (\Cref{sec:algorithm-results}).}
For $A$ actions, $T$ rounds, and a finite realizable function class
$\F$, vanilla PO with least-squares regression achieves
$\widetilde O(\sqrt{AT\log(|\F|/\delta)})$ regret with probability
at least $1-\delta$. Our analysis reveals an implicit exploration
mechanism: gradual updates preserve enough historical action probability
that persistent suboptimal actions require substantial squared prediction
error, which is controlled by the regression guarantee.

\item \textbf{Batched and private PO
(\Cref{sec:batched-implementation,sec:private}).}
We develop batched versions of vanilla PO and the bonus-based
method of \citet{levy2026}, retaining near-optimal regret with
only $O(\log T)$ regression calls.
Replacing regression by a private oracle protects the complete
sequence of released predictors and policy maps, without a
composition factor in the number of epochs.
For finite realizable $\F$, the exponential mechanism yields
pure $\varepsilon_{\rm priv}$-DP and regret
$\widetilde O\!\left(
\sqrt{AT\log(|\F|/\delta)}
(1+\varepsilon_{\rm priv}^{-1/2})
\right)$
with probability at least $1-\delta$.
More generally, our guarantees accommodate realizable function
classes through a suitable private regression oracle, beyond
the linear setting studied under joint differential privacy
by \citet{shariff2018private}.

\item \textbf{Private instantiations of other CMAB algorithms
(\Cref{sec:private}).}
The same privacy argument applies to other oracle-based CMAB
algorithms that collect each batch under a fixed policy and use
it in only one regression fit.
When subsequent policies depend only on fitted predictors and
public parameters, replacing the regression oracle with a private
one yields private counterparts while preserving the policy
construction.
We instantiate and evaluate these additional private algorithms
in our experiments.

\item \textbf{Empirical evaluation (\Cref{sec:experiments}).}
We compare vanilla PO, exploration-enhanced PO, and several private and
non-private oracle-based CMAB algorithms under linear and logistic regression.
The experiments show that the privatization scheme is broadly effective,
that vanilla PO is a simple and competitive practical method, and that
additional exploration can improve the privacy--utility tradeoff.
\end{enumerate}

\section{Contextual multi-armed bandit preliminaries}
\label{sec:setup}

We consider a stochastic contextual bandit problem
\citep{agarwal2012contextual,falcon} with context space $\C$
and a finite action set $\A$ of size $A\ge2$.
A policy $\pi:\C\to\Delta(\A)$ maps each context to a distribution
over actions. Let $\pi(c,a)$ denote the probability that $\pi(c)$
plays action $a$.
The stochastic loss of action $a$ for context $c$ in round $t$ is denoted
by $\ell_t(c,a)\in[0,1]$.
We assume that the context--loss pairs
$(c_t,\ell_t(c_t,\cdot))$ are drawn i.i.d. from a fixed unknown joint
distribution, independently of the learner's randomness.
We write $\D$ for its marginal distribution over contexts.

Let $\mathcal H_{t-1}$ denote the history of previous contexts,
actions, and observed losses, together with any random choices
already used by the learner. At each round $t\le T$, the learner:
(i) chooses a policy $\pi_t$ based on $\mathcal H_{t-1}$;
(ii) observes a fresh context $c_t\sim\D$;
(iii) samples $a_t\sim\pi_t(c_t,\cdot)$; and
(iv) observes only the loss $\ell_t(c_t,a_t)$ of the selected action.

Define the expected loss function by
$f^\star(c,a)=\E[\ell_t(c_t,a)\mid c_t=c]$,
for $c\in\C$ and $a\in\A$.
This function does not depend on $t$, since the joint distribution
is the same in every round.
Let $\F\subseteq[0,1]^{\C\times\A}$ be a known class of predictors.

\begin{assumption}[Realizability]
\label{ass:realizability}
The expected loss function belongs to the function class:
$f^\star\in\F$.
\end{assumption}
Let $a^\star(c)\in\arg\min_{a\in\A}f^\star(c,a)$, with ties broken
by a fixed ordering, and define the action gap
$\Delta(c,a)=f^\star(c,a)-f^\star(c,a^\star(c))$.
The expected one-round regret of a policy $\pi$ is denoted by 
$\mathcal R(\pi)
=\E_{c\sim\D}\bigl[
\sum_{a\in\A}\pi(c,a)\Delta(c,a)
\bigr]$.
We measure performance by the cumulative regret
\begin{equation}
 \Reg
 =\sum_{t=1}^T\sum_{a\in\A}
   \pi_t(c_t,a)\Delta(c_t,a)
 =\sum_{t=1}^T
   \mathop{\E}\limits_{a_t\sim\pi_t(c_t,\cdot)}
   \left[
     \Delta(c_t,a_t)\mid\mathcal H_{t-1},c_t
   \right].
 \label{eq:pseudo-regret}
\end{equation}
The expectation averages over the current action, with the
history and current context held fixed.
The observed contexts and learned policies remain random.
Our analyses first bound $\sum_{t=1}^T\mathcal R(\pi_t)$,
which averages each learned policy's regret over the context
distribution $\D$.
The concentration argument in \cref{app:concentration}
then yields high-probability bounds on $\Reg$, evaluated
at the observed contexts.

\paragraph{Offline regression oracle.}
The learner estimates $f^\star$ using an offline regression oracle.
For a policy $\pi$, define the mean squared prediction error
$
 \|f-f^\star\|_{\pi}^2
 =\E_{c\sim\D}\!\left[
   \sum_{a\in\A}\pi(c,a)
   \bigl(f(c,a)-f^\star(c,a)\bigr)^2
 \right].
$
When the predictor is refitted after each round using all
observations collected so far, we require the following accuracy
guarantee under the policies used to collect those observations.

\begin{assumption}[Regression-oracle accuracy]
\label{ass:oracle}
Let $\widehat f_t:\C\times\A\to[0,1]$ be the predictor returned
by the oracle after round $t$.
For every $\delta\in(0,1)$, with probability at least $1-\delta$,
$\sum_{s=1}^t\|\widehat f_t-f^\star\|_{\pi_s}^2
\le\Err_{\rm orc}(T,\delta)$
simultaneously for all $t\le T$,
where $\Err_{\rm orc}(T,\delta)$ is a known error bound fixed
before the interaction.
\end{assumption}
Equivalently, $\widehat f_t$ has mean squared error at most
$\Err_{\rm orc}(T,\delta)/t$ under the average of the policies
used to collect its training data.
For a finite class $\F$, least-squares empirical risk minimization
(ERM),
$\widehat f_t\in\arg\min_{f\in\F}
\sum_{s=1}^t
\bigl(f(c_s,a_s)-\ell_s(c_s,a_s)\bigr)^2$,
satisfies \cref{ass:oracle} with
$\Err_{\rm orc}(T,\delta)=O(\log(|\F|/\delta))$.
The proof is in \cref{app:regression}.

Differential privacy definitions appear in \cref{sec:private}.
We use $\widetilde O(\cdot)$ to suppress logarithmic factors.

\section{Near-optimal regret for vanilla policy optimization}
\label{sec:algorithm-results}
We consider the standard policy-optimization update based on
exponential weights
\citep{lattimore2019bandit,Slivkins2019,agarwal2019reinforcement}.
Starting from the uniform policy, at each round $t$ the learner
observes $c_t$, samples an action from $\pi_t(c_t,\cdot)$, and
observes its loss. It then fits a predictor $\widehat f_t$ using
all observations collected so far and defines $\pi_{t+1}$ by
exponentially downweighting actions with larger predicted losses.
The algorithm has a simple structure and requires no additional
computations for explicit exploration. Its batched implementation
in \cref{sec:batched-implementation} is also computationally lightweight.
To our knowledge, vanilla PO is the first policy-optimization
method shown to achieve near-optimal regret in stochastic
contextual bandits without an explicit exploration bonus.

\begin{algorithm}[ht]
\caption{Vanilla Policy Optimization (\textsc{VanillaPO})}
\label{alg:plain-po}
\begin{algorithmic}
\State \textbf{Input:} $T$, $\A$, $\eta\in(0,1]$, regression oracle.
\State Initialize $\pi_1(c,a)=1/A$ for all $(c,a)\in\C\times\A$.
\State \textbf{for} $t=1,\ldots,T$ \textbf{do}
\State \AlgIndent Observe $c_t$, sample $a_t\sim\pi_t(c_t,\cdot)$, and observe
$\ell_t(c_t,a_t)$.
\State \AlgIndent Fit
$\widehat f_t\gets
\operatorname{Oracle}(\{(c_s,a_s,\ell_s(c_s,a_s))\}_{s=1}^t)$.
\State \AlgIndent For all $(c,a)\in\C\times\A$:
\(
\pi_{t+1}(c,a)
\propto
\pi_t(c,a)\exp\{-\eta\widehat f_t(c,a)\}.
\)
\State \textbf{end for}
\end{algorithmic}
\end{algorithm}

The algorithm enjoys near-optimal regret as stated next.
\begin{theorem}[Regret of vanilla PO]
\label{thm:general-oracle}
Suppose the offline regression oracle satisfies \cref{ass:oracle}.
Fix $\delta\in(0,1/2]$.
Then vanilla PO with learning rate
$\eta=\sqrt{(1+\log(AT))/
\bigl(AT(1+\Err_{\rm orc}(T,\delta/2))\bigr)}$
satisfies, with probability at least $1-\delta$,
$
 \Reg
 =O\!\left(
   \sqrt{AT\bigl(1+\Err_{\rm orc}(T,\delta/2)\bigr)}
   \bigl(1+\log(AT)\bigr)^{3/2}
   +\sqrt{T\log(2/\delta)}
 \right).
$
In particular, for a finite realizable class $\F$, using
least-squares ERM gives
$\Reg=\widetilde O\!\left(\sqrt{AT\log(|\F|/\delta)}\right)$.
\end{theorem}

\subsection{Proof overview}
\label{sec:proof-overview}
The proof of \cref{thm:general-oracle} hinges on understanding when a
suboptimal action can continue to receive nonnegligible probability.
The key insight is that, beyond a gap-dependent initial period, this
can happen only if the cumulative squared prediction error is large
relative to the action's gap. We use this connection to bound regret
deterministically at each context, then average over contexts and
apply the regression-oracle guarantee. The complete proof is given
in \cref{app:regret-proof}.

Fix a context $c$ and a realized sequence of fitted predictors.
We first analyze the resulting policies deterministically.
Let $\bar\pi_t(c,a)=\frac{1}{t}\sum_{r=1}^t\pi_r(c,a)$ be the average
of the first $t$ policies, and define the cumulative squared
prediction error by $ S_t(c)
 =\sum_{r=1}^t\sum_{a\in\A}
 \bar\pi_r(c,a)
 \bigl(\widehat f_r(c,a)-f^\star(c,a)\bigr)^2.$
After averaging over contexts, these weighted squared errors
are controlled by Assumption~\ref{ass:oracle}.

\paragraph{Key lemma.}
The central ingredient is the
\emph{late-round probability--error dichotomy}
in Lemma~\ref{lem:activation}.
The lemma states that, for every suboptimal action $a$ and round $t$,
\begin{equation}
 t-1\gtrsim\frac{\log(AT)}{\eta\Delta(c,a)}
 \quad\Longrightarrow\quad
 \left[
   \pi_t(c,a)<\frac1{AT}
   \quad\text{or}\quad
   \Delta(c,a)\lesssim A\eta S_{t-1}(c)
 \right].
 \label{eq:key-lemma}
\end{equation}
Thus, after an initial period that is shorter for larger gaps,
an action either has negligible probability or its gap is controlled
by the cumulative prediction error.

\paragraph{Why the dichotomy in \eqref{eq:key-lemma} holds.}
For a suboptimal action $a$ to retain probability at least $1/(AT)$
in a late round, the exponential update requires the predictors
to substantially underestimate its loss gap in total.
This can result from underestimating the loss of $a$ or
overestimating the loss of the optimal action $a^\star(c)$,
so we need historical coverage of both actions.
Suppose that $S_{t-1}(c)$ is smaller than a sufficiently small
constant times $\Delta(c,a)/(A\eta)$.
An induction shows that both actions then retain probability
$\Omega(1/A)$ during the first $1/(\eta\Delta(c,a))$ updates.
Earlier probability bounds control accumulated errors in the
predicted loss gaps, which in turn control the next policy.
The argument bounds all action weights together, since they
share the same normalizing denominator.

This initial exploration provides lasting historical coverage:
the early probabilities remain in the average policies even
after the current probabilities decrease.
Prediction errors on both actions therefore receive enough
weight in $S_{t-1}(c)$ for weighted Cauchy--Schwarz to control
the accumulated error in their predicted loss gap.
Under the assumed small-error condition, this accumulated error
is too small to offset the true cumulative loss gap.
This contradicts the underestimation required for action $a$
to persist, establishing the large-error alternative in
\eqref{eq:key-lemma}.

\paragraph{Regret decomposition.}
The dichotomy suggests separating action--round pairs into three
regimes: small gaps, large gaps in early rounds, and large gaps in
late rounds. Let $\tau(c)$ be a sufficiently large universal constant times
$A\eta S_T(c)$, and call a gap small when
$\Delta(c,a)\le\tau(c)$.
For a suboptimal action, call a round late when the time condition
in \eqref{eq:key-lemma} holds, and early otherwise.
Small gaps contribute at most $\tau(c)$ regret per round.
For large gaps in early rounds, rearranging the definition of
an early round gives
$\Delta(c,a)\lesssim \frac{\log(AT)}{\eta(t-1)}$ when $t\ge2$.
Both contributions are bounded by summing action probabilities
at each round, rather than paying a separate exploration cost for
each action. For large gaps in late rounds, the choice of $\tau(c)$
and $S_{t-1}(c)\le S_T(c)$ rule out the large-error alternative
in the dichotomy. Each such action therefore has probability
below $1/(AT)$, and their combined regret is at most $1/T$.
Consequently, for $t\ge2$,
\begin{equation*}
 \sum_{a\in\A}\pi_t(c,a)\Delta(c,a)
 \lesssim
 \underbrace{A\eta S_T(c)}_{\text{small gap}}
 +
 \underbrace{\frac{\log(AT)}{\eta(t-1)}}_{
   \substack{\text{large gap}\\\text{early round}}}
 +
 \underbrace{\frac1T}_{
   \substack{\text{large gap}\\\text{late round}}}.
\end{equation*}
Let $R_T(c)$ denote the cumulative regret obtained by evaluating
every policy in the learned sequence at the same fixed context $c$.
Summing over rounds and bounding the first-round regret by one yields
\begin{equation}
 R_T(c)
 :=\sum_{t=1}^T\sum_{a\in\A}\pi_t(c,a)\Delta(c,a)
 \lesssim A\eta T S_T(c)+\frac{\log^2(AT)}{\eta}.
 \label{eq:overview-pointwise}
\end{equation}
The oracle guarantee, applied with failure probability $\delta/2$,
gives $ \E_{c\sim\D}[S_T(c)]
 \le \Err_{\rm orc}(T,\delta/2)(1+\log T).$
Averaging \eqref{eq:overview-pointwise} therefore bounds
$\sum_{t=1}^T\mathcal R(\pi_t)$.
A concentration argument
transfers this bound to $\Reg$, adding
$\sqrt{(T/2)\log(2/\delta)}$.
Substituting the learning rate chosen using
$\Err_{\rm orc}(T,\delta/2)$ and taking a union bound
gives the regret guarantee with probability at least $1-\delta$.

\section{Batched implementation of vanilla policy optimization}
\label{sec:batched-implementation}

The vanilla PO implementation in the previous section updates the policy and calls the
regression oracle after every round. While conceptually simple, this can
be undesirable in practice: fitting the regression model may itself be
computationally expensive, and changing a deployed policy may incur an
additional switching or deployment cost. Differential privacy provides
an important example, since repeatedly releasing updated predictors or
policies may incur additional privacy loss.
These considerations motivate a batched implementation of PO, in which
the predictor and policy are updated only at a small number of epochs.
To obtain that, we use geometrically growing batches. Within each epoch, the policy is
fixed, and the data collected in that epoch are used once to fit the
predictor for the next epoch. Consequently, the algorithm makes only
$O(\log T)$ regression-oracle calls and $O(\log T)$ policy switches, while preserving the optimal regret guarantee.

\paragraph{Batched vanilla PO.}
Let $B_m=2^{m-1}$ denote the length of epoch $m$. At the beginning of
epoch $m\ge2$, we fit a predictor $\widetilde f_m$ using only the data
collected in epoch $m-1$. We then construct a single policy that remains
fixed throughout epoch $m$.
To define this policy, consider starting from the uniform distribution
and repeatedly applying the exponential update using the same predictor
$\widetilde f_m$. After $k$ updates, the resulting distribution is
proportional to
$\exp\{-\eta_m k\widetilde f_m(c,a)\}$.
We take the average of the distributions corresponding to
$k=0,\ldots,B_m-1$:
\begin{equation}
\pi_m(c,a)
=\frac1{B_m}\sum_{k=0}^{B_m-1}
\frac{\exp\{-\eta_m k\widetilde f_m(c,a)\}}
{\sum_{b\in\A}\exp\{-\eta_m k\widetilde f_m(c,b)\}}.
\label{eq:vanilla-batch-policy}
\end{equation}
The early terms in this average spread probability relatively evenly
across actions, whereas the later terms place increasingly more mass on
actions with smaller predicted loss. Their average therefore preserves
the exploration induced by the early exponential iterates while
exploiting the predictions in the later ones.
Only the final epoch may be truncated by the horizon. Its policy is
still defined using all $B_M$ mixture components in
\eqref{eq:vanilla-batch-policy}; the data collected in this final epoch
are simply never used for another regression fit.

\begin{algorithm}[ht]
\caption{Batched Vanilla Policy Optimization}
\label{alg:batched-vanilla-po}
\begin{algorithmic}
\State \textbf{Input:} horizon $T$, action set $\A$, learning rates
$\eta_m$, and an offline regression oracle.
\State Set $M\gets\lceil\log_2(T+1)\rceil$ and
$B_m\gets2^{m-1}$ for $m=1,\ldots,M$.
\State Initialize $\pi_1(c,a)\gets1/A$ for every
$c\in\C$ and $a\in\A$.
\State \textbf{for} $m=1,\ldots,M$ \textbf{do}
\State \AlgIndent \textbf{if} $m\ge2$ \textbf{then}
\State \AlgIndent\AlgIndent Fit
$\widetilde f_m\gets\operatorname{Oracle}(\mathcal S_{m-1})$
and discard $\mathcal S_{m-1}$.
\State \AlgIndent\AlgIndent Define $\pi_m$ by
\eqref{eq:vanilla-batch-policy}.
\State \AlgIndent \textbf{end if}
\State \AlgIndent Initialize $\mathcal S_m\gets()$.
\State \AlgIndent \textbf{for}
$t=B_m,\ldots,\min\{2B_m-1,T\}$ \textbf{do}
\State \AlgIndent\AlgIndent Observe $c_t$, sample
$a_t\sim\pi_m(c_t,\cdot)$, and observe $\ell_t(c_t,a_t)$.
\State \AlgIndent\AlgIndent Append
$(c_t,a_t,\ell_t(c_t,a_t))$ to $\mathcal S_m$.
\State \AlgIndent \textbf{end for}
\State \textbf{end for}
\end{algorithmic}
\end{algorithm}

\paragraph{Efficient implementation.}
Although \eqref{eq:vanilla-batch-policy} defines $\pi_m$ for every
context--action pair, it never needs to be constructed explicitly over
the context space. At round $t$ in epoch $m$, the learner evaluates
$\widetilde f_m(c_t,a)$ only for the observed context $c_t$ and the
$A$ actions.
Moreover, sampling from the mixture in
\eqref{eq:vanilla-batch-policy} does not require explicitly computing
all $B_m$ component distributions. Instead, independently sample
$K_t\sim\operatorname{Unif}\{0,\ldots,B_m-1\}$ and then draw
$a_t$ according to
\begin{equation}
\Pp(a_t=a\mid c_t,K_t,\widetilde f_m)
=
\frac{\exp\{{-\eta_mK_t\widetilde f_m(c_t,a)}\}}
{\sum_{b\in\A}\exp\{{-\eta_mK_t\widetilde f_m(c_t,b)}\}}.
\label{eq:vanilla-batch-sampler}
\end{equation}
Averaging over $K_t$ recovers exactly $\pi_m(c_t,\cdot)$.
Importantly, drawing $K_t$ only samples a component of the fixed mixture;
it does not constitute a policy update. Thus, each round requires only
$A$ predictor evaluations and $O(A)$ additional arithmetic, while the
regression oracle and the deployed policy are updated only
$O(\log T)$ times.

\paragraph{Regret guarantee.}
Because the policy is fixed within each epoch, conditional on the history
at its start, the observations in that epoch are i.i.d. under a fixed
policy. We can therefore apply a standard offline regression guarantee
to each batch, as in FALCON+ \citep{falcon}.
Fix $\delta\in(0,1/2]$ and suppose that, with probability at least
$1-\delta/2$,
$\|\widetilde f_m-f^\star\|_{\pi_{m-1}}^2
\le\Err_{\rm orc}^{\rm batch}(T,\delta)/B_{m-1}$
simultaneously for all $m=2,\ldots,M$, where
$\Err_{\rm orc}^{\rm batch}(T,\delta)\ge1$ is a known error bound
fixed before the interaction.

\begin{theorem}[Regret of batched vanilla PO]
\label{thm:batched-vanilla}
Under the regression oracle guarantee above, batched vanilla PO with
learning rates
$
\eta_m
=
\Theta\!\left(
\min\left\{
1/A,
1/{\sqrt{AB_m\Err_{\rm orc}^{\rm batch}(T,\delta)}}
\right\}
\right)
$
uses $O(\log T)$ regression calls and $O(\log T)$ policy switches, and
satisfies
$
\Reg
=
\widetilde O\!\left(
\sqrt{AT\,\Err_{\rm orc}^{\rm batch}(T,\delta)}
\right)
$
with probability at least $1-\delta$.
In particular, for a finite realizable class $\F$, least-squares ERM
with the corresponding error bound gives
$
\Reg
=
\widetilde O\!\left(
\sqrt{AT\log(|\F|/\delta)}
\right).
$
\end{theorem}
The proof appears in \cref{app:batched}.
Using each batch in at most one fit also enables the privacy
guarantees developed next.

\section{Differential privacy and policy optimization }
\label{sec:private}

We now privatize batched PO by replacing each offline regression call
with a differentially private regression oracle. Since the batches are
disjoint and each batch is used in only one fit, the privacy cost does
not accumulate across epochs. This allows us to protect the released
predictors and policy maps while retaining the near-optimal regret
guarantee of the batched algorithm.

\paragraph{Privacy model.}
Each round represents one user with record
$z_t=(c_t,\ell_t(c_t,\cdot))$, consisting of a context and its
potential-loss vector. Each user participates once.
Two streams are neighboring if they differ in one record.
Our privacy guarantee holds for arbitrary fixed streams, and
the stochastic assumptions in \cref{sec:setup} are needed
only for regret.

\begin{definition}[Differential privacy]
\label{def:dp}
Let $\mathcal M$ denote the learning algorithm, whose public output
on input stream $z$ is
$\mathcal M(z)=((\widetilde f_m,\pi_m))_{m=2}^M$.
Contexts, selected actions, observed losses, and policy evaluations
at protected contexts remain internal.
The algorithm $\mathcal M$ is
$(\varepsilon_{\rm priv},\delta_{\rm priv})$-differentially
private if, for every pair of neighboring streams $z,z'$
and every event $E$ in the output space,
\begin{equation}
 \Pp(\mathcal M(z)\in E)
 \le
 e^{\varepsilon_{\rm priv}}
 \Pp(\mathcal M(z')\in E)
 +\delta_{\rm priv}.
 \label{eq:dp}
\end{equation}
The probabilities are over the learner's randomness.
\end{definition}

\begin{assumption}[Private regression oracle]
\label{ass:private-oracle}
Each oracle call uses fresh randomness and is
$(\varepsilon_{\rm priv},\delta_{\rm priv})$-DP under replacement
of one row, for every pair of datasets of the same size.
For every $\delta\in(0,1/2]$, with probability at least
$1-\delta/2$,
$\|\widetilde f_m-f^\star\|_{\pi_{m-1}}^2
\le \Err_{\rm orc}^{\rm batch}(T,\delta)/B_{m-1}$
simultaneously for all $m=2,\ldots,M$,
where $\Err_{\rm orc}^{\rm batch}(T,\delta)\ge1$
is a known error bound.
The function class, the error bound, and all tuning parameters
are public and fixed before the interaction.
\end{assumption}

\subsection{Differentially private batched vanilla PO}
We instantiate \cref{alg:batched-vanilla-po} with the private
regression oracle of \cref{ass:private-oracle} and release each
predictor and its policy map before the corresponding epoch.
Each batch enters only one private fit, and subsequent releases
depend on that batch only through the private output.
Consequently, the complete transcript inherits the privacy
parameters of a single oracle call.
\begin{theorem}[Private batched vanilla PO]
\label{thm:private-batched}
Under \cref{ass:private-oracle}, batched vanilla PO releases an
$(\varepsilon_{\rm priv},\delta_{\rm priv})$-DP transcript using
$O(\log T)$ private regression calls.
With the learning rates of \cref{thm:batched-vanilla}
chosen using the private oracle's error bound, it satisfies
$\Reg=\widetilde O\!\left(
\sqrt{AT\,\Err_{\rm orc}^{\rm batch}(T,\delta)}
\right)$
with probability at least $1-\delta$.
\end{theorem}

For a finite function class, the exponential mechanism gives
the following pure-DP instantiation.

\begin{corollary}[Finite-class private regression]
\label{cor:private-exp}
Let $\F$ be finite and realizable, fix
$\delta\in(0,1/2]$, and let $\varepsilon_{\rm priv}>0$.
Batched vanilla PO, instantiated with the exponential mechanism
and the parameter choices specified in
\cref{app:batch-oracle-parameters},
releases a pure $\varepsilon_{\rm priv}$-DP transcript
and satisfies
$\Reg
=
\widetilde O\left(
\sqrt{AT
\log(|\F|/\delta)}(1+\varepsilon_{\rm priv}^{-1/2})
\right)$
with probability at least $1-\delta$.
\end{corollary}

\paragraph{Relation to joint differential privacy.}
Joint differential privacy (JDP) allows a user's recommendation to depend on their own context while still ``hiding'' their data from the other users. This is achievable in our construction via the standard billboard model \citep{kearns2014mechanism,hsu2014private}, in which users personalize their recommendation as a function of a public output (the ``billboard'') and their own private data. In particular, our finite-class exponential-mechanism instantiation gives pure JDP. A similar connection underlies \citet{shariff2018private}, who obtain approximate JDP for contextual linear bandits by maintaining differentially private regression statistics. Our main guarantee is stated directly for the released predictors and policy maps.

\paragraph{Proof overview.}
The regret and privacy arguments follow from two complementary properties
of the batched construction. For regret, we simply invoke
\cref{thm:batched-vanilla} with the prediction-error guarantee of the
private regression oracle in \cref{ass:private-oracle}; for a finite
class, the exponential mechanism provides a concrete pure-DP
instantiation whose excess empirical loss yields the required batch
prediction-error bound, as shown in
\cref{lem:private-exp-oracle}. For privacy, each record belongs to a
single batch and enters the training data of at most one private oracle call.
After that batch is discarded, all subsequent predictors, policies, and
released outputs depend on the record only through this private oracle
output and are therefore postprocessing. Thus, the full transcript
inherits the privacy guarantee of a single oracle call, with no
composition cost across epochs. The full argument appears in
\cref{app:privacy}.

\paragraph{Extension to other batched algorithms.}
The same privacy argument applies to other oracle-based CMAB
algorithms that collect each batch under a fixed policy and use
it in only one regression fit.
If subsequent policies depend only on the fitted predictors and
public parameters, replacing the regression oracle with a private
one protects the complete sequence of released predictors and
policy maps with the privacy parameters of a single oracle call.
We use this observation to construct private versions of several
standard algorithms for the experiments in
\Cref{sec:experiments}.

\paragraph{Batched bonus PO.}
We also develop a batched private adaptation of the bonus-based
PO method of \citet{levy2026}, whose exploration bonus favors
actions with little accumulated probability.
The batched construction satisfies the conditions above and
therefore inherits the same privacy guarantee.
With the prescribed learning rates and bonus parameters, it
satisfies
$\Reg=\widetilde O\!\left(
\sqrt{AT\,\Err_{\rm orc}^{\rm batch}(T,\delta)}
\right)$
with probability at least $1-\delta$.
The construction and regret analysis appear in
\cref{app:bonus-po}.
For the experiments, we develop a
faster approximation of this algorithm.

\section{Experiments}\label{sec:experiments}

Following standard practice in the CMAB literature~\citep{bietti2021bakeoff},
we evaluate the algorithms on randomly permuted multiclass
classification datasets from \textsc{OpenML}\footnote{\url{https://www.openml.org/}},
treating class labels as actions, feature vectors as contexts, and the $0$--$1$ classification
error as the loss. We compare \algname{VanillaPO} and
\algname{BonusPO} against the regression-based baselines
\algname{SquareCB}~\citep{foster2020squarecb},
\algname{FastCB}~\citep{foster2021efficient},
\algname{RegCB}~\citep{foster2018regression}, and
\algname{AdaCB}~\citep{foster2021instance},
as well as \algname{LinUCB}~\citep{li2010contextual,shariff2018private},
together with the fully supervised baseline \algname{Supervised}.
We outline the experimental setup here and defer full implementation details and privacy accounting to \cref{app:experiments}.\footnote{
Our code is publicly available at \url{https://github.com/alex-rbch/private-contextual-bandits}.}

\paragraph{Implementation details.}
All algorithms follow the same doubling batch schedule, updating
their policies at the end of each batch. The oracles fit regularized
logistic or linear regression models to produce the action-loss
predictors used in these updates. To ensure differential privacy
of the predictors and, by post-processing, the resulting policies,
we employ objective perturbation~\citep{chaudhuri2008privacy,chaudhuri2011private}
and sufficient-statistic perturbation~\citep{wang2018revisiting}.
We note that a direct implementation of \algname{BonusPO} incurs
$O(T^2)$ computational cost. To make this tractable in our experiments, we use
a historical-replay modification inspired by \citet{levy2026}:
at each context, we replay bonus-based updates using the privatized
predictors from previous batches, weighting each update by the
corresponding batch size. Since the doubling schedule produces only $O(\log T)$ predictors,
replay requires $O(T\log T)$ computation across all $T$ rounds.

\paragraph{Experimental setup and evaluation.}
We evaluate the algorithms on 100 datasets with horizons $T \in [10^3,10^5]$, allowing $9$ to $16$ policy updates per run. For brevity, throughout this section we write $\varepsilon$ and $\delta$ for $\varepsilon_{\mathrm{priv}}$ and $\delta_{\mathrm{priv}}$, respectively. We consider four private regimes with $\varepsilon\in\{8,4,2,1\}$ and $\delta=10^{-5}$, together with a non-private regime. Performance is measured by progressive-validation (PV) loss, the average incurred loss up to round $t$: $L_{\mathrm{PV}}(t)=\frac1t\sum_{s=1}^{t}\ell_s(a_s)$.

Following previous work, we tune each algorithm independently for each dataset and privacy regime, selecting the hyperparameter configuration that minimizes the mean final PV loss. Table~\ref{tab:exp-hyperparameters} lists the search grids. The reported results average ten runs with independently sampled dataset permutations and privacy perturbations, both shared across algorithms and privacy regimes within each run.

\begin{figure*}[t]
\centering
\includegraphics[width=\textwidth]{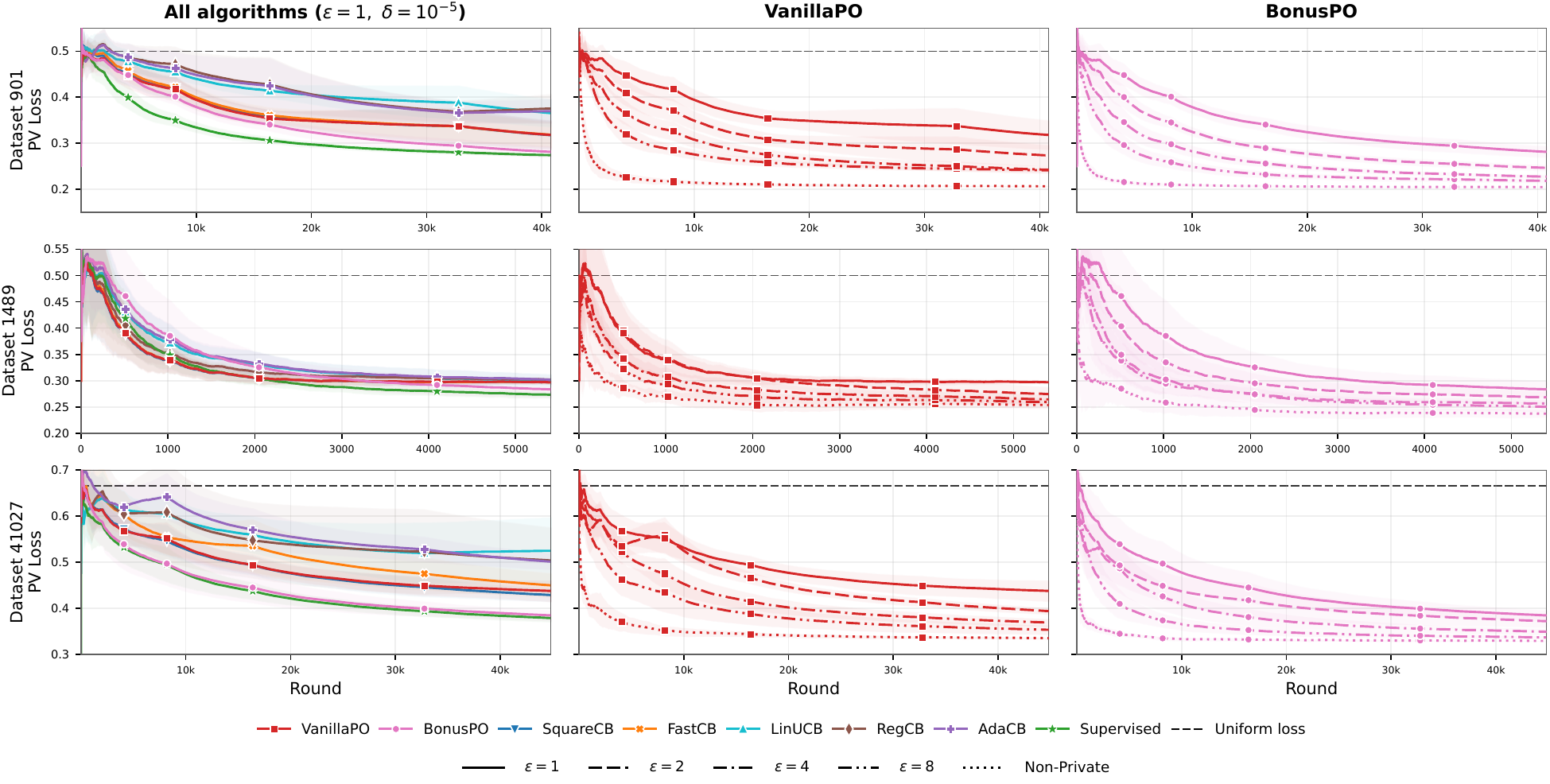}
\caption{Progressive-validation loss on \textsc{OpenML} datasets 901, 1489, and 41027 (rows). The left column compares all algorithms at $(\varepsilon,\delta)=(1,10^{-5})$, while the middle and right columns show \algname{VanillaPO} and \algname{BonusPO}, respectively, across privacy levels.}
\label{fig:experiments-privacy}
\end{figure*}

\paragraph{Discussion of results in Figure~\ref{fig:experiments-privacy}.}
At the strongest tested private regime, $(\varepsilon,\delta)=(1,10^{-5})$, datasets 901 and 41027 exhibit a clear separation into three groups: \algname{BonusPO} is the strongest bandit method and approaches the supervised reference, followed by \algname{VanillaPO}, \algname{SquareCB}, and \algname{FastCB}, while \algname{LinUCB}, \algname{RegCB}, and \algname{AdaCB} perform worse. On dataset 1489, the algorithms are closer together, although the same separation begins to emerge near the end. \Cref{app:experiments} shows that these patterns persist more broadly across 100 \textsc{OpenML} datasets.

A plausible explanation for this behavior is that stronger privacy introduces more noise into the fitted action-loss predictors. The exponential updates of \algname{VanillaPO} and \algname{BonusPO} change the policy gradually in response to these predictors, whereas the baseline methods translate each updated fit more directly into their action probabilities and can therefore react more strongly to this noise. \algname{BonusPO} additionally benefits from its exploration bonus, which favors actions that have received little past probability and provides additional robustness when the fitted losses are inaccurate.

The remaining panels show that \algname{VanillaPO} and \algname{BonusPO} degrade gracefully as privacy becomes stronger. In \cref{app:experiments}, we provide the full privacy-regime and algorithm-wise plots for these three datasets, allowing the same trends to be compared across all methods and privacy regimes.

\section{Discussion}
Our results establish theoretical guarantees for one of the simplest forms of policy optimization, an approach widely used in practice. Vanilla PO achieves near-optimal regret in stochastic contextual bandits without explicit exploration bonuses under realizability and using an offline regression oracle. Gradual updates preserve exploration during the initial rounds, helping the regression oracle correct predictions that make suboptimal actions appear better than they are.
Batching makes this approach attractive for both computation and privacy. The batched algorithms use $O(\log T)$ regression calls and policy switches. Fitting each batch once with a private oracle protects the released predictors and policy maps without additional privacy cost across epochs. Regret and computational efficiency depend on the oracle's accuracy and runtime, respectively.

The experiments suggest that exploration bonuses can remain useful even when they are unnecessary for the regret guarantee. The historical-replay variant of bonus PO performs strongly in our comparisons, including under privacy constraints. These findings motivate understanding when additional exploration improves robustness to private prediction errors.

Our guarantees rely on stochastic contexts, realizability, and exact policy updates. A natural direction is to investigate whether the principles underlying our exploration and privacy analyses can guide algorithm design in more complex interactive settings, such as reinforcement learning. Such extensions may require modified policy updates or additional exploration. Another direction is to improve the privacy--utility tradeoff within our regression-oracle framework.

\subsection*{Acknowledgment}
This work is supported by the European Research Council (ERC) under the European Union’s Horizon 2020 research and innovation program (grant agreement No. 882396); the Israel Science Foundation (grants 1357/24 and 1419/24); the Blavatnik Foundation at Tel Aviv University; the Tel Aviv University Center for AI and Data Science (TAD); and the National Science Foundation (grant ECCS-2217023). Orin Levy is also supported by a Google PhD Fellowship (2025). This research was enabled in part by computational resources provided by the Vector Institute and the Digital Research Alliance of Canada.

\subsection*{AI use statement}
Generative AI assistance was used to check and develop mathematical arguments, verify references, assist with code development and debugging, and draft and revise the exposition. The authors are responsible for the final content.

\bibliography{references}

@inproceedings{agarwal2014,
  title = {Taming the monster: A fast and simple algorithm for contextual bandits},
  author = {Agarwal, Alekh and Hsu, Daniel and Kale, Satyen and Langford, John and Li, Lihong and Schapire, Robert},
  booktitle = {Proceedings of the 31st International Conference on Machine Learning},
  volume = {32},
  pages = {1638--1646},
  year = {2014},
  series = {Proceedings of Machine Learning Research},
  publisher = {PMLR}
}

@article{dwork-roth-2014,
  title = {The Algorithmic Foundations of Differential Privacy},
  author = {Dwork, Cynthia and Roth, Aaron},
  journal = {Foundations and Trends in Theoretical Computer Science},
  volume = {9},
  number = {3--4},
  pages = {211--407},
  year = {2014},
  publisher = {Now Publishers Inc.}
}

@article{howard-2020,
  title = {Time-uniform {Chernoff} bounds via nonnegative supermartingales},
  author = {Howard, Steven R. and Ramdas, Aaditya and McAuliffe, Jon and Sekhon, Jasjeet},
  journal = {Probability Surveys},
  volume = {17},
  pages = {257--317},
  year = {2020}
}

@inproceedings{levy2026,
  title = {Optimal Regret for Policy Optimization in Contextual Bandits},
  author = {Levy, Orin and Mansour, Yishay},
  booktitle = {Proceedings of the 43rd International Conference on Machine Learning},
  year = {2026},
  note = {arXiv:2602.13700}
}

@article{falcon,
  title = {Bypassing the monster: A faster and simpler optimal algorithm for contextual bandits under realizability},
  author = {Simchi-Levi, David and Xu, Yunzong},
  journal = {Mathematics of Operations Research},
  volume = {47},
  number = {3},
  pages = {1904--1931},
  year = {2022},
  publisher = {INFORMS}
}

@article{bietti2021bakeoff,
  title = {A Contextual Bandit Bake-off},
  author = {Bietti, Alberto and Agarwal, Alekh and Langford, John},
  journal = {Journal of Machine Learning Research},
  volume = {22},
  number = {133},
  pages = {1--49},
  year = {2021}
}

@inproceedings{foster2021efficient,
  title = {Efficient First-Order Contextual Bandits: Prediction, Allocation, and Triangular Discrimination},
  author = {Foster, Dylan J. and Krishnamurthy, Akshay},
  booktitle = {Advances in Neural Information Processing Systems},
  volume = {34},
  pages = {18907--18919},
  year = {2021}
}

@inproceedings{li2010contextual,
  title = {A Contextual-Bandit Approach to Personalized News Article Recommendation},
  author = {Li, Lihong and Chu, Wei and Langford, John and Schapire, Robert E.},
  booktitle = {Proceedings of the 19th International Conference on World Wide Web},
  pages = {661--670},
  year = {2010},
  publisher = {ACM}
}

@inproceedings{bun2016concentrated,
  title = {Concentrated Differential Privacy: Simplifications, Extensions, and Lower Bounds},
  author = {Bun, Mark and Steinke, Thomas},
  booktitle = {Theory of Cryptography},
  volume = {9985},
  pages = {635--658},
  year = {2016},
  series = {Lecture Notes in Computer Science},
  publisher = {Springer}
}

@inproceedings{shariff2018private,
  title = {Differentially Private Contextual Linear Bandits},
  author = {Shariff, Roshan and Sheffet, Or},
  booktitle = {Advances in Neural Information Processing Systems},
  volume = {31},
  year = {2018},
  publisher = {Curran Associates, Inc.}
}

@article{agarwal2017decisions,
  title = {Making Contextual Decisions with Low Technical Debt},
  author = {Agarwal, Alekh and Bird, Sarah and Cozowicz, Markus and Hoang, Luong and Langford, John and Lee, Stephen and Li, Jiaji and Melamed, Dan and Oshri, Gal and Ribas, Oswaldo and Sen, Siddhartha and Slivkins, Alex},
  journal = {arXiv preprint arXiv:1606.03966},
  year = {2017}
}

@article{xu2020uccb,
  title = {Upper Counterfactual Confidence Bounds: A New Optimism Principle for Contextual Bandits},
  author = {Xu, Yunbei and Zeevi, Assaf},
  journal = {arXiv preprint arXiv:2007.07876},
  year = {2020}
}

@article{schulman2017ppo,
  title = {Proximal Policy Optimization Algorithms},
  author = {Schulman, John and Wolski, Filip and Dhariwal, Prafulla and Radford, Alec and Klimov, Oleg},
  journal = {arXiv preprint arXiv:1707.06347},
  year = {2017}
}

@article{chen2025private,
  title = {Beyond Covariance Matrix: The Statistical Complexity of Private Linear Regression},
  author = {Chen, Fan and Li, Jiachun and Rakhlin, Alexander and Simchi-Levi, David},
  journal = {arXiv preprint arXiv:2502.13115},
  year = {2025}
}

@inproceedings{he2025privatepo,
  title = {On the Sample Complexity of Differentially Private Policy Optimization},
  author = {He, Yi and Zhou, Xingyu},
  booktitle = {Advances in Neural Information Processing Systems},
  volume = {38},
  pages = {11181--11217},
  year = {2025},
  publisher = {Curran Associates, Inc.}
}

@inproceedings{zheng2020local,
  title = {Locally Differentially Private (Contextual) Bandits Learning},
  author = {Zheng, Kai and Cai, Tianle and Huang, Weiran and Li, Zhenguo and Wang, Liwei},
  booktitle = {Advances in Neural Information Processing Systems},
  volume = {33},
  pages = {12300--12310},
  year = {2020},
  publisher = {Curran Associates, Inc.}
}

@article{zhou2022privaterl,
  title = {Differentially Private Reinforcement Learning with Linear Function Approximation},
  author = {Zhou, Xingyu},
  journal = {Proceedings of the ACM on Measurement and Analysis of Computing Systems},
  volume = {6},
  number = {1},
  pages = {1--27},
  year = {2022}
}

@inproceedings{foster2018regression,
  title = {Practical Contextual Bandits with Regression Oracles},
  author = {Foster, Dylan and Agarwal, Alekh and Dudik, Miroslav and Luo, Haipeng and Schapire, Robert},
  booktitle = {Proceedings of the 35th International Conference on Machine Learning},
  volume = {80},
  pages = {1539--1548},
  year = {2018},
  series = {Proceedings of Machine Learning Research},
  publisher = {PMLR}
}

@inproceedings{shani2020optimistic,
  title = {Optimistic Policy Optimization with Bandit Feedback},
  author = {Shani, Lior and Efroni, Yonathan and Rosenberg, Aviv and Mannor, Shie},
  booktitle = {Proceedings of the 37th International Conference on Machine Learning},
  volume = {119},
  pages = {8604--8613},
  year = {2020},
  series = {Proceedings of Machine Learning Research},
  publisher = {PMLR}
}

@inproceedings{foster2020squarecb,
  title = {Beyond {UCB}: Optimal and Efficient Contextual Bandits with Regression Oracles},
  author = {Foster, Dylan and Rakhlin, Alexander},
  booktitle = {Proceedings of the 37th International Conference on Machine Learning},
  volume = {119},
  pages = {3199--3210},
  year = {2020},
  series = {Proceedings of Machine Learning Research},
  publisher = {PMLR}
}

@inproceedings{vietri2020private,
  title = {Private Reinforcement Learning with {PAC} and Regret Guarantees},
  author = {Vietri, Giuseppe and Balle, Borja and Krishnamurthy, Akshay and Wu, Steven},
  booktitle = {Proceedings of the 37th International Conference on Machine Learning},
  volume = {119},
  pages = {9754--9764},
  year = {2020},
  series = {Proceedings of Machine Learning Research},
  publisher = {PMLR}
}

@inproceedings{sajed2019private,
  title = {An Optimal Private Stochastic-{MAB} Algorithm based on Optimal Private Stopping Rule},
  author = {Sajed, Touqir and Sheffet, Or},
  booktitle = {Proceedings of the 36th International Conference on Machine Learning},
  volume = {97},
  pages = {5579--5588},
  year = {2019},
  series = {Proceedings of Machine Learning Research},
  publisher = {PMLR}
}

@inproceedings{cai2020exploration,
  title = {Provably Efficient Exploration in Policy Optimization},
  author = {Cai, Qi and Yang, Zhuoran and Jin, Chi and Wang, Zhaoran},
  booktitle = {Proceedings of the 37th International Conference on Machine Learning},
  volume = {119},
  pages = {1283--1294},
  year = {2020},
  series = {Proceedings of Machine Learning Research},
  publisher = {PMLR}
}

@inproceedings{pavlovic2025private,
  title = {Differentially private kernelized contextual bandits},
  author = {Pavlovic, Nikola and Salgia, Sudeep and Zhao, Qing},
  booktitle = {Proceedings of the 28th International Conference on Artificial Intelligence and Statistics},
  volume = {258},
  pages = {4618--4626},
  year = {2025},
  series = {Proceedings of Machine Learning Research},
  publisher = {PMLR}
}

@inproceedings{chenrakhlin2025decision,
  title = {Decision making in changing environments: Robustness, query-based learning, and differential privacy},
  author = {Chen, Fan and Rakhlin, Alexander},
  booktitle = {Proceedings of Thirty Eighth Conference on Learning Theory},
  volume = {291},
  pages = {983--985},
  year = {2025},
  series = {Proceedings of Machine Learning Research},
  publisher = {PMLR}
}

@inproceedings{he2026general,
  title = {Towards Differentially Private Reinforcement Learning with General Function Approximation},
  author = {He, Yi and Zhou, Xingyu},
  booktitle = {Advances in Neural Information Processing Systems},
  year = {2026},
  note = {To appear. arXiv:2605.07049}
}

@article{Slivkins2019,
  title = {Introduction to Multi-Armed Bandits},
  author = {Aleksandrs Slivkins},
  journal = {Foundations and Trends in Machine Learning},
  volume = {12},
  number = {1--2},
  pages = {1--286},
  year = {2019},
  publisher = {Now Publishers}
}

@book{lattimore2019bandit,
  title = {Bandit Algorithms},
  author = {Lattimore, Tor and Szepesv{\'a}ri, Csaba},
  year = {2020},
  publisher = {Cambridge University Press}
}

@article{agarwal2019reinforcement,
  title={Reinforcement learning: Theory and algorithms},
  author={Agarwal, Alekh and Jiang, Nan and Kakade, Sham M and Sun, Wen},
  journal={CS Dept., UW Seattle, Seattle, WA, USA, Tech. Rep},
  volume={32},
  pages={96},
  year={2019}
}

@inproceedings{agarwal2012contextual,
  title = {Contextual bandit learning with predictable rewards},
  author = {Agarwal, Alekh and Dud{\'\i}k, Miroslav and Kale, Satyen and Langford, John and Schapire, Robert},
  booktitle = {Proceedings of the Fifteenth International Conference on Artificial Intelligence and Statistics},
  volume = {22},
  pages = {19--26},
  year = {2012},
  series = {Proceedings of Machine Learning Research},
  publisher = {PMLR}
}

@inproceedings{chaudhuri2008privacy,
  title = {Privacy-preserving logistic regression},
  author = {Chaudhuri, Kamalika and Monteleoni, Claire},
  booktitle = {Advances in Neural Information Processing Systems},
  volume = {21},
  pages = {289--296},
  year = {2008}
}

@article{chaudhuri2011private,
  title = {Differentially Private Empirical Risk Minimization},
  author = {Chaudhuri, Kamalika and Monteleoni, Claire and Sarwate, Anand D.},
  journal = {Journal of Machine Learning Research},
  volume = {12},
  number = {29},
  pages = {1069--1109},
  year = {2011}
}

@inproceedings{abbasi2019politex,
  title = {{POLITEX}: Regret Bounds for Policy Iteration using Expert Prediction},
  author = {Abbasi-Yadkori, Yasin and Bartlett, Peter and Bhatia, Kush and Lazic, Nevena and Szepesvari, Csaba and Weisz, Gell{\'e}rt},
  booktitle = {Proceedings of the 36th International Conference on Machine Learning},
  volume = {97},
  pages = {3692--3702},
  year = {2019},
  series = {Proceedings of Machine Learning Research},
  publisher = {PMLR}
}

@article{agarwal2021theory,
  title = {On the theory of policy gradient methods: Optimality, approximation, and distribution shift},
  author = {Agarwal, Alekh and Kakade, Sham M and Lee, Jason D and Mahajan, Gaurav},
  journal = {Journal of Machine Learning Research},
  volume = {22},
  number = {98},
  pages = {1--76},
  year = {2021}
}

@inproceedings{han2021generalized,
  title = {Generalized linear bandits with local differential privacy},
  author = {Han, Yuxuan and Liang, Zhipeng and Wang, Yang and Zhang, Jiheng},
  booktitle = {Advances in Neural Information Processing Systems},
  volume = {34},
  pages = {26511--26522},
  year = {2021}
}

@article{chowdhury2022differentially,
  title = {Differentially Private Regret Minimization in Episodic {Markov} Decision Processes},
  author = {Chowdhury, Sayak Ray and Zhou, Xingyu},
  journal = {Proceedings of the AAAI Conference on Artificial Intelligence},
  volume = {36},
  number = {6},
  pages = {6375--6383},
  year = {2022}
}

@inproceedings{bassily2014privateerm,
  title = {Private Empirical Risk Minimization: Efficient Algorithms and Tight Error Bounds},
  author = {Bassily, Raef and Smith, Adam and Thakurta, Abhradeep},
  booktitle = {2014 IEEE 55th Annual Symposium on Foundations of Computer Science},
  pages = {464--473},
  year = {2014},
  publisher = {IEEE}
}

@article{bastani2021mostly,
  title = {Mostly Exploration-Free Algorithms for Contextual Bandits},
  author = {Bastani, Hamsa and Bayati, Mohsen and Khosravi, Khashayar},
  journal = {Management Science},
  volume = {67},
  number = {3},
  pages = {1329--1349},
  year = {2021}
}

@inproceedings{cesabianchi2017boltzmann,
  title = {{Boltzmann} Exploration Done Right},
  author = {Cesa-Bianchi, Nicol\`{o} and Gentile, Claudio and Lugosi, Gabor and Neu, Gergely},
  booktitle = {Advances in Neural Information Processing Systems},
  volume = {30},
  year = {2017}
}

@inproceedings{chowdhury2022shuffle,
  title = {Shuffle Private Linear Contextual Bandits},
  author = {Chowdhury, Sayak Ray and Zhou, Xingyu},
  booktitle = {Proceedings of the 39th International Conference on Machine Learning},
  volume = {162},
  pages = {3984--4009},
  year = {2022},
  series = {Proceedings of Machine Learning Research},
  publisher = {PMLR}
}

@inproceedings{garcelon2021localrl,
  title = {Local Differential Privacy for Regret Minimization in Reinforcement Learning},
  author = {Garcelon, Evrard and Perchet, Vianney and Pike-Burke, Ciara and Pirotta, Matteo},
  booktitle = {Advances in Neural Information Processing Systems},
  volume = {34},
  pages = {10561--10573},
  year = {2021}
}

@inproceedings{garcelon2022shuffle,
  title = {Privacy Amplification via Shuffling for Linear Contextual Bandits},
  author = {Garcelon, Evrard and Chaudhuri, Kamalika and Perchet, Vianney and Pirotta, Matteo},
  booktitle = {Proceedings of the 33rd International Conference on Algorithmic Learning Theory},
  volume = {167},
  pages = {381--407},
  year = {2022},
  series = {Proceedings of Machine Learning Research},
  publisher = {PMLR}
}

@article{hanna2024private,
  title = {Differentially Private Stochastic Linear Bandits: (Almost) for Free},
  author = {Hanna, Osama and Girgis, Antonious M. and Fragouli, Christina and Diggavi, Suhas},
  journal = {IEEE Journal on Selected Areas in Information Theory},
  volume = {5},
  pages = {135--147},
  year = {2024}
}

@inproceedings{kannan2018smoothed,
  title = {A Smoothed Analysis of the Greedy Algorithm for the Linear Contextual Bandit Problem},
  author = {Kannan, Sampath and Morgenstern, Jamie H. and Roth, Aaron and Waggoner, Bo and Wu, Zhiwei Steven},
  booktitle = {Advances in Neural Information Processing Systems},
  volume = {31},
  year = {2018}
}

@inproceedings{krishnamurthy2021misspecification,
  title = {Adapting to Misspecification in Contextual Bandits with Offline Regression Oracles},
  author = {Krishnamurthy, Sanath Kumar and Hadad, Vitor and Athey, Susan},
  booktitle = {Proceedings of the 38th International Conference on Machine Learning},
  volume = {139},
  pages = {5805--5814},
  year = {2021},
  series = {Proceedings of Machine Learning Research},
  publisher = {PMLR}
}

@inproceedings{luo2021dilated,
  title = {Policy Optimization in Adversarial {MDP}s: Improved Exploration via Dilated Bonuses},
  author = {Luo, Haipeng and Wei, Chen-Yu and Lee, Chung-Wei},
  booktitle = {Advances in Neural Information Processing Systems},
  volume = {34},
  pages = {22931--22942},
  year = {2021}
}

@inproceedings{pavlovic2025projections,
  title = {Differential Privacy in Kernelized Contextual Bandits via Random Projections},
  author = {Pavlovic, Nikola and Salgia, Sudeep and Zhao, Qing},
  booktitle = {Proceedings of the 29th International Conference on Artificial Intelligence and Statistics},
  volume = {300},
  pages = {1135--1143},
  year = {2026},
  series = {Proceedings of Machine Learning Research},
  publisher = {PMLR}
}

@inproceedings{qiao2023private,
  title = {Near-Optimal Differentially Private Reinforcement Learning},
  author = {Qiao, Dan and Wang, Yu-Xiang},
  booktitle = {Proceedings of the 26th International Conference on Artificial Intelligence and Statistics},
  volume = {206},
  pages = {9914--9940},
  year = {2023},
  series = {Proceedings of Machine Learning Research},
  publisher = {PMLR}
}

@inproceedings{qin2026monster,
  title = {Taming the Monster Every Context: Complexity Measure and Unified Framework for Offline-Oracle Efficient Contextual Bandits},
  author = {Qin, Hao and Zhang, Chicheng},
  booktitle = {Proceedings of Thirty Ninth Conference on Learning Theory},
  volume = {336},
  pages = {5399--5464},
  year = {2026},
  series = {Proceedings of Machine Learning Research},
  publisher = {PMLR}
}

@article{sarmasarkar2026shuffle,
  title = {Shuffle and Joint Differential Privacy for Generalized Linear Contextual Bandits},
  author = {Sarmasarkar, Sahasrajit},
  journal = {arXiv preprint arXiv:2602.00417},
  year = {2026}
}

@inproceedings{tenenbaum2023concurrent,
  title = {Concurrent Shuffle Differential Privacy Under Continual Observation},
  author = {Tenenbaum, Jay and Kaplan, Haim and Mansour, Yishay and Stemmer, Uri},
  booktitle = {Proceedings of the 40th International Conference on Machine Learning},
  volume = {202},
  pages = {33961--33982},
  year = {2023},
  series = {Proceedings of Machine Learning Research},
  publisher = {PMLR}
}

@inproceedings{dwork2006calibrating,
  title = {Calibrating noise to sensitivity in private data analysis},
  author = {Dwork, Cynthia and McSherry, Frank and Nissim, Kobbi and Smith, Adam},
  booktitle = {Theory of Cryptography},
  volume = {3876},
  pages = {265--284},
  year = {2006},
  series = {Lecture Notes in Computer Science},
  publisher = {Springer}
}

@inproceedings{kearns2014mechanism,
  title = {Mechanism design in large games: Incentives and privacy},
  author = {Kearns, Michael and Pai, Mallesh M. and Roth, Aaron and Ullman, Jonathan},
  booktitle = {Proceedings of the 5th Conference on Innovations in Theoretical Computer Science},
  pages = {403--410},
  year = {2014}
}

@inproceedings{hsu2014private,
  title = {Private matchings and allocations},
  author = {Hsu, Justin and Huang, Zhiyi and Roth, Aaron and Roughgarden, Tim and Wu, Zhiwei Steven},
  booktitle = {Proceedings of the 46th Annual ACM Symposium on Theory of Computing},
  pages = {21--30},
  year = {2014}
}

@inproceedings{wang2018revisiting,
  title = {Revisiting Differentially Private Linear Regression: Optimal and Adaptive Prediction \& Estimation in Unbounded Domain},
  author = {Wang, Yu-Xiang},
  booktitle = {Proceedings of the 34th Conference on Uncertainty in Artificial Intelligence},
  pages = {93--103},
  year = {2018},
  publisher = {AUAI Press}
}

@article{liu1989limited,
  title = {On the Limited Memory {BFGS} Method for Large Scale Optimization},
  author = {Liu, Dong C. and Nocedal, Jorge},
  journal = {Mathematical Programming},
  volume = {45},
  pages = {503--528},
  year = {1989}
}

@inproceedings{foster2021instance,
  title     = {Instance-Dependent Complexity of Contextual Bandits and Reinforcement Learning: A Disagreement-Based Perspective},
  author    = {Foster, Dylan and Rakhlin, Alexander and Simchi-Levi, David and Xu, Yunzong},
  booktitle = {Proceedings of the 34th Annual Conference on Learning Theory},
  pages     = {2059--2059},
  year      = {2021},
  volume    = {134},
  series    = {Proceedings of Machine Learning Research},
  publisher = {PMLR}
}
\clearpage
\appendix

\section{Additional related work}
\label{app:related-work}

\paragraph{Contextual bandits with optimization oracles.}
Oracle reductions separate exploration from the computational task
of learning a predictor or policy. \citet{agarwal2014} obtain
near-optimal regret using a cost-sensitive classification oracle
and optimization over policy distributions. Under realizability,
\citet{foster2018regression} develop regression-based exploration
methods, while \citet{falcon} obtain optimal rates through inverse-gap
sampling. Their FALCON algorithm uses finite-class least-squares
regression; FALCON+ accommodates a general offline regression oracle
trained on the preceding epoch. Our batched methods share this
fixed-policy data-collection structure, but construct the next policy
through averaged exponential updates. \citet{xu2020uccb} develop
counterfactual confidence bounds. Further work studies robustness
to misspecification \citep{krishnamurthy2021misspecification} and
offline-oracle reductions for large action spaces
\citep{qin2026monster}. The empirical study of
\citet{bietti2021bakeoff} motivates comparisons across oracle-based
methods beyond their worst-case regret guarantees.

Online regression provides a complementary interface. SquareCB
\citep{foster2020squarecb} reduces contextual bandits to online
square-loss regression and handles adversarial contexts under
realizability, with a guarantee expressed through the online oracle's
regret. Our vanilla result instead uses an offline regression
guarantee averaged over past data-collection policies. The common
stochastic context distribution connects this guarantee to the
prediction errors affecting future policy updates. The distinction
therefore concerns both the oracle interface and the context model.

\paragraph{Exploration in policy optimization.}
The empirical success of PO, exemplified by PPO
\citep{schulman2017ppo}, is complemented by theoretical analyses of
exponential policy updates and natural policy-gradient methods
\citep{abbasi2019politex,agarwal2021theory}.
Optimistic PO algorithms obtain regret guarantees in tabular and
structured Markov decision processes
\citep{shani2020optimistic,cai2020exploration}, including adversarial
settings through exploration bonuses \citep{luo2021dilated}.
For stochastic contextual bandits with general offline function
approximation, \citet{levy2026} prove near-optimal regret for PO
with a bonus based on accumulated action probabilities.
Our first result establishes that the vanilla exponential update
already suffices under stochastic contexts, realizability, and the
stated regression-oracle guarantee. Our bonus-based result is a
batched private adaptation of their method. These results concern
the contextual-bandit setting; bonuses can still improve empirical
performance, as our experiments illustrate.

\paragraph{Exploration supplied by the context distribution.}
Prior work shows that greedy contextual-bandit algorithms can succeed
under distributional conditions such as covariate diversity or
smoothed contexts \citep{bastani2021mostly,kannan2018smoothed}.
Our algorithm uses gradual exponential updates rather than greedy
action selection, and its guarantee does not require these additional
conditions on the context distribution. It also differs from
Boltzmann exploration based on current empirical arm means, whose
limitations are analyzed by \citet{cesabianchi2017boltzmann}:
unrolling our updates gives weights based on the sum of successive
fitted predictors.

\paragraph{Private bandits and privacy models.}
\citet{sajed2019private} study private stochastic multi-armed bandits
using private stopping rules and successive elimination.
\citet{hanna2024private} study stochastic linear bandits under
central, local, and shuffle privacy.
For contextual bandits, \citet{shariff2018private} establish regret
guarantees for linear models under joint differential privacy,
and \citet{chen2025private} connect private linear regression to
private contextual-bandit learning. Local privacy for contextual
bandits, including generalized linear models, is studied by
\citet{zheng2020local,han2021generalized}.
Shuffle-private linear contextual bandits are studied by
\citet{garcelon2022shuffle,chowdhury2022shuffle}.
\citet{tenenbaum2023concurrent} introduce the concurrent shuffle
model, obtain improved guarantees for continual summation, and
apply them to contextual linear bandits.
\citet{sarmasarkar2026shuffle} studies shuffle and joint privacy
for generalized linear contextual bandits.
Kernel methods provide another extension beyond linear models:
\citet{pavlovic2025private} study private policy learning in
kernelized contextual bandits, and
\citet{pavlovic2025projections} establish cumulative-regret
guarantees using private kernel regression.

\paragraph{Private reinforcement learning.}
\citet{vietri2020private} establish PAC and regret guarantees for
private tabular reinforcement learning, in the fundamental model of tabular Markov Decision Process (MDP).
Later, \citet{garcelon2021localrl} study regret minimization under local
privacy in MDPs, protecting trajectories on the user's side.
\citet{chowdhury2022differentially} develop private policy-optimization
and value-iteration methods for episodic tabular MDPs under joint
and local privacy, while \citet{qiao2023private} obtain improved
regret guarantees in these models.
\citet{zhou2022privaterl} studies private exploration with linear
function approximation in linear-mixture MDPs.
The policy-optimization for MDPs methods of
\citet{chowdhury2022differentially,zhou2022privaterl} use optimistic
value estimates, including exploration bonuses.
Beyond particular parametric models,
\citet{chenrakhlin2025decision} develop a general framework for
decision making that includes local differential privacy.
\citet{he2026general} study online reinforcement learning in
finite-horizon episodic MDPs with general value-function approximation
under joint differential privacy. Their approach combines batched
policy updates with the exponential mechanism and achieves
$\widetilde{O}(T^{3/5})$ dependence on the number of episodes $T$,
with additional dependence on coverability, function-class complexity,
the horizon, and privacy parameters. In contrast, we study realizable
stochastic contextual bandits and analyze vanilla policy optimization,
whose policy updates rely solely on exponential weighting of predicted
losses. We obtain near-optimal $\widetilde{O}(\sqrt{T})$ dependence on
the number of rounds $T$, without coverability assumptions or additional
structural complexity measures: the dependence on the loss-prediction
class is captured entirely by its regression-learning complexity.
We also consider a different privacy model, requiring differential
privacy of the complete released transcript of predictors and policy
maps rather than joint differential privacy.

\paragraph{Private PO through regression.}
Especially close to our approach, \citet{he2025privatepo} analyze
private PG, NPG, and REBEL, including reductions to private
regression using fresh batches. We share their principle of
using each batch in one private update.
Their guarantees include sample-complexity and average
policy-suboptimality bounds.
For NPG, regression estimates advantages using policy-gradient
features, and the guarantee depends on coverage by a base policy;
their sampling procedure obtains feedback for two actions at
each context.
Our algorithms instead fit expected losses from one observed
action loss per arriving context.
We analyze how the resulting batched policies provide sufficient
exploration to control cumulative regret over all interactions,
under realizability and the stated regression-oracle guarantee.

\paragraph{Scope of our privacy and computational guarantees.}
Our guarantee is central differential privacy for the complete
sequence of released predictors and policy maps under replacement
of one user's context and potential-loss vector. Actions and
evaluations at protected contexts remain internal.
This release convention differs from joint privacy for user-specific
actions \citep{shariff2018private,vietri2020private}, local privacy
for individual messages, and privacy for shuffled messages.
Comparisons must therefore account for the protected records,
observable outputs, and trust assumptions.
Our dependence on the prediction class enters through the private
regression oracle's accuracy. Private empirical risk minimization
provides relevant computational tools under suitable conditions on
the loss and parameter space
\citep{chaudhuri2011private,bassily2014privateerm}.
Computational efficiency additionally requires an efficient
implementation of the oracle; in particular,
our finite-class exponential-mechanism instantiation does not assume
efficient sampling over an arbitrary function class.

\section{Concentration of pseudo-regret}
\label{app:concentration}

Our regret analyses control
$\sum_{t=1}^T\mathcal R(\pi_t)$, which evaluates each learned
policy under the context distribution $\D$.
The regret $\Reg$ instead evaluates each policy at the context
observed in that round, while still averaging over the current
action. The policies remain random in both quantities.
The following proposition bounds their difference.

\begin{proposition}
\label{cor:mean-regret}
Suppose that each policy $\pi_t$ is chosen using
$\mathcal H_{t-1}$, as in \cref{sec:setup}.
For every $\delta\in(0,1)$, with probability at least $1-\delta$,
\begin{equation*}
 \Reg
 \le\sum_{t=1}^T\mathcal R(\pi_t)
   +\sqrt{\frac T2\log\frac1\delta}.
\end{equation*}
\end{proposition}

\begin{proof}
Conditional on $\mathcal H_{t-1}$, the policy $\pi_t$ is fixed
and the fresh context $c_t$ has distribution $\D$.
Therefore,
\begin{equation*}
 \E\!\left[
   \sum_{a\in\A}\pi_t(c_t,a)\Delta(c_t,a)
   \,\middle|\,\mathcal H_{t-1}
 \right]
 =\mathcal R(\pi_t).
\end{equation*}
Since every gap lies in $[0,1]$ and the action probabilities
sum to one,
\begin{equation*}
 0\le\sum_{a\in\A}\pi_t(c_t,a)\Delta(c_t,a)\le1.
\end{equation*}
Conditional Hoeffding's lemma thus gives, for every $\lambda>0$,
\begin{equation*}
 \E\!\left[
   \exp\!\left\{
     \lambda\left(
       \sum_{a\in\A}\pi_t(c_t,a)\Delta(c_t,a)
       -\mathcal R(\pi_t)
     \right)
   \right\}
   \,\middle|\,\mathcal H_{t-1}
 \right]
 \le e^{\lambda^2/8}.
\end{equation*}
Iterating this inequality over rounds yields
\begin{equation*}
 \E\!\left[
   \exp\!\left\{
     \lambda\left(
       \Reg-\sum_{t=1}^T\mathcal R(\pi_t)
     \right)
   \right\}
 \right]
 \le e^{T\lambda^2/8}.
\end{equation*}
By exponential Markov's inequality, for every $u>0$,
\begin{equation*}
 \Pp\!\left(
   \Reg-\sum_{t=1}^T\mathcal R(\pi_t)\ge u
 \right)
 \le
 \inf_{\lambda>0}
 \exp\!\left(-\lambda u+\frac{T\lambda^2}{8}\right)
 =e^{-2u^2/T},
\end{equation*}
where the minimum is attained at $\lambda=4u/T$.
Taking $u=\sqrt{(T/2)\log(1/\delta)}$ proves the claim.
\end{proof}

To combine this proposition with a high-probability bound on
$\sum_{t=1}^T\mathcal R(\pi_t)$, apply each result with failure
probability $\delta/2$ and take a union bound.
The resulting bound on $\Reg$ holds with probability at least
$1-\delta$ and adds
$\sqrt{(T/2)\log(2/\delta)}$
to the bound obtained at confidence level $1-\delta/2$.
This additional term does not change the stated finite-class
regret rates in $\widetilde O$ notation.

These regret quantities are internal performance measures and
are not part of the private algorithms' released transcript.

\section{Finite-class regression guarantees}
\label{app:regression}

We establish the regression guarantees used by vanilla PO and its
batched variants. The first result handles adaptive action selection
and holds simultaneously over all rounds. The remaining results
concern a single batch collected under a fixed policy.

\subsection{Least-squares ERM under adaptive data collection}

\begin{lemma}[Least-squares ERM under adaptive data collection]
\label{lem:erm-oracle}
In the model of \cref{sec:setup}, suppose that $\F$ is finite
and $f^\star\in\F$.
For every round $t$, let $\widehat f_t$ minimize
\begin{equation*}
 \sum_{r=1}^t
 \bigl(f(c_r,a_r)-\ell_r(c_r,a_r)\bigr)^2
\end{equation*}
over $f\in\F$, with ties broken by a fixed ordering.
Then, for every $\delta\in(0,1)$, \cref{ass:oracle} holds with
\begin{equation*}
 \Err_{\rm orc}(T,\delta)
 =24\log\frac{|\F|}{\delta}.
\end{equation*}
\end{lemma}

\begin{proof}
Fix $f\in\F$ and define its excess squared loss at round $r$ by
\begin{equation*}
 Z_r(f)
 =\bigl(f(c_r,a_r)-\ell_r(c_r,a_r)\bigr)^2
  -\bigl(f^\star(c_r,a_r)-\ell_r(c_r,a_r)\bigr)^2.
\end{equation*}

Conditional on $\mathcal H_{r-1}$, the policy $\pi_r$ is fixed.
The current context--loss pair is independent of this history
and follows the distribution specified in \cref{sec:setup}.
Moreover, the action is sampled from $\pi_r(c_r,\cdot)$ using
fresh randomness, without observing the current losses.
For each fixed action, realizability implies that the expected
excess squared loss at context $c$ is
$\bigl(f(c,a)-f^\star(c,a)\bigr)^2$:
expanding the squares cancels the term involving the centered loss.
Averaging over the action and context therefore gives
\begin{align*}
 \nu_r(f)
 &:=\E[Z_r(f)\mid\mathcal H_{r-1}]\\
 &=\E_{c\sim\D}\!\left[
   \sum_{a\in\A}\pi_r(c,a)
   \bigl(f(c,a)-f^\star(c,a)\bigr)^2
 \right]
 =\|f-f^\star\|_{\pi_r}^2.
\end{align*}
Let $V_t(f)=\sum_{r=1}^t\nu_r(f)$.

Since losses and predictions lie in $[0,1]$, we have
$|Z_r(f)|\le1$. Factoring the difference of squares also gives
\begin{align*}
 Z_r(f)^2
 &=
 \bigl(f(c_r,a_r)-f^\star(c_r,a_r)\bigr)^2
 \bigl(f(c_r,a_r)+f^\star(c_r,a_r)
       -2\ell_r(c_r,a_r)\bigr)^2\\
 &\le4\bigl(f(c_r,a_r)-f^\star(c_r,a_r)\bigr)^2.
\end{align*}
Thus $X_r(f)=\nu_r(f)-Z_r(f)$ satisfies
\begin{equation*}
 \E[X_r(f)\mid\mathcal H_{r-1}]=0,
 \qquad
 |X_r(f)|\le2,
 \qquad
 \E[X_r(f)^2\mid\mathcal H_{r-1}]
 \le4\nu_r(f).
\end{equation*}

Take $\lambda=1/12$. The conditional Bernstein moment bound yields
\begin{equation*}
 \log\E[e^{\lambda X_r(f)}\mid\mathcal H_{r-1}]
 \le
 \frac{\lambda^2
       \E[X_r(f)^2\mid\mathcal H_{r-1}]}
      {2(1-2\lambda/3)}
 \le\frac{\lambda}{2}\nu_r(f).
\end{equation*}
Consequently,
\begin{equation*}
 M_t(f)
 =\exp\left\{
   \lambda\left(
     \frac12V_t(f)-\sum_{r=1}^tZ_r(f)
   \right)
 \right\}
\end{equation*}
is a nonnegative supermartingale with $M_0(f)=1$.
By Ville's inequality, for every $x>0$, with probability at
least $1-e^{-x}$, simultaneously for all $t\le T$,
\begin{equation}
 \sum_{r=1}^tZ_r(f)
 \ge\frac12V_t(f)-12x.
 \label{eq:excess-lower}
\end{equation}
Indeed, a violation would imply that $M_t(f)>e^x$ at some round,
whose probability is at most $e^{-x}$; see \citet{howard-2020}.

Set $x=\log(|\F|/\delta)$ and take a union bound over $f\in\F$.
With probability at least $1-\delta$, \eqref{eq:excess-lower}
holds for every $f\in\F$ and every $t\le T$.
On this event, it applies to the data-dependent choice
$f=\widehat f_t$. The ERM property and $f^\star\in\F$ imply
\begin{equation*}
 0
 \ge\sum_{r=1}^tZ_r(\widehat f_t)
 \ge\frac12V_t(\widehat f_t)-12x.
\end{equation*}
Therefore,
\begin{equation*}
 \sum_{r=1}^t
 \|\widehat f_t-f^\star\|_{\pi_r}^2
 =V_t(\widehat f_t)
 \le24\log\frac{|\F|}{\delta}
\end{equation*}
simultaneously for all $t\le T$.
This proves \cref{ass:oracle} without a union bound over rounds.
\end{proof}

\subsection{Regression on a fixed batch and approximate ERM}

Fix a policy $\pi$. Each observation is generated by drawing a fresh
context--loss pair in the model of \cref{sec:setup}, sampling
an action from $\pi$ using independent randomness, and observing
its loss.
For an ordered batch $\mathcal S=((c_i,a_i,y_i))_{i=1}^n$ of
such observations, define
\begin{equation*}
 L_{\mathcal S}(f)
 =\sum_{i=1}^n\bigl(f(c_i,a_i)-y_i\bigr)^2.
\end{equation*}

The following result allows the fitted predictor to approximately
minimize the empirical squared loss. This covers both exact ERM
and the private oracle analyzed below.

\begin{lemma}[Accuracy of approximate ERM on a batch]
\label{lem:private-fast-rate}
Let $\mathcal S$ contain $n\ge1$ i.i.d. observations generated
under a fixed policy $\pi$, and suppose that $\F$ is finite
and $f^\star\in\F$.
Fix $\delta\in(0,1)$ and $\alpha\ge0$.
Suppose that a possibly randomized $\widetilde f\in\F$ satisfies
\begin{equation*}
 L_{\mathcal S}(\widetilde f)
 \le\min_{f\in\F}L_{\mathcal S}(f)+\alpha
\end{equation*}
with probability at least $1-\delta/2$ over its randomness,
conditional on each dataset $\mathcal S$.
Then, jointly over the observations and oracle randomness,
with probability at least $1-\delta$,
\begin{equation*}
 n\|\widetilde f-f^\star\|_\pi^2
 \le2\alpha+10\log\frac{2|\F|}{\delta}.
\end{equation*}
\end{lemma}

\begin{proof}
For a fixed $f\in\F$, define
\begin{equation*}
 Z_i(f)
 =\bigl(f(c_i,a_i)-y_i\bigr)^2
  -\bigl(f^\star(c_i,a_i)-y_i\bigr)^2,
 \qquad
 \mu_f=\E[Z_i(f)]=\|f-f^\star\|_\pi^2.
\end{equation*}
The equality for $\mu_f$ follows by expanding the squares
and applying realizability, as in \cref{lem:erm-oracle}.
The same boundedness calculation gives
\begin{equation*}
 |Z_i(f)|\le1,
 \qquad
 \operatorname{Var}(Z_i(f))\le4\mu_f,
 \qquad
 |Z_i(f)-\mu_f|\le2.
\end{equation*}

Bernstein's inequality implies that, for every $u>0$,
with probability at least $1-e^{-u}$,
\begin{equation*}
 \sum_{i=1}^nZ_i(f)
 \ge n\mu_f-\sqrt{8n\mu_f u}-\frac{2u}{3}
 \ge\frac{n\mu_f}{2}-\frac{14u}{3}.
\end{equation*}
The second inequality uses
$\sqrt{8xu}\le x/2+4u$.
Rearranging gives
\begin{equation*}
 n\mu_f
 \le2\sum_{i=1}^nZ_i(f)+\frac{28u}{3}
 \le2\sum_{i=1}^nZ_i(f)+10u.
\end{equation*}

Set $u=\log(2|\F|/\delta)$ and take a union bound over $f\in\F$.
With probability at least $1-\delta/2$, this inequality holds
for every candidate, including the randomly selected
$\widetilde f$.
On the oracle's success event,
\begin{equation*}
 \sum_{i=1}^nZ_i(\widetilde f)
 =L_{\mathcal S}(\widetilde f)-L_{\mathcal S}(f^\star)
 \le\alpha,
\end{equation*}
because $f^\star\in\F$.
The oracle's failure probability is at most $\delta/2$
conditional on every dataset, and hence also unconditionally.
A union bound over the two failure events proves the claim.
\end{proof}

For ordinary least-squares ERM, the empirical optimality
condition holds deterministically with $\alpha=0$.
The lemma therefore gives $n\|\widetilde f-f^\star\|_\pi^2
\le10\log(2|\F|/\delta)$ with probability at least $1-\delta$.

If $\pi$ is chosen using earlier observations, the same result
applies conditional on the history before collecting the batch.
This fixes $\pi$; the new context--loss pairs and fresh action
randomness then generate i.i.d. observations under that policy.
The application across epochs is given in \cref{app:batched}.

\subsection{Exponential-mechanism instantiation}

\begin{lemma}[Privacy and accuracy of finite-class regression]
\label{lem:private-exp-oracle}
Let $\F\subseteq[0,1]^{\C\times\A}$ be a finite prediction
class and fix $\varepsilon_{\rm priv}>0$.
On an ordered batch $\mathcal S$ of $n\ge1$ observations in
$\C\times\A\times[0,1]$,
consider the exponential mechanism
\begin{equation}
 \Pp(\widetilde f=f\mid\mathcal S)
 =\frac{\exp\{-\varepsilon_{\rm priv}L_{\mathcal S}(f)/2\}}
 {\sum_{g\in\F}\exp\{-\varepsilon_{\rm priv}L_{\mathcal S}(g)/2\}},
 \qquad f\in\F.
 \label{eq:em}
\end{equation}

This oracle is pure $\varepsilon_{\rm priv}$-DP under
replacement of one observation.
If $\mathcal S$ consists of i.i.d. observations generated
under a fixed policy $\pi$ and $f^\star\in\F$, then,
for every $\delta\in(0,1)$, with probability at least
$1-\delta$,
\begin{equation*}
 n\|\widetilde f-f^\star\|_\pi^2
 \le
 \left(10+\frac4{\varepsilon_{\rm priv}}\right)
 \log\frac{2|\F|}{\delta}.
\end{equation*}
The probability is over the observations and the oracle's
randomness.
\end{lemma}

\begin{proof}
We first prove privacy, then show that the sampled predictor
has sufficiently small empirical loss to apply
\cref{lem:private-fast-rate}.

\paragraph{Privacy.}
Fix neighboring batches $\mathcal S,\mathcal S'$ of the
same size.
Each squared loss lies in $[0,1]$, so
$|L_{\mathcal S}(f)-L_{\mathcal S'}(f)|\le1$
for every $f\in\F$.
Let
\begin{equation*}
 \mathcal Z_{\mathcal S}
 =\sum_{f\in\F}
   \exp\{-\varepsilon_{\rm priv}L_{\mathcal S}(f)/2\}
\end{equation*}
denote the normalizing constant.
For every $f\in\F$,
\begin{equation*}
 \frac{\Pp(\widetilde f=f\mid\mathcal S)}
      {\Pp(\widetilde f=f\mid\mathcal S')}
 =
 \exp\left\{
   -\frac{\varepsilon_{\rm priv}}2
   \bigl(L_{\mathcal S}(f)-L_{\mathcal S'}(f)\bigr)
 \right\}
 \frac{\mathcal Z_{\mathcal S'}}{\mathcal Z_{\mathcal S}}
 \le e^{\varepsilon_{\rm priv}}.
\end{equation*}
The sensitivity bound makes each factor at most
$e^{\varepsilon_{\rm priv}/2}$.
For the normalizing constants, this follows by comparing
the corresponding summands.
Summing over any set of outputs proves pure differential
privacy.
This argument applies to arbitrary batches and requires
no distributional assumption \citep{dwork-roth-2014}.

\paragraph{Regression accuracy.}
The exponential mechanism favors predictors with smaller
empirical squared loss.
Write
$L_{\mathcal S}^{\min}=\min_{f\in\F}L_{\mathcal S}(f)$.
Conditional on any batch $\mathcal S$, for every $\alpha\ge0$,
\begin{equation*}
 \Pp\!\left(
   L_{\mathcal S}(\widetilde f)>
   L_{\mathcal S}^{\min}+\alpha
   \,\middle|\,\mathcal S
 \right)
 \le|\F|e^{-\varepsilon_{\rm priv}\alpha/2}.
\end{equation*}
Indeed, the total weight of predictors exceeding this
threshold is at most
$|\F|\exp\{-\varepsilon_{\rm priv}
(L_{\mathcal S}^{\min}+\alpha)/2\}$,
whereas the normalizing constant is at least
$\exp\{-\varepsilon_{\rm priv}L_{\mathcal S}^{\min}/2\}$.

Taking
$\alpha=(2/\varepsilon_{\rm priv})
\log(2|\F|/\delta)$
makes this conditional failure probability at most
$\delta/2$.
Thus the oracle satisfies the approximate-ERM requirement
of \cref{lem:private-fast-rate}.
That lemma gives, with probability at least $1-\delta$,
\begin{equation*}
 n\|\widetilde f-f^\star\|_\pi^2
 \le
 2\alpha+10\log\frac{2|\F|}{\delta}
 =
 \left(10+\frac4{\varepsilon_{\rm priv}}\right)
 \log\frac{2|\F|}{\delta}.
\end{equation*}
\end{proof}

The application to all fitted predictors, including the choice of
$\Err_{\rm orc}^{\rm batch}(T,\delta)$, is given in
\cref{app:batch-oracle-parameters}.

\section{Complete proof for vanilla PO}\label{app:regret-proof}

We begin by restating the algorithm and recording the identities used throughout the proof.
\\

\noindent
\fbox{\begin{minipage}{\dimexpr\linewidth-2\fboxsep-2\fboxrule\relax}
\textbf{Vanilla Policy Optimization (\textsc{VanillaPO})}

\begin{algorithmic}
\State \textbf{Input:} horizon $T$, action set $\A$, learning rate
$\eta\in(0,1]$, and an offline regression oracle satisfying
\cref{ass:oracle}.
\State Initialize $\pi_1(c,a)\gets 1/A$ for every $c\in\C$ and
$a\in\A$, and let $\mathcal S_0\gets()$.
\State \textbf{for} $t=1,\ldots,T$ \textbf{do}
\State \AlgIndent Observe the context $c_t$.
\State \AlgIndent Sample $a_t\sim\pi_t(c_t,\cdot)$ and observe
$\ell_t(c_t,a_t)$.
\State \AlgIndent Append the observation:
$
 \mathcal S_t\gets
 \operatorname{append}\!\left(
   \mathcal S_{t-1},
   (c_t,a_t,\ell_t(c_t,a_t))
 \right).
$
\State \AlgIndent Fit the predictor:
$\widehat f_t\gets\operatorname{Oracle}(\mathcal S_t)$.
\State \AlgIndent Define the next policy, for every $c\in\C$ and
$a\in\A$, by
\begin{equation}
 \pi_{t+1}(c,a)=
 \frac{\pi_t(c,a)\exp\{-\eta\widehat f_t(c,a)\}}
 {\sum_{b\in\A}\pi_t(c,b)\exp\{-\eta\widehat f_t(c,b)\}}.
 \label{eq:vanilla-update}
\end{equation}
\State \textbf{end for}
\end{algorithmic}
\end{minipage}}
\\

Recall that $\Delta(c,a)=f^\star(c,a)-f^\star(c,a^\star(c))$ is the gap of action $a$. For each action $a\in\A$, define
\begin{equation*}
    e_{t,a}(c)=\widehat f_t(c,a)-f^\star(c,a),
\end{equation*}
and, for $0\le t\le T$, let
\begin{equation*}
    D_{t,a}(c)=\sum_{r=1}^t
    \bigl(e_{r,a}(c)-e_{r,a^\star(c)}(c)\bigr),\qquad D_{0,a}(c)=0.
\end{equation*}
Here, $D_{t,a}(c)$ measures the accumulated prediction error through round $t$ in estimating the loss gap between action $a$ and the optimal action $a^{\star}(c)$. Starting from the uniform policy and unrolling the exponential update in \eqref{eq:vanilla-update} gives, for $1\le t\le T+1$,
\begin{equation*}
 \frac{\pi_t(c,a)}{\pi_t(c,a^\star(c))}
 =\exp\{-\eta((t-1)\Delta(c,a)+D_{t-1,a}(c))\}.
\end{equation*}
Define the relative action weights for round $t$ by
\begin{equation*}
 W_{t,a}(c)
 =\exp\{-\eta((t-1)\Delta(c,a)+D_{t-1,a}(c))\}.
\end{equation*}
In particular, $W_{t,a^\star(c)}(c)=1$, and
\begin{equation}
 \pi_t(c,a)
 =\frac{W_{t,a}(c)}{\sum_{b\in\A}W_{t,b}(c)} \le W_{t,a}(c).
 \label{eq:weight-policy}
\end{equation}
To measure the errors under past policies, define, for $c \in \C$, $a\in\A$, and $1\le t\le T$,
\begin{align*}
    \bar\pi_t(c,a)=\frac{N_{t,a}(c)}{t},\qquad \text{where} \qquad N_{t,a}(c)=\sum_{r=1}^t\pi_r(c,a).
\end{align*}
We denote the mean squared prediction error of $\widehat f_t$ under $\bar\pi_t(c,.)$ and the corresponding cumulative mean squared prediction error by
\begin{align*}
    s_t(c) =\sum_{a\in\A}\bar\pi_t(c,a)e_{t,a}(c)^2, \qquad S_t(c) =\sum_{r=1}^t s_r(c).
\end{align*}
Let $R_T(c)$ denote the cumulative regret obtained by evaluating
every policy in the learned sequence at the same fixed context $c$:
\begin{equation*}
 R_T(c)=\sum_{t=1}^T\sum_{a\in\A}\pi_t(c,a)\Delta(c,a).
\end{equation*}
Holding the learned policies fixed and averaging over a fresh
context gives
\begin{equation*}
 \E_{c\sim\D}[R_T(c)]
 =\sum_{t=1}^T\mathcal R(\pi_t).
\end{equation*}
Throughout this appendix, let
\begin{align*}
 \Lambda=1+\log(AT),
 \qquad H_T=\sum_{t=1}^T\frac1t,
 \qquad C=16.
\end{align*}

\subsection{Proof of \cref{thm:general-oracle}}

We prove \cref{thm:general-oracle} by first bounding $R_T(c)$
at each context $c$ in terms of $S_T(c)$.
Averaging over contexts and applying \cref{ass:oracle}
then bounds $\sum_{t=1}^T\mathcal R(\pi_t)$.
Finally, the concentration argument in \cref{app:concentration}
transfers this bound to $\Reg$.
Each summand $\pi_t(c,a)\Delta(c,a)$ in $R_T(c)$
can be bounded through either the loss gap $\Delta(c,a)$ or the
action probability $\pi_t(c,a)$. The following dichotomy is the
central tool for doing so.

\begin{lemma}[Late-round probability--error dichotomy]
\label{lem:activation}
Fix a context $c \in \mathcal{C}$, a suboptimal action $a \in \mathcal{A}$ with $\Delta(c,a)>0$, and a round $t\in [T]$ satisfying $t-1 \ge \frac{4\Lambda}{\eta\Delta(c,a)}$. Then
\begin{equation*}
 \pi_t(c,a) < \frac1{AT} \qquad \text{or} \qquad  \Delta(c,a) \le C^2A\eta\, S_{t-1}(c).
\end{equation*}
\end{lemma}
The lemma identifies the only two ways a suboptimal action can retain non-negligible probability: its gap is small relative to the accumulated prediction error, or the round is too early in the game. Carefully choosing the learning rate keeps the cumulative cost of both possibilities under control.

We now deduce the theorem from \cref{lem:activation}, with the rest of this appendix devoted to proving it.

\begin{proof}[Proof of \cref{thm:general-oracle}]
Fix a sequence of predictors $\widehat f_1, \ldots, \widehat f_T$ and consequently policies $\pi_1, \ldots, \pi_T$. Fix also a context $c\in\C$; we first derive a deterministic bound on $R_T(c)$ in terms of $S_T(c)$.

Define the error-dependent gap threshold $\tau = C^2 A \eta S_T(c)$, and, for each action $a\in\A$, define its action-dependent time scale by $T_a = \frac{4\Lambda}{\eta\Delta(c,a)} + 1$ with $T_a = \infty$ when $\Delta(c,a) = 0$. We partition the action--round pairs $\mathcal{A} \times [T]$ as
\begin{align*}
    \underbrace{\left\{(a,t):\Delta(c,a)\le\tau\right\}}_{\mathcal S_1}
    \,\cup\, \underbrace{\left\{(a,t): \Delta(c,a)>\tau,\, t\le T_a \right\}}_{\mathcal S_2}
    \,\cup\, \underbrace{\left\{(a,t): \Delta(c,a)>\tau,\, t> T_a\right\}}_{\mathcal S_3}.
\end{align*}
The three sets are pairwise disjoint. We bound their regret contributions separately.

For every $(a,t)\in\mathcal{S}_1$, $\Delta(c,a)\le\tau$. Thus
\begin{equation}
    \sum_{(a,t)\in\mathcal{S}_1} \pi_t(c,a)\Delta(c,a) \le T\tau = C^2 A\eta T S_T(c).
    \label{eq:S1-regret}
\end{equation}
For every $(a,t)\in\mathcal{S}_2$ with $t \ge 2$, $\Delta(c,a) \le \frac{4\Lambda}{\eta(t-1)}$. Therefore
\begin{align}
    \sum_{(a,t)\in\mathcal{S}_2} \pi_t(c,a)\Delta(c,a)
    &\le 1 + \sum_{t=2}^T \frac{4\Lambda}{\eta(t-1)}
    \le 1+\frac{4\Lambda H_T}{\eta}. \label{eq:S2-regret}
\end{align}
For every $(a,t)\in\mathcal{S}_3$, $S_{t-1}(c) \le S_T(c) < \frac{\Delta(c,a)}{C^2 A\eta}$, and therefore, by \cref{lem:activation}, $\pi_t(c,a)\le\frac{1}{AT}$. Hence
\begin{equation}
    \sum_{(a,t)\in\mathcal{S}_3} \pi_t(c,a)\Delta(c,a)\le 1.\label{eq:S3-regret}
\end{equation}

Combining \eqref{eq:S1-regret}--\eqref{eq:S3-regret} and using $\Lambda \ge H_T \ge 1$, $\eta \le 1$, and $C^2\ge6$ yields
\begin{align}
    R_T(c)
    &\le C^2 A\eta T S_T(c) +2 +\frac{4\Lambda H_T}{\eta} \le C^2\left\{A\eta T S_T(c) +\frac{\Lambda^2}{\eta}\right\}.\label{eq:pathwise-regret}
\end{align}

We now apply \cref{ass:oracle} with failure probability
$\delta/2$. Fix any realization of the interaction on which
this oracle event holds. Holding the fitted predictors and
policies fixed, we average over a fresh context $c\sim\D$.

By the definition of $s_t(c)$ and the oracle guarantee,
\begin{equation*}
 \E_{c\sim\D}[s_t(c)]
 =\frac1t\sum_{r=1}^t
   \|\widehat f_t-f^\star\|_{\pi_r}^2
 \le\frac{\Err_{\rm orc}(T,\delta/2)}t.
\end{equation*}
Summing over $t=1,\ldots,T$ gives
\begin{equation*}
 \E_{c\sim\D}[S_T(c)]
 \le\Err_{\rm orc}(T,\delta/2)H_T.
\end{equation*}

Integrating \eqref{eq:pathwise-regret} over $c\sim\D$
gives $\sum_{t=1}^T\mathcal R(\pi_t)$ on the left-hand side.
Consequently,
\begin{equation*}
 \sum_{t=1}^T\mathcal R(\pi_t)
 \le C^2\left\{
   \frac{\Lambda^2}{\eta}
   +A\eta T\Err_{\rm orc}(T,\delta/2)H_T
 \right\}.
\end{equation*}
This inequality holds with probability at least $1-\delta/2$.

We next transfer this bound to $\Reg$, which evaluates the
policies at the observed contexts.
By \cref{cor:mean-regret}, with probability at least
$1-\delta/2$,
\begin{equation*}
 \Reg
 \le\sum_{t=1}^T\mathcal R(\pi_t)
   +\sqrt{\frac T2\log\frac2\delta}.
\end{equation*}
A union bound therefore gives, with probability at least
$1-\delta$,
\begin{equation*}
 \Reg
 \le C^2\left\{
   \frac{\Lambda^2}{\eta}
   +A\eta T\Err_{\rm orc}(T,\delta/2)H_T
 \right\}
 +\sqrt{\frac T2\log\frac2\delta}.
\end{equation*}

Substituting
$\eta=\sqrt{\Lambda/
[AT(1+\Err_{\rm orc}(T,\delta/2))]}$
and using $H_T\le\Lambda$ yields
\begin{equation}
 \begin{aligned}
 \Reg
 &\le
 2C^2\sqrt{AT(1+\Err_{\rm orc}(T,\delta/2))}
       \,(1+\log(AT))^{3/2}\\
 &\qquad+\sqrt{\frac T2\log\frac2\delta}.
 \end{aligned}
 \label{eq:tuned-bound}
\end{equation}
This proves the general-oracle bound with probability at least
$1-\delta$.

For a finite realizable class $\F$, the ERM guarantee in
\cref{app:regression} and $\delta\le1/2$ give
$1+\Err_{\rm orc}(T,\delta/2)
=O(\log(|\F|/\delta))$.
Substituting into \eqref{eq:tuned-bound}, including the
concentration term, yields
$\Reg=\widetilde O(\sqrt{AT\log(|\F|/\delta)})$.
\end{proof}

\subsection{Proof of Lemma~\ref{lem:activation}}

For every action $a \in \mathcal A$, define
\begin{equation}
 \begin{aligned}
 \kappa_{t,a}(c)
 &=\frac{1}{\bar\pi_t(c,a)}
   +\frac{1}{\bar\pi_t(c,a^\star(c))},
 \qquad
 B_{t,a}(c)=\sum_{r=1}^t\kappa_{r,a}(c).
 \end{aligned}
 \label{eq:inverse-average-probabilities}
\end{equation}
The term $\kappa_{t,a}(c)$ becomes large when the average policy
assigns little probability to one of the actions: $a$ or $a^{\star}(c)$. Prediction errors on
that action then receive little weight in $s_t(c)$, so small average
squared error alone gives weak control over the estimated loss gap.
The sum $B_{t,a}(c)$ accumulates these factors across updates.
The following supporting lemmas are used to prove Lemma~\ref{lem:activation}.

\begin{lemma}[Bounding the accumulated prediction error]
\label{lem:accumulated-prediction-error}
For every $a \in \A$ and $1\le r\le t\le T$,
\begin{equation*}
 D_{r,a}(c)^2\le S_r(c)B_{r,a}(c)\le S_t(c)B_{r,a}(c).
\end{equation*}
\end{lemma}

\begin{lemma}[Historical coverage under small prediction error]
\label{lem:bootstrap}
Fix a context $c \in \mathcal C$, a suboptimal action $a \in \mathcal A$ with $\Delta(c,a)>0$, and a round $t\in\{1,\ldots,T\}$.
If $S_t(c)<\frac{\Delta(c,a)}{C^2A\eta}$, then for $1\le r\le t$,
\begin{equation}
    B_{r,a}(c) \le CA\bigl(r+\eta\Delta(c,a)r^2\bigr). \label{eq:B-bootstrap}
\end{equation}
\end{lemma}

\begin{proof}[Proof of \cref{lem:activation}]
Suppose, for contradiction, that $\pi_t(c,a)\ge\frac{1}{AT}$ and $S_{t-1}(c)<\frac{\Delta(c,a)}{C^2A\eta}$.

We show that the bound on $\pi_t(c,a)$ forces the accumulated prediction error to be large and negative.

Since $W_{t,a^\star(c)}(c)=1$, \eqref{eq:weight-policy} gives
$\pi_t(c,a)\le W_{t,a}(c)$. Hence
\begin{equation*}
 \exp\{-\eta((t-1)\Delta(c,a)+D_{t-1,a}(c))\}
 =W_{t,a}(c)\ge \pi_t(c,a) \ge \frac1{AT}.
\end{equation*}
Taking logarithms and applying $\eta(t-1)\Delta(c,a)\ge4\Lambda$
and $\log(AT)\le\Lambda$ yields
\begin{equation}
 D_{t-1,a}(c)
 \le\frac{\log(AT)}{\eta}-(t-1)\Delta(c,a)
 \le-\frac34(t-1)\Delta(c,a).
 \label{eq:activation-drift}
\end{equation}

We now show that the bound on $S_{t-1}(c)$ makes such a large prediction error impossible.

By \cref{lem:bootstrap}, applied to action $a$ and round $t - 1 \ge \frac{4\Lambda}{\eta \Delta(c,a)} \ge \frac{4}{\eta \Delta(c,a)}$,
\begin{align*}
 B_{t-1,a}(c)
 \le CA\bigl((t-1)+\eta\Delta(c,a)(t-1)^2\bigr) \le 2CA\eta\Delta(c,a)(t-1)^2.
\end{align*}
Therefore, by \cref{lem:accumulated-prediction-error} and the assumed bound on $S_{t-1}(c)$,
\begin{equation*}
 |D_{t-1,a}(c)|
 \le\sqrt{S_{t-1}(c)B_{t-1,a}(c)}
 <\sqrt{\frac2C}\,(t-1)\Delta(c,a)
 \le\frac12(t-1)\Delta(c,a),
\end{equation*}
where the last inequality uses $C\ge8$.
This contradicts \eqref{eq:activation-drift} and proves the claim.
\end{proof}

\subsection{Proof of \cref{lem:accumulated-prediction-error}}

\begin{proof}[Proof of \cref{lem:accumulated-prediction-error}]
For a fixed $t$ and $a \ne a^{\star}(c)$, using the AM-GM inequality, we can show
\begin{align*}
    \bigl(e_{t,a}(c)-e_{t,a^\star(c)}(c)\bigr)^2
    &\le \left( \bar\pi_t(c,a)e_{t,a}(c)^2 +\bar\pi_t(c,a^\star(c))e_{t,a^\star(c)}(c)^2 \right) \left(\frac1{\bar\pi_t(c,a)} +\frac1{\bar\pi_t(c,a^\star(c))}\right)\\
    &\le s_t(c)\kappa_{t,a}(c).
\end{align*}
For $a = a^{\star}(c)$, it vacuously holds that $\bigl(e_{t,a}(c)-e_{t,a^\star(c)}(c)\bigr)^2 = 0 \le s_t(c)\kappa_{t,a}(c)$.

Applying Cauchy--Schwarz across rounds gives
\begin{align*}
 D_{t,a}(c)^2
 &=\left(
   \sum_{r=1}^t
   \frac{e_{r,a}(c)-e_{r,a^\star(c)}(c)}{\sqrt{\kappa_{r,a}(c)}}
   \sqrt{\kappa_{r,a}(c)}
 \right)^2\\
 &\le
 \left(\sum_{r=1}^t
       \frac{(e_{r,a}(c)-e_{r,a^\star(c)}(c))^2}{\kappa_{r,a}(c)}\right)
 \left(\sum_{r=1}^t\kappa_{r,a}(c)\right)
 \le S_t(c)B_{t,a}(c).
\end{align*}
Applying this bound at an earlier round $r$ and using $S_r(c)\le S_t(c)$ for $r\le t$ proves the claim.
\end{proof}

\subsection{Proof of \cref{lem:bootstrap}}

\begin{proof}[Proof of \cref{lem:bootstrap}]
With $c$ and $a$ fixed throughout the proof, for every action $b \in \mathcal A$, write
\begin{equation*}
\Gamma_b=\max\{\Delta(c,a),\Delta(c,b)\}.
\end{equation*}
To control the normalizing denominator, we prove by induction that, for every $1\le r\le t$ and $b \in \mathcal A$,
\begin{equation}
B_{r,b}(c)\le CA\bigl(r+\eta\Gamma_br^2\bigr).
\label{eq:B-bootstrap-all}
\end{equation}
Alongside this bound, we establish the early-round action probability bounds
\begin{align}
r-1\le\frac1{\eta\Gamma_b}
\quad\Longrightarrow\quad
\pi_r(c,b)\ge\frac{e^{-2}}A.\label{eq:init-exp-bound}
\end{align}

For the base case $r=1$, the policy is uniform and $B_{1,b}(c)=2A$, so both \eqref{eq:B-bootstrap-all} and \eqref{eq:init-exp-bound} hold.

The proof of the induction step consists of two parts. First, we show that the bounds on $B_{r-1,b}(c)$ control the accumulated prediction errors and therefore the action probabilities $\pi_r(c,b)$. Second, summing the early-round action probabilities allows us to bound $N_{r,b}(c)$ and consequently $B_{r,b}(c)$.

\paragraph{Part 1: Lower-bound $\pi_{r}(c,b)$.}
Fix $2 \le r \le t+1$ and suppose that \eqref{eq:B-bootstrap-all} holds at round $r-1$ for every $b\in \mathcal A$. To show \eqref{eq:init-exp-bound} at round $r$, suppose that $r-1\le1/(\eta\Delta(c,a))$.

Let $x_b=\eta\Gamma_b(r-1) \ge 0$. By \cref{lem:accumulated-prediction-error}, the induction hypothesis, $\Gamma_b \ge \Delta(c,a)$, and $C \ge 8$,
\begin{equation}
\eta^2D_{r-1,b}(c)^2
\le \eta^2S_t(c)B_{r-1,b}(c)
\le \frac{\Delta(c,a)\,x_b(1+x_b)}{C\Gamma_b} \le \frac{\max\{1, x_b^2\}}{4}.
\label{eq:D-bootstrap}
\end{equation}
As $\Gamma_b = \max\{\Delta(c, b), \Delta(c,a)\}$ and $\eta \Delta(c,a) (r-1) \le 1$, it holds that
\begin{align}
    \eta \Delta(c,b) (r-1) \le x_b \qquad \text{and} \qquad  \eta \Delta(c,b) (r-1) \ge x_b\, \mathbb{I}\{x_b > 1\}.\label{eq:both-bounds-combined}
\end{align}
Combining \eqref{eq:D-bootstrap} and \eqref{eq:both-bounds-combined} with the identity $W_{r,b}(c)=\exp\{-\eta((r-1)\Delta(c,b)+D_{r-1,b}(c))\}$,
\begin{align*}
    \exp\left\{-x_b - \tfrac{1}{2}\max\{1, x_b\}\right\} \le W_{r,b}(c) \le \exp\left\{-x_b\,\mathbb{I}\{x_b > 1\} + \tfrac{1}{2}\max\{1, x_b\}\right\}.
\end{align*}

In particular, when $x_b \le 1$, $W_{r,b}(c) \in [e^{-3/2},e^{1/2}]$, while when $x_b > 1$, $W_{r,b}(c) \in (0,e^{-1/2})$.

The normalizing denominator for $\pi_r(c, .)$, therefore, satisfies
\begin{equation*}
\sum_{k\in\A}W_{r,k}(c)\le Ae^{1/2}.
\end{equation*}
For any action $b$ satisfying $r-1\le1/(\eta\Gamma_b)$, $x_b \le 1$, and so
\begin{equation*}
\pi_r(c,b)
\ge\frac{e^{-3/2}}{Ae^{1/2}}
=\frac{e^{-2}}A.
\end{equation*}
This establishes \eqref{eq:init-exp-bound} for probabilities $\pi_{r}(c,b)$ at round $r$ using \eqref{eq:B-bootstrap-all} for $B_{r-1,b}(c)$ at round $r-1$.

\paragraph{Part 2: Lower-bound $N_{r,b}(c)$ and upper-bound $B_{r,b}(c)$.}
Now let $2\le r\le t$. By the induction hypothesis and Part 1, the probability bounds in \eqref{eq:init-exp-bound} hold on their respective intervals for all policies $\pi_1,\ldots,\pi_r$.

Consider arbitrary $b \in \mathcal A$ and $1\le u\le r$. It holds that
\begin{align*}
    N_{u,b}(c) &= \sum_{s=1}^u \pi_s(c,b) \ge \frac{e^{-2}}{A} \min\left\{u, \frac{1}{\eta \Gamma_b}\right\}.
\end{align*}
By \eqref{eq:inverse-average-probabilities}, as $\Gamma_{a^{\star}(c)} \le \Gamma_b$, it follows that
\begin{equation*}
\kappa_{u,b}(c)
=\frac{u}{N_{u,b}(c)}
+\frac{u}{N_{u,a^\star(c)}(c)}
\le 2e^2A\max\{1,\eta\Gamma_bu\}.
\end{equation*}
Summing over $u$ yields
\begin{align*}
B_{r,b}(c) \le 2e^2A\sum_{u=1}^r\max\{1,\eta\Gamma_bu\}
\le 2e^2A\bigl(r+\eta\Gamma_br^2\bigr)\le CA\bigl(r+\eta\Gamma_br^2\bigr).
\end{align*}
This proves \eqref{eq:B-bootstrap-all} at round $r$ and completes the induction.

Since $\Gamma_a=\Delta(c,a)$, taking $b=a$ gives \eqref{eq:B-bootstrap} for all $1 \le r \le t$, completing the proof.
\end{proof}

\section{Complete proof for batched vanilla PO}
\label{app:batched}

We prove \cref{thm:batched-vanilla} for the algorithm in
\cref{alg:batched-vanilla-po}, under the stochastic model of
\cref{sec:setup}. Throughout this section, $\pi_m$ denotes the policy
used throughout epoch $m$, and the predictors take values in $[0,1]$.
Fix $\delta\in(0,1/2]$ and use the known bound
$\Err_{\rm orc}^{\rm batch}(T,\delta)\ge1$ from the main body.
We reserve failure probability $\delta/2$ for regression accuracy
and $\delta/2$ for concentration over the observed contexts.
Accordingly, assume that, with probability at least $1-\delta/2$,
\begin{equation}
 \|\widetilde f_m-f^\star\|_{\pi_{m-1}}^2
 \le\frac{\Err_{\rm orc}^{\rm batch}(T,\delta)}{B_{m-1}}
 \quad\text{simultaneously for }m=2,\ldots,M.
 \label{eq:batch-epoch-accuracy}
\end{equation}
We use the following explicit learning rates, which realize the
order stated in the theorem:
\begin{equation}
 \eta_m=\min\left\{\frac1A,
                  \frac1{\sqrt{6AB_m\Err_{\rm orc}^{\rm batch}(T,\delta)}}\right\},
 \qquad m=2,\ldots,M.
 \label{eq:vanilla-batch-rate}
\end{equation}

We first prove a deterministic regret bound for every realized
predictor sequence satisfying \eqref{eq:batch-epoch-accuracy}.
Averaging exponential updates gives both a lower bound on action
probabilities and a bound on predicted regret. These probability
bounds let us transfer prediction accuracy from the preceding
epoch's policy to other policies. An induction then compares
predicted and true regret, after which we sum their population
regrets over epochs. The concentration result in
\cref{app:concentration} then converts this bound to $\Reg$.
Finally, \cref{app:batch-oracle-verification} shows how an i.i.d.
regression guarantee supplies the theorem's accuracy condition
and proves its finite-class ERM conclusion.

\subsection{Policy representation and the epoch schedule}
\label{app:algorithms}

Recall that $B_m=2^{m-1}$ and $M=\lceil\log_2(T+1)\rceil$.
Since $\sum_{j<m}B_j=B_m-1$, epoch $m$ starts at round $B_m$ and
contains
\begin{equation*}
 n_m=\min\{B_m,T-B_m+1\}
\end{equation*}
rounds. Every started epoch has $B_m\le T$. For $m<M$, the batch
$\mathcal S_m$ has exactly $B_m$ observations; only the final epoch
may be shortened. Thus every batch used to fit a predictor is complete,
and the algorithm makes exactly $M-1$ oracle calls. When $T=1$, it
uses the uniform policy once and makes no calls.

Datasets are ordered lists, so repeated observations occupy distinct
rows. The predictor $\widetilde f_m$ is fitted only to
$\mathcal S_{m-1}$, after which that batch is discarded. No fit uses
$\mathcal S_M$.

\begin{proposition}[Batched vanilla PO as an average of policy updates]
\label{prop:mixture-sampling}
Fix an epoch $m\ge2$ and a fitted predictor $\widetilde f_m$.
For every context $c$, define $\pi_{m,1}(c,a)=1/A$ and, for
$s=1,\ldots,B_m$,
\begin{equation*}
 \pi_{m,s+1}(c,a)
 =\frac{\pi_{m,s}(c,a)\exp\{-\eta_m\widetilde f_m(c,a)\}}
 {\sum_{b\in\A}\pi_{m,s}(c,b)\exp\{-\eta_m\widetilde f_m(c,b)\}}.
\end{equation*}
Then the policy in \eqref{eq:vanilla-batch-policy} satisfies
\begin{equation}
 \pi_m(c,a)=\frac1{B_m}\sum_{s=1}^{B_m}\pi_{m,s}(c,a).
 \label{eq:vanilla-average}
\end{equation}
The sampler in \eqref{eq:vanilla-batch-sampler} draws an action from
this distribution exactly.
\end{proposition}

\begin{proof}
Unrolling the fixed-predictor updates gives
\begin{equation*}
 \pi_{m,s}(c,a)
 =\frac{\exp\{-\eta_m(s-1)\widetilde f_m(c,a)\}}
 {\sum_{b\in\A}\exp\{-\eta_m(s-1)\widetilde f_m(c,b)\}}.
\end{equation*}
This is the summand indexed by $k=s-1$ in
\eqref{eq:vanilla-batch-policy}. Conditional on $K_t=s-1$, the
sampler in \eqref{eq:vanilla-batch-sampler} has this same action
distribution. Averaging over the uniform choice of $K_t$ proves
both claims.
\end{proof}

The sequence $\pi_{m,s}$ is an auxiliary sequence used to analyze the
policy. The learner stores $(\widetilde f_m,\eta_m,B_m)$ and, for
each observed context, evaluates the $A$ predicted losses and samples
one softmax distribution. This requires no iteration over the context
space and no explicit evaluation of all $B_m$ mixture components.
In a shortened final epoch, the policy and sampler still use the
nominal value $B_M$.

\subsection{Action probabilities and predicted regret}
\label{app:kernels}

For any predictor $\widetilde f_m$, define
\begin{equation}
 d_m(c,a)
 =\widetilde f_m(c,a)-\min_{b\in\A}\widetilde f_m(c,b)
 \label{eq:batch-predicted-gap}
\end{equation}
and the corresponding predicted regret
\begin{equation*}
 \widehat{\mathcal R}_m(\pi)
 =\E_{c\sim\D}\sum_{a\in\A}\pi(c,a)d_m(c,a).
\end{equation*}
Thus $d_m(c,a)$ is the predicted loss gap, and
$\widehat{\mathcal R}_m(\pi)$ is the regret of $\pi$ if the predicted
losses were correct. These definitions also apply to the bonus
variant introduced in \cref{app:bonus-po}.

\begin{lemma}[Action probabilities and predicted regret of batched vanilla PO]
\label{lem:virtual-properties}
For every epoch $m\ge2$, context $c$, and action $a$,
\begin{equation}
 \frac1{\pi_m(c,a)}
 \le eA\bigl(1+B_m\eta_m d_m(c,a)\bigr).
 \label{eq:vanilla-batch-coverage}
\end{equation}
Consequently, for every policy $\pi$,
\begin{equation}
 \E_{c\sim\D}\sum_{a\in\A}\frac{\pi(c,a)}{\pi_m(c,a)}
 \le eA\bigl(1+B_m\eta_m\widehat{\mathcal R}_m(\pi)\bigr).
 \label{eq:vanilla-policy-coverage}
\end{equation}
The policy deployed in epoch $m$ also satisfies
\begin{equation}
 \widehat{\mathcal R}_m(\pi_m)
 \le\frac{2\log A}{B_m\eta_m}.
 \label{eq:vanilla-predicted-regret}
\end{equation}
\end{lemma}

\begin{proof}
Fix $m$ and $c$. Subtracting the minimum predicted loss from every
action leaves each exponential update unchanged. Hence
\begin{equation*}
 \pi_{m,s}(c,a)
 =\frac{\exp\{-\eta_m(s-1)d_m(c,a)\}}
 {\sum_{b\in\A}\exp\{-\eta_m(s-1)d_m(c,b)\}}
 \ge\frac1A\exp\{-\eta_m(s-1)d_m(c,a)\},
\end{equation*}
because all predicted gaps are nonnegative.

For every integer $B\ge1$ and $x\ge0$,
\begin{equation}
 \frac1B\sum_{k=0}^{B-1}e^{-kx}
 \ge e^{-1}\min\left\{1,\frac1{Bx}\right\},
 \label{eq:geometric-average}
\end{equation}
where $1/(Bx)=+\infty$ when $x=0$. If $Bx\le1$, every term is at
least $e^{-1}$. If $Bx>1$ and $x\ge1$, the term $k=0$ suffices.
Otherwise, the terms $k=0,\ldots,\lfloor1/x\rfloor$ are all present
and are each at least $e^{-1}$; their number is at least $1/x$.
Applying \eqref{eq:geometric-average} to \eqref{eq:vanilla-average}
gives
\begin{equation*}
 \pi_m(c,a)\ge\frac{e^{-1}}A
 \min\left\{1,\frac1{B_m\eta_m d_m(c,a)}\right\}.
\end{equation*}
Taking reciprocals and using $\max\{1,x\}\le1+x$ proves
\eqref{eq:vanilla-batch-coverage}. Multiplying by $\pi(c,a)$, summing over
actions, and averaging over contexts proves
\eqref{eq:vanilla-policy-coverage}. This implication is deterministic
and holds for every policy simultaneously.

To bound predicted regret, let
\begin{equation*}
 W_{m,s}(c)=\sum_{a\in\A}
 \exp\{-\eta_m(s-1)d_m(c,a)\}.
\end{equation*}
For $x\ge0$, $e^{-x}\le1-x+x^2/2$. Therefore
\begin{align*}
 \log\frac{W_{m,s+1}(c)}{W_{m,s}(c)}
 &\le-\eta_m\sum_{a\in\A}\pi_{m,s}(c,a)d_m(c,a)
   +\frac{\eta_m^2}{2}
     \sum_{a\in\A}\pi_{m,s}(c,a)d_m(c,a)^2.
\end{align*}
Initially $W_{m,1}(c)=A$, and $W_{m,B_m+1}(c)\ge1$ because an
action with zero predicted gap retains weight one. Summing the last
inequality gives
\begin{align*}
 \sum_{s=1}^{B_m}\sum_{a\in\A}\pi_{m,s}(c,a)d_m(c,a)
 &\le\frac{\log A}{\eta_m}
 +\frac{\eta_m}{2}
   \sum_{s=1}^{B_m}\sum_{a\in\A}\pi_{m,s}(c,a)d_m(c,a)^2.
\end{align*}
Since $0\le d_m(c,a)\le1$ and $\eta_m\le1$, the last sum can be
absorbed into the left-hand side, leaving an upper bound
$2\log A/\eta_m$. Divide by $B_m$ and average over $c$ to obtain
\eqref{eq:vanilla-predicted-regret}.
\end{proof}

The first bound quantifies the exploration retained by the average:
inverse action probabilities grow at most linearly with the predicted
gap. The second shows that this exploration does not prevent the
policy from achieving small predicted regret.

\subsection{Transferring regression accuracy between policies}

The predictor for epoch $m$ is trained under $\pi_{m-1}$.
To analyze the new policy $\pi_m$, we must control its prediction
error under a different action distribution. Define the error in
the predicted population loss of a policy by
\begin{equation*}
 E_m(\pi)=\left|\E_{c\sim\D}\sum_{a\in\A}\pi(c,a)
       \bigl(\widetilde f_m(c,a)-f^\star(c,a)\bigr)\right|.
\end{equation*}
All statements below are deterministic once the predictors and
policies have been fixed. In particular, they apply to policies chosen
using those predictors.

\begin{lemma}[Prediction error under another policy]
\label{lem:batch-error-transfer}
Fix $m\ge2$. Suppose $\pi_{m-1}(c,a)>0$ for every context and
action, and the accuracy bound \eqref{eq:batch-epoch-accuracy}
holds at epoch $m$.
Then, for every policy $\pi$,
\begin{equation}
 E_m(\pi)
 \le\sqrt{\frac{\Err_{\rm orc}^{\rm batch}(T,\delta)}{B_{m-1}}
       \E_{c\sim\D}\sum_{a\in\A}
           \frac{\pi(c,a)}{\pi_{m-1}(c,a)}}.
 \label{eq:batch-error-transfer-general}
\end{equation}
For batched vanilla PO and $m\ge3$, this gives
\begin{equation*}
 E_m(\pi)
 \le\sqrt{\frac{eA\Err_{\rm orc}^{\rm batch}(T,\delta)}{B_{m-1}}
       \bigl(1+B_{m-1}\eta_{m-1}
                    \widehat{\mathcal R}_{m-1}(\pi)\bigr)}.
\end{equation*}
For $m=2$, uniform sampling gives $E_2(\pi)\le\sqrt{A\Err_{\rm orc}^{\rm batch}(T,\delta)/B_1}$.
\end{lemma}

\begin{proof}
Write $e_m(c,a)=\widetilde f_m(c,a)-f^\star(c,a)$.
Cauchy--Schwarz, with weights $\pi_{m-1}(c,a)>0$, gives
\begin{align*}
 E_m(\pi)
 &\le\|e_m\|_{\pi_{m-1}}
   \left(\E_{c\sim\D}\sum_{a\in\A}
        \frac{\pi(c,a)^2}{\pi_{m-1}(c,a)}\right)^{1/2}\\
 &\le\|e_m\|_{\pi_{m-1}}
   \left(\E_{c\sim\D}\sum_{a\in\A}
        \frac{\pi(c,a)}{\pi_{m-1}(c,a)}\right)^{1/2},
\end{align*}
where the second inequality uses $\pi(c,a)^2\le\pi(c,a)$.
Apply \eqref{eq:batch-epoch-accuracy} to obtain
\eqref{eq:batch-error-transfer-general}. For $m\ge3$, use
\eqref{eq:vanilla-policy-coverage} for the preceding epoch. For
$m=2$, the expectation in \eqref{eq:batch-error-transfer-general}
equals $A$, since $\pi_1(c,a)=1/A$.
\end{proof}

\subsection{Comparing predicted regret with true regret}

We now show that policies with small predicted regret also have small
true regret. Following the regression-oracle analysis of
\citet{falcon}, the induction uses the preceding epoch's comparison
to control the next predictor's error. The argument depends only on
action probabilities and regression accuracy, so we state it in a
form that also applies to the bonus variant.

\begin{lemma}[Comparison of true and predicted regret]
\label{lem:batch-regret-comparison}
Suppose \eqref{eq:batch-epoch-accuracy} holds for every
$m=2,\ldots,M$, and $\pi_1$ is uniform. Let $C_0\ge1$ be a
constant and let $\gamma_2,\ldots,\gamma_M$ be positive and
nondecreasing. Assume
\begin{equation}
 \pi_m(c,a)^{-1}
 \le C_0\bigl(A+\gamma_m d_m(c,a)\bigr),
 \qquad
 \gamma_m^2\le\frac{AB_{m-1}}{C_0 \Err_{\rm orc}^{\rm batch}(T,\delta)}
 \quad (m\ge2).
 \label{eq:batch-comparison-conditions}
\end{equation}
Then, simultaneously for every epoch $m\ge2$ and every policy $\pi$,
\begin{align}
 \mathcal R(\pi)
 &\le2\widehat{\mathcal R}_m(\pi)+\frac{30A}{\gamma_m},
 \label{eq:batch-true-to-predicted}\\
 \widehat{\mathcal R}_m(\pi)
 &\le2\mathcal R(\pi)+\frac{30A}{\gamma_m}.
 \label{eq:batch-predicted-to-true}
\end{align}
\end{lemma}

\begin{proof}
Let $\pi^\star$ choose a true optimal action at each context, and let
$\widehat\pi_m$ choose an action minimizing $\widetilde f_m(c,\cdot)$,
with ties broken by a fixed ordering. Define
\begin{equation*}
 L(\pi)=\E_{c\sim\D}\sum_{a\in\A}\pi(c,a)f^\star(c,a),
 \qquad
 \widetilde L_m(\pi)
 =\E_{c\sim\D}\sum_{a\in\A}\pi(c,a)\widetilde f_m(c,a).
\end{equation*}
Then $\mathcal R(\pi)=L(\pi)-L(\pi^\star)$ and
$\widehat{\mathcal R}_m(\pi)
=\widetilde L_m(\pi)-\widetilde L_m(\widehat\pi_m)$.
The optimality of these two comparator policies implies
\begin{align}
 \mathcal R(\pi)
 &\le\widehat{\mathcal R}_m(\pi)+E_m(\pi)+E_m(\pi^\star),
 \label{eq:batch-true-comparison}\\
 \widehat{\mathcal R}_m(\pi)
 &\le\mathcal R(\pi)+E_m(\pi)+E_m(\widehat\pi_m).
 \label{eq:batch-predicted-comparison}
\end{align}
For example, the first inequality follows by inserting
$\widetilde L_m(\pi)$ and $\widetilde L_m(\pi^\star)$ into
$L(\pi)-L(\pi^\star)$ and using
$\widetilde L_m(\widehat\pi_m)\le\widetilde L_m(\pi^\star)$.
The second follows by the analogous decomposition and
$L(\pi^\star)\le L(\widehat\pi_m)$.

For $m=2$, \cref{lem:batch-error-transfer} and
\eqref{eq:batch-comparison-conditions} give
\begin{equation*}
 E_2(\pi)\le\sqrt{\frac{A\Err_{\rm orc}^{\rm batch}(T,\delta)}{B_1}}\le\frac A{\gamma_2}
\end{equation*}
for every policy. Substitution into
\eqref{eq:batch-true-comparison} and \eqref{eq:batch-predicted-comparison}
proves both claimed comparisons at epoch two, even with coefficient
one on the regret and additive term $2A/\gamma_2$.

Suppose the comparisons hold at epoch $m-1$, where $m\ge3$.
Integrating the coverage inequality in
\eqref{eq:batch-comparison-conditions} and applying
\eqref{eq:batch-error-transfer-general} gives
\begin{align*}
 E_m(\pi)
 &\le\sqrt{\frac{C_0 \Err_{\rm orc}^{\rm batch}(T,\delta)}{B_{m-1}}
            \bigl(A+\gamma_{m-1}
                    \widehat{\mathcal R}_{m-1}(\pi)\bigr)}\\
 &\le\sqrt{\frac{C_0 \Err_{\rm orc}^{\rm batch}(T,\delta)}{B_{m-1}}
            \bigl(31A+2\gamma_{m-1}\mathcal R(\pi)\bigr)}.
\end{align*}
The parameter condition gives
$C_0 \Err_{\rm orc}^{\rm batch}(T,\delta)/B_{m-1}\le A/\gamma_m^2$, and monotonicity gives
$\gamma_{m-1}\le\gamma_m$. Consequently,
\begin{align}
 E_m(\pi)
 &\le\sqrt{31}\frac A{\gamma_m}
       +\sqrt{\frac{2A\mathcal R(\pi)}{\gamma_m}}\notag\\
 &\le\frac18\mathcal R(\pi)+10\frac A{\gamma_m}.
 \label{eq:batch-prediction-error}
\end{align}
Here we used $\sqrt{2xy}\le x/8+4y$ and $\sqrt{31}+4<10$.

Applying \eqref{eq:batch-prediction-error} in
\eqref{eq:batch-true-comparison}, with $\mathcal R(\pi^\star)=0$,
yields
\begin{equation*}
 \frac78\mathcal R(\pi)
 \le\widehat{\mathcal R}_m(\pi)+20\frac A{\gamma_m}.
\end{equation*}
Since $8/7\le2$ and $160/7\le30$, this proves
\eqref{eq:batch-true-to-predicted}. In particular,
$\mathcal R(\widehat\pi_m)\le30A/\gamma_m$.
Using this bound in \eqref{eq:batch-predicted-comparison} gives
\begin{align*}
 \widehat{\mathcal R}_m(\pi)
 &\le\frac98\mathcal R(\pi)+20\frac A{\gamma_m}
       +\frac18\mathcal R(\widehat\pi_m)\\
 &\le2\mathcal R(\pi)+30\frac A{\gamma_m}.
\end{align*}
This proves \eqref{eq:batch-predicted-to-true} and closes the induction.
\end{proof}

For vanilla PO, the lemma applies with $C_0=e$ and the analysis
quantity
\begin{equation*}
 \gamma_m=AB_m\eta_m
 =\min\left\{B_m,\sqrt{\frac{AB_{m-1}}{3\Err_{\rm orc}^{\rm batch}(T,\delta)}}\right\}.
\end{equation*}
This introduces no additional algorithm parameter: it is determined
by the learning rate in \eqref{eq:vanilla-batch-rate} and
$B_m=2B_{m-1}$. Both entries in the minimum increase with $m$,
so $\gamma_m$ is nondecreasing. Since $3>e$, its square is at most
$AB_{m-1}/(e\Err_{\rm orc}^{\rm batch}(T,\delta))$.
Together with \eqref{eq:vanilla-batch-coverage}, this verifies the other
conditions of the comparison lemma. The constant $C_0$ is used
only in the analysis: here $C_0=e$, while the bonus variant
satisfies the same conditions with $C_0=1$.
Combining \eqref{eq:batch-true-to-predicted} with
\eqref{eq:vanilla-predicted-regret} therefore gives
\begin{equation}
 \mathcal R(\pi_m)
 \le\frac{30+4\log A}{B_m\eta_m}.
 \label{eq:vanilla-epoch-regret}
\end{equation}

\subsection{The cumulative regret bound}

\begin{proof}[Regret bound in \cref{thm:batched-vanilla}]
If $T=1$, the regret is at most one and there are no oracle calls,
so both conclusions hold. Assume $T\ge2$ and fix a realization
satisfying \eqref{eq:batch-epoch-accuracy}.

We first bound the regret averaged over the context distribution.
The first epoch contributes at most one, and
\eqref{eq:vanilla-epoch-regret} bounds every later epoch. Since
epoch $m$ uses the same policy for its $n_m$ rounds,
\begin{align*}
 \sum_{m=1}^M n_m\mathcal R(\pi_m)
 &\le1+(30+4\log A)\sum_{m=2}^M\frac1{\eta_m}\\
 &\le1+(30+4\log A)
 \left(A(M-1)+\sqrt{6A\Err_{\rm orc}^{\rm batch}(T,\delta)}
                    \sum_{m=2}^M\sqrt{B_m}\right).
\end{align*}
The first inequality uses $n_m\le B_m$; the second uses
\begin{equation*}
 \frac1{\eta_m}
 =\max\{A,\sqrt{6AB_m\Err_{\rm orc}^{\rm batch}(T,\delta)}\}
 \le A+\sqrt{6AB_m\Err_{\rm orc}^{\rm batch}(T,\delta)}.
\end{equation*}
Since the epoch lengths double and $B_M\le T$,
\begin{equation}
 \sum_{m=1}^M\sqrt{B_m}
 \le\frac{\sqrt{B_M}}{1-2^{-1/2}}<3.5\sqrt T.
 \label{eq:batch-length-sum}
\end{equation}
Consequently,
\begin{equation*}
 \sum_{m=1}^M n_m\mathcal R(\pi_m)
 \le1+(30+4\log A)
       \bigl(AM+9\sqrt{AT\,\Err_{\rm orc}^{\rm batch}(T,\delta)}\bigr).
\end{equation*}
This bound holds on the simultaneous regression event, whose
probability is at least $1-\delta/2$.

The target regret $\Reg$ evaluates the policies at the contexts
actually observed. Applying \cref{cor:mean-regret} to the deployed
sequence of round policies, with failure probability $\delta/2$,
gives
\begin{equation*}
 \Reg\le\sum_{m=1}^M n_m\mathcal R(\pi_m)
             +\sqrt{\frac T2\log\frac2\delta}.
\end{equation*}
This concentration argument is applied unconditionally, rather
than after conditioning on the regression event. A union bound
therefore yields, with probability at least $1-\delta$,
\begin{equation}
 \begin{aligned}
 \Reg
 &\le1+(30+4\log A)
       \bigl(AM+9\sqrt{AT\,\Err_{\rm orc}^{\rm batch}(T,\delta)}\bigr)\\
 &\qquad+\sqrt{\frac T2\log\frac2\delta}.
 \end{aligned}
 \label{eq:batched-vanilla-bound}
\end{equation}

To simplify the bound, first suppose
$A\le\Err_{\rm orc}^{\rm batch}(T,\delta)T$.
Then $A\le\sqrt{AT\,\Err_{\rm orc}^{\rm batch}(T,\delta)}$,
so the $AM$ term is absorbed up to logarithmic factors.
If instead $A>\Err_{\rm orc}^{\rm batch}(T,\delta)T$, then
$\sqrt{AT\,\Err_{\rm orc}^{\rm batch}(T,\delta)}>T$,
and the deterministic bound $\Reg\le T$ suffices.
Because $\Err_{\rm orc}^{\rm batch}(T,\delta)\ge1$, the additional
concentration term is also absorbed by logarithmic factors in
$1/\delta$. Thus
$\Reg=\widetilde O(\sqrt{AT\,\Err_{\rm orc}^{\rm batch}(T,\delta)})$.
The schedule makes exactly $M-1=O(\log T)$ regression calls.
\end{proof}

\subsection{Verifying the prediction-accuracy condition}
\label{app:batch-oracle-verification}

The preceding argument uses only the simultaneous accuracy condition
\eqref{eq:batch-epoch-accuracy}. The following lemma explains how
ordinary regression guarantees for one fixed-policy batch imply that
condition. This is the offline-oracle setting used by FALCON+
\citep{falcon}.

\begin{lemma}[Simultaneous accuracy of the batch predictors]
\label{lem:batch-accuracy-event}
Fix $\delta\in(0,1/2]$ and a bound
$\Err_{\rm orc}^{\rm batch}(T,\delta)\ge1$ before the interaction.
Suppose that, for every fixed policy $\pi$ and every completed batch
size $n\in\{B_1,\ldots,B_{M-1}\}$, fitting the oracle to $n$ i.i.d.
observations under $\pi$ gives
$n\|\widetilde f-f^\star\|_\pi^2
\le\Err_{\rm orc}^{\rm batch}(T,\delta)$
with probability at least $1-\delta/(2M)$.
Then \eqref{eq:batch-epoch-accuracy} holds simultaneously for all
$m=2,\ldots,M$ with probability at least $1-\delta/2$.
This applies to any sequence of policies fixed before their epochs,
using fresh randomness for each action and oracle call.
\end{lemma}

\begin{proof}
The claim is vacuous when $T=1$. Otherwise, fix $m\ge2$ and
condition on the history before epoch $m-1$, including the randomness
used to construct $\pi_{m-1}$. This policy is now fixed. The new
context--loss pairs are i.i.d. and independent of that history, and
fresh action randomness makes the observed triples conditionally
i.i.d. under $\pi_{m-1}$. This batch contains exactly $B_{m-1}$
observations because $m-1<M$.

The assumed single-batch guarantee, including the oracle's fresh
randomness, bounds the conditional failure probability by
$\delta/(2M)$. Averaging over the history gives the same unconditional
bound. A union bound over the $M-1$ fitted predictors yields total
failure probability at most $(M-1)\delta/(2M)\le\delta/2$.
No independence between fits is required.
\end{proof}

\begin{proof}[Finite-class conclusion of \cref{thm:batched-vanilla}]
For a finite realizable class $\F$, ordinary least-squares ERM
satisfies \cref{lem:private-fast-rate} with $\alpha=0$.
Applying that result at failure probability $\delta/(2M)$ shows
that the choice
\begin{equation}
 \Err_{\rm orc}^{\rm batch}(T,\delta)
 =10\log\frac{4|\F|M}{\delta}
 \label{eq:batch-erm-error}
\end{equation}
satisfies \cref{lem:batch-accuracy-event}. This bound is at least one.
Thus all batch predictors meet \eqref{eq:batch-epoch-accuracy}
with probability at least $1-\delta/2$. Using this bound in
\eqref{eq:vanilla-batch-rate} and applying
\eqref{eq:batched-vanilla-bound}, including its concentration term, gives
$\Reg=\widetilde O(\sqrt{AT\log(|\F|/\delta)})$.
\end{proof}

\section{Complete proof for private batched vanilla PO}
\label{app:privacy}

The private algorithm uses the same schedule and policy construction
as \cref{alg:batched-vanilla-po}, with a private regression oracle.
Before epoch $m\ge2$, it releases $(\widetilde f_m,\pi_m)$.
The policy map is specified by the private predictor and the public
parameters; contexts, selected actions, observed losses, and evaluations
at protected contexts remain internal, as in \cref{def:dp}.

The regret proof above applies with the private oracle's error bound.
It remains to prove privacy of all releases together. Disjoint batches
alone do not justify treating future datasets as fixed: changing an
early private fit can change later policies and observed rows.
The proof instead shows that this later dependence is postprocessing
of the single private fit that uses the changed record.

\subsection{Privacy of the complete transcript}

\begin{lemma}[Privacy when each batch is used in one fit]
\label{lem:private-dp}
Consider the fixed epoch schedule of \cref{alg:batched-vanilla-po}.
Suppose the first policy is uniform and each later epoch policy is a
deterministic function of the preceding private predictor and public
parameters, remaining fixed while its batch is collected.
Each batch except the final one is used in one
$(\varepsilon_{\rm priv},\delta_{\rm priv})$-DP oracle call under
replacement of one row, then discarded. The final batch is never
fitted, and all oracle calls and action draws use fresh randomness.
Then the public output
\begin{equation*}
 \mathcal M(z)=((\widetilde f_m,\pi_m))_{m=2}^M
\end{equation*}
is $(\varepsilon_{\rm priv},\delta_{\rm priv})$-DP under replacement
of one record $z_t=(c_t,\ell_t(c_t,\cdot))$.
This statement holds for arbitrary fixed record streams.
\end{lemma}

\begin{proof}
Fix neighboring deterministic streams $z,z'$ that differ only at
round $i$, in epoch $j$. The schedule is public and does not depend
on the records.

\paragraph{Only one row changes in the affected batch.}
Condition on a common realization of all learner randomness before
epoch $j$. The earlier records agree, so the states and transcripts
before that epoch agree as well. In particular, $\pi_j$ is the
same fixed map in both executions.

Also fix common values of the independent action-selection seeds
within epoch $j$. For the efficient vanilla sampler, these determine
both $K_t$ and the sample from its softmax distribution.
At every round in this epoch other than $i$, the record, policy,
and action seed are the same, so the observed triple is the same.
At round $i$, changing the record may change the context, selected
action, and observed loss, but affects only that row. Thus
$\mathcal S_j$ and $\mathcal S'_j$ differ in at most one row.
The distribution of the conditioning variables is identical in the
two executions and independent of the changed record.

\paragraph{The remaining transcript is postprocessing of one fit.}
If $j=M$, no oracle is fitted to the affected batch. Its records
therefore have no effect on the released transcript.

Otherwise, the oracle call producing $\widetilde f_{j+1}$ is
$(\varepsilon_{\rm priv},\delta_{\rm priv})$-DP on the two
neighboring batches. The map $\pi_{j+1}$ is a deterministic
function of that output and public parameters. After this call,
the affected batch is discarded.

For the fixed earlier state and fixed future record stream, there
is a single randomized continuation procedure that maps
$\widetilde f_{j+1}$ to all subsequent releases: it constructs
the policies, interacts with the future records, and runs the later
oracles with fresh randomness. This procedure is the same in both
executions because all records after epoch $j$ agree. Later actions
and collected rows need not agree; they are generated by this common
procedure from its potentially different private input.

By postprocessing, for any event $E$ involving the complete released
transcript and any fixed value $v$ of the common conditioning
variables,
\begin{equation*}
 \Pp(\mathcal M(z)\in E\mid v)
 \le e^{\varepsilon_{\rm priv}}
       \Pp(\mathcal M(z')\in E\mid v)
       +\delta_{\rm priv}.
\end{equation*}
The earlier outputs are already fixed as part of $v$ and can be
included in $E$. Integrating over the common distribution of $v$
gives \eqref{eq:dp}, with the same additive
$\delta_{\rm priv}$. No oracle-accuracy event or stochastic
assumption on the record stream was used.
\end{proof}

\begin{proof}[Proof of \cref{thm:private-batched}]
Each epoch uses a fixed policy, and its batch enters at most one
private fit. Thus \cref{lem:private-dp} gives privacy of the transcript.
For regret, \cref{ass:private-oracle} supplies
\eqref{eq:batch-epoch-accuracy} with probability at least
$1-\delta/2$. The proof of \cref{thm:batched-vanilla}, using the
private oracle's error bound in \eqref{eq:vanilla-batch-rate},
first bounds $\sum_{t=1}^T\mathcal R(\pi_t)$ on this event.
The concentration step, with failure probability $\delta/2$,
then gives the stated bound on $\Reg$ with probability at least
$1-\delta$.
The schedule makes exactly $M-1$ private oracle calls, including
zero when $T=1$.
\end{proof}

The absence of a composition factor comes from using each batch
in only one private fit and discarding it afterward. The doubling
schedule controls the oracle count and regret; privacy itself does
not require geometrically increasing batch lengths. In contrast,
refitting after every round on all observations would use the same
record repeatedly and would require a different privacy analysis.

\subsection{Finite-class private regression}
\label{app:batch-oracle-parameters}

Let $\F$ be finite and realizable, and fix
$\delta\in(0,1/2]$ and $\varepsilon_{\rm priv}>0$.
Use the exponential mechanism in \eqref{eq:em} on each completed
training batch. By \cref{lem:private-exp-oracle}, every fit is pure
$\varepsilon_{\rm priv}$-DP under row replacement.
Applying its regression guarantee at failure probability
$\delta/(2M)$ gives the choice
\begin{equation}
 \Err_{\rm orc}^{\rm batch}(T,\delta)
 =\left(10+\frac4{\varepsilon_{\rm priv}}\right)
   \log\frac{4|\F|M}{\delta}.
 \label{eq:private-batch-error}
\end{equation}
This is a known bound of at least one; its dependence on
$\varepsilon_{\rm priv}$ is implicit in the oracle notation.
By \cref{lem:batch-accuracy-event}, all fitted predictors satisfy
\eqref{eq:batch-epoch-accuracy} with probability at least
$1-\delta/2$.

For batched vanilla PO, use this bound in the explicit learning rates
\eqref{eq:vanilla-batch-rate}. The regret analysis in
\cref{app:batched} then gives, with probability at least $1-\delta$,
\begin{equation*}
 \Reg=\widetilde O\!\left(
   \sqrt{AT\log(|\F|/\delta)}
   \bigl(1+\varepsilon_{\rm priv}^{-1/2}\bigr)
 \right).
\end{equation*}
Independently of this accuracy event, \cref{lem:private-dp} gives
pure $\varepsilon_{\rm priv}$-DP for the complete transcript.
This proves \cref{cor:private-exp}.
The same oracle bound also applies to batched bonus PO, with its
parameters specified in \cref{app:bonus-po}.
The nonprivate ERM choice is given in \eqref{eq:batch-erm-error}.

\section{Policy optimization with an exploration bonus}
\label{app:bonus-po}

\subsection{From the original bonus to a batched algorithm}
\label{sec:private-bonus}

The optimistic policy-optimization algorithm of \citet{levy2026}
subtracts an exploration bonus from the predicted loss before
clipping at zero and applying an exponential update. Its bonus
decreases with the sum of the probabilities that earlier policies
assign to an action at the same context. An action with little
past probability receives a larger bonus, which makes it more
attractive to the policy. The original algorithm fits predictors
throughout the interaction.

Our batched adaptation fits one predictor per epoch from the
preceding batch and holds it fixed. It constructs an auxiliary
sequence of bonus-adjusted exponential updates, starting from
uniform, and deploys their average throughout the epoch.
The probabilities entering the bonus are those of this auxiliary
sequence. These updates require no additional observations or
oracle calls.

We also subtract the minimum predicted loss before applying the
bonus and clipping. This ensures that a predicted best action
always has adjusted loss zero, which is used in the proof below.
Centering before clipping can change the update, so it is an explicit
part of the algorithm specification. The theorem below analyzes
this centered, batched construction.

\subsection{The batched bonus policy}

Fix $\delta\in(0,1/2]$ and use the doubling schedule of
\cref{alg:batched-vanilla-po},
with a uniform first-epoch policy and one fit to the preceding
batch. Let the known bound $\Err_{\rm orc}^{\rm batch}(T,\delta)\ge1$
satisfy \eqref{eq:batch-epoch-accuracy} for these policies
with probability at least $1-\delta/2$.
As in \cref{app:batched}, the remaining failure probability
$\delta/2$ is reserved for concentration over contexts.
For each $m\ge2$, set
\begin{equation}
 \begin{aligned}
 \gamma_m&=\min\left\{B_m,
   \sqrt{\frac{AB_{m-1}}{\Err_{\rm orc}^{\rm batch}(T,\delta)}}\right\},
 &\Lambda_m&=\frac{B_m}{\gamma_m},\\
 \eta_m&=\min\left\{1,\sqrt{\frac{2\log A}{B_m}}\right\}.
 \end{aligned}
 \label{eq:bonus-parameters}
\end{equation}
Here $\gamma_m$ controls exploration, $\Lambda_m$ sets the bonus
scale, and $\eta_m$ is the learning rate for the auxiliary updates.
The learning rate here is specific to the bonus algorithm.

For every context $c$, define $d_m(c,a)$ by
\eqref{eq:batch-predicted-gap}, initialize $\pi_{m,1}(c,a)=1/A$, and let
\begin{equation}
 N_{m,s}(c,a)=\sum_{r=1}^s \pi_{m,r}(c,a),
 \qquad N_{m,0}(c,a)=0.
 \label{eq:bonus-count}
\end{equation}
The index $s$ runs over auxiliary policies within an epoch's
construction: $\pi_{m,s}$ is obtained after $s-1$ updates.
The quantity $N_{m,s}(c,a)$ sums the probabilities assigned to
action $a$ by the first $s$ auxiliary policies at the same context
$c$; it does not count observed actions in a batch.
For $s=1,\ldots,B_m$, define
\begin{equation}
 \begin{aligned}
 b_{m,s}(c,a)
 &=\min\left\{1,\frac{2\Lambda_m}{1+N_{m,s-1}(c,a)}\right\},\\
 u_{m,s}(c,a)&=[d_m(c,a)-b_{m,s}(c,a)]_+,\\
 \pi_{m,s+1}(c,a)
 &=\frac{\pi_{m,s}(c,a)\exp\{-\eta_m u_{m,s}(c,a)\}}
 {\sum_{b\in\A}\pi_{m,s}(c,b)\exp\{-\eta_m u_{m,s}(c,b)\}},
 \end{aligned}
 \label{eq:bonus-update}
\end{equation}
where $[x]_+=\max\{x,0\}$. The deployed policy is
\begin{equation}
 \pi_m(c,a)=\frac1{B_m}\sum_{s=1}^{B_m}\pi_{m,s}(c,a)
           =\frac{N_{m,B_m}(c,a)}{B_m}.
 \label{eq:virtual-average}
\end{equation}
When an action has accumulated little probability, its bonus is
large and its adjusted loss $u_{m,s}(c,a)$ can be zero. Its weight
then remains unchanged while weights of other actions can decrease.
As it accumulates probability, the bonus decreases and its predicted
gap has more influence on the update.

\begin{algorithm}[ht]
\caption{Batched Policy Optimization with an Exploration Bonus}
\label{alg:private-BonusPO}
\begin{algorithmic}
\State \textbf{Input:} horizon $T$, action set $\A$, confidence $\delta$,
a regression oracle, and a known bound $\Err_{\rm orc}^{\rm batch}(T,\delta)$
for \eqref{eq:batch-epoch-accuracy}.
\State Set $M\gets\lceil\log_2(T+1)\rceil$ and
$B_m\gets2^{m-1}$ for $m=1,\ldots,M$.
\State Initialize $\pi_1(c,a)\gets1/A$ for every $c\in\C$, $a\in\A$.
\State \textbf{for} $m=1,\ldots,M$ \textbf{do}
\State \AlgIndent \textbf{if} $m\ge2$ \textbf{then}
\State \AlgIndent\AlgIndent Fit
$\widetilde f_m\gets\operatorname{Oracle}(\mathcal S_{m-1})$
and discard $\mathcal S_{m-1}$.
\State \AlgIndent\AlgIndent Set $\gamma_m,\Lambda_m,\eta_m$ by
\eqref{eq:bonus-parameters}.
\State \AlgIndent\AlgIndent Represent $\pi_m$ by $\widetilde f_m$
and the public construction
\eqref{eq:bonus-count}--\eqref{eq:virtual-average}.
\State \AlgIndent\AlgIndent In the private version, release
$(\widetilde f_m,\pi_m)$.
\State \AlgIndent \textbf{end if}
\State \AlgIndent Initialize $\mathcal S_m\gets()$.
\State \AlgIndent \textbf{for}
$t=B_m,\ldots,\min\{2B_m-1,T\}$ \textbf{do}
\State \AlgIndent\AlgIndent Observe $c_t$ and evaluate
$\pi_m(c_t,\cdot)$ by the procedure below.
\State \AlgIndent\AlgIndent Sample $a_t\sim\pi_m(c_t,\cdot)$
with fresh randomness and observe $\ell_t(c_t,a_t)$.
\State \AlgIndent\AlgIndent Append
$(c_t,a_t,\ell_t(c_t,a_t))$ to $\mathcal S_m$.
\State \AlgIndent \textbf{end for}
\State \textbf{end for}
\State \textbf{Policy evaluation for epoch $m$ at context $c$:}
\State \AlgIndent If $m=1$, return the uniform distribution.
\State \AlgIndent Otherwise, evaluate $\widetilde f_m(c,a)$ for all
$a\in\A$ and form $d_m(c,a)$.
\State \AlgIndent Initialize $\pi_{m,1}(c,a)\gets1/A$ and
$N_{m,0}(c,a)\gets0$ for every action.
\State \AlgIndent \textbf{for} $s=1,\ldots,B_m$ \textbf{do}
\State \AlgIndent\AlgIndent Compute $b_{m,s}(c,a)$, $u_{m,s}(c,a)$,
and $\pi_{m,s+1}(c,a)$ by \eqref{eq:bonus-update}.
\State \AlgIndent\AlgIndent Set
$N_{m,s}(c,a)\gets N_{m,s-1}(c,a)+\pi_{m,s}(c,a)$ for each action.
\State \AlgIndent \textbf{end for}
\State \AlgIndent Return $\pi_m(c,a)=N_{m,B_m}(c,a)/B_m$.
\end{algorithmic}
\end{algorithm}

\paragraph{Implementation.}
The predictor and policy map remain fixed throughout an epoch.
The evaluation procedure computes their values only at the queried
context, using $A$ predictor evaluations and $O(AB_m)$ arithmetic.
The sums $N_{m,s}(c,a)$ are initialized anew for each evaluation;
they are not carried over between bandit rounds. Repeating an
evaluation at the same context therefore gives the same distribution.
The update producing $\pi_{m,B_m+1}$ can be omitted, because that
iterate is not included in the average.
A shortened final epoch still uses all $B_M$ terms in
\eqref{eq:virtual-average}, and its batch is never fitted.

\subsection{Regret and privacy guarantees}

\begin{theorem}[Batched bonus PO]
\label{thm:batched-bonus}
Fix $\delta\in(0,1/2]$ and suppose
\eqref{eq:batch-epoch-accuracy} holds with probability at least
$1-\delta/2$, for a known bound
$\Err_{\rm orc}^{\rm batch}(T,\delta)\ge1$ fixed before the interaction.
Then \cref{alg:private-BonusPO}, with parameters
\eqref{eq:bonus-parameters}, makes exactly $M-1$ oracle calls and,
with probability at least $1-\delta$, satisfies
\begin{equation}
 \begin{aligned}
 \Reg\le{}&1+\bigl(32+8\log(T+1)\bigr)
              \bigl(AM+5\sqrt{AT\,\Err_{\rm orc}^{\rm batch}(T,\delta)}\bigr)\\
          &+\sqrt{\frac T2\log\frac2\delta}.
 \end{aligned}
 \label{eq:bonus-total-regret}
\end{equation}
Together with $\Reg\le T$, this gives
$\Reg=\widetilde O(\sqrt{AT\,\Err_{\rm orc}^{\rm batch}(T,\delta)})$.
If the oracle also satisfies \cref{ass:private-oracle}, then the
released transcript is
$(\varepsilon_{\rm priv},\delta_{\rm priv})$-DP for arbitrary fixed
record streams, with no change to the regret guarantee.
\end{theorem}

\paragraph{Proof overview.}
The bonus keeps an action's unnormalized weight at its initial value
until that action has accumulated enough probability. Either it
receives this much probability or its probability remains at least
$1/A$ throughout the auxiliary sequence. This gives an explicit
lower bound on the averaged policy.
To bound predicted regret, we add the exponential-weights bound
for the adjusted losses to the total bonuses. Since each bonus
decreases as accumulated probability grows, its probability-weighted
sum is bounded by the bonus scale times a logarithmic factor.
The probability bound lets us apply
\cref{lem:batch-regret-comparison}; the predicted-regret bound then
controls the population regret in each epoch. Summing over epochs
and applying concentration gives the bound on $\Reg$.
Privacy follows from the same argument as for batched vanilla PO:
each batch enters one private fit, and subsequent releases depend
on that batch only through the private predictor.

\begin{lemma}[Action probabilities and predicted regret of batched bonus PO]
\label{lem:bonus-properties}
For every epoch $m\ge2$, context $c$, and action $a$,
\begin{equation}
 \pi_m(c,a)
 \ge\min\left\{\frac1A,\frac1{\gamma_m d_m(c,a)}\right\},
 \qquad
 \pi_m(c,a)^{-1}\le A+\gamma_m d_m(c,a),
 \label{eq:bonus-coverage}
\end{equation}
where the second term in the minimum is $+\infty$ when $d_m(c,a)=0$.
Moreover,
\begin{equation}
 \widehat{\mathcal R}_m(\pi_m)
 \le\bigl(1+4\log(B_m+1)\bigr)\frac A{\gamma_m}.
 \label{eq:bonus-regret}
\end{equation}
\end{lemma}

\begin{proof}
Fix $m$ and $c$. Define unnormalized weights by $w_{m,1}(c,a)=1$ and
\begin{equation*}
 w_{m,s+1}(c,a)
 =w_{m,s}(c,a)\exp\{-\eta_m u_{m,s}(c,a)\}.
\end{equation*}
These weights never increase, since $u_{m,s}(c,a)\ge0$, and
$\pi_{m,s}(c,a)=w_{m,s}(c,a)/\sum_b w_{m,s}(c,b)$.

\paragraph{Lower bounds on action probabilities.}
If $d_m(c,a)=0$, then $u_{m,s}(c,a)=0$ at every step. Hence
$w_{m,s}(c,a)=1$ and $\pi_{m,s}(c,a)\ge1/A$ throughout.

Now suppose $d_m(c,a)>0$. Since $\Lambda_m\ge1\ge d_m(c,a)$,
whenever $N_{m,s-1}(c,a)<\Lambda_m/d_m(c,a)$ we have
\begin{equation*}
 \frac{2\Lambda_m}{1+N_{m,s-1}(c,a)}
 >\frac{2\Lambda_m d_m(c,a)}{\Lambda_m+d_m(c,a)}
 \ge d_m(c,a).
\end{equation*}
Capping this quantity at one still leaves it at least $d_m(c,a)$,
so $u_{m,s}(c,a)=0$. The counts are nondecreasing, so all preceding
adjusted losses of this action were also zero. Its weight is
therefore still one, and $\pi_{m,s}(c,a)\ge1/A$.

Either $N_{m,B_m}(c,a)\ge\Lambda_m/d_m(c,a)$, or this threshold is
never reached and $\pi_{m,s}(c,a)\ge1/A$ at every step. In both cases,
\begin{equation*}
 N_{m,B_m}(c,a)
 \ge\min\left\{\frac{B_m}{A},\frac{\Lambda_m}{d_m(c,a)}\right\}.
\end{equation*}
Divide by $B_m$ and use $\Lambda_m/B_m=1/\gamma_m$ to obtain the
first inequality in \eqref{eq:bonus-coverage}. Taking reciprocals
and using $\max\{A,\gamma_m d_m(c,a)\}\le A+\gamma_m d_m(c,a)$
gives the second.

\paragraph{Regret for the adjusted losses.}
Choose an action $a_0$ with $d_m(c,a_0)=0$. It has adjusted loss zero
and weight one at every step. Thus the total weight is initially
$A$ and is at least one after all $B_m$ updates.
Applying the same log-weight calculation as in the proof of
\cref{lem:virtual-properties}, now to the varying adjusted losses,
gives
\begin{align}
 \sum_{s=1}^{B_m}\sum_{a\in\A}\pi_{m,s}(c,a)u_{m,s}(c,a)
 &\le\frac{\log A}{\eta_m}
   +\frac{\eta_m}{2}\sum_{s=1}^{B_m}\sum_{a\in\A}
                    \pi_{m,s}(c,a)u_{m,s}(c,a)^2\notag\\
 &\le\frac{\log A}{\eta_m}+\frac{\eta_m B_m}{2},
 \label{eq:bonus-hedge}
\end{align}
because every adjusted loss lies in $[0,1]$.

\paragraph{Adding back the bonuses.}
The definition of the adjusted loss gives
$d_m(c,a)\le u_{m,s}(c,a)+b_{m,s}(c,a)$.
For each action, the identity
$N_{m,s}(c,a)=N_{m,s-1}(c,a)+\pi_{m,s}(c,a)$ and the inequality
$x\le2\log(1+x)$ for $0\le x\le1$ yield
\begin{align*}
 \sum_{s=1}^{B_m}\pi_{m,s}(c,a)b_{m,s}(c,a)
 &\le2\Lambda_m\sum_{s=1}^{B_m}
        \frac{\pi_{m,s}(c,a)}{1+N_{m,s-1}(c,a)}\\
 &\le4\Lambda_m\sum_{s=1}^{B_m}
        \log\frac{1+N_{m,s}(c,a)}{1+N_{m,s-1}(c,a)}\\
 &=4\Lambda_m\log\bigl(1+N_{m,B_m}(c,a)\bigr)\\
 &\le4\Lambda_m\log(B_m+1).
\end{align*}
The logarithms telescope because the probability added to the count
at each step is the same probability multiplying the bonus.
Sum over actions, combine with \eqref{eq:bonus-hedge}, divide by
$B_m$, and average over contexts to get
\begin{equation}
 \widehat{\mathcal R}_m(\pi_m)
 \le\frac{\log A}{\eta_m B_m}+\frac{\eta_m}{2}
       +\frac{4A\log(B_m+1)}{\gamma_m}.
 \label{eq:bonus-intermediate}
\end{equation}

If $B_m<2\log A$, then $\gamma_m\le B_m<2\log A\le A$,
so the simpler bound
$\widehat{\mathcal R}_m(\pi_m)\le1\le A/\gamma_m$ already proves
\eqref{eq:bonus-regret}.
Otherwise $\eta_m=\sqrt{2\log A/B_m}$, and the first two terms
in \eqref{eq:bonus-intermediate} sum to
$\sqrt{2\log A/B_m}$. By \eqref{eq:bonus-parameters},
$B_m=2B_{m-1}$, and $\Err_{\rm orc}^{\rm batch}(T,\delta)\ge1$,
\begin{equation*}
 \gamma_m\sqrt{\frac{2\log A}{B_m}}
 \le\sqrt{\frac{A\log A}{\Err_{\rm orc}^{\rm batch}(T,\delta)}}
 \le\sqrt{A\log A}\le A.
\end{equation*}
Thus these two terms are at most $A/\gamma_m$. Substitution into
\eqref{eq:bonus-intermediate} proves \eqref{eq:bonus-regret}.
\end{proof}

\begin{proof}[Proof of \cref{thm:batched-bonus}]
As in \cref{app:algorithms}, the schedule makes exactly $M-1$ fits,
and every batch used in a fit is complete. If $T=1$, the only
policy is uniform and the regret is at most one.

For $T\ge2$, work on the simultaneous accuracy event
\eqref{eq:batch-epoch-accuracy}, which has probability at least
$1-\delta/2$ by hypothesis.
The sequence $(\gamma_m)_{m=2}^M$ in \eqref{eq:bonus-parameters}
is nondecreasing and satisfies
\begin{equation*}
 \gamma_m^2\le\frac{AB_{m-1}}{\Err_{\rm orc}^{\rm batch}(T,\delta)}.
\end{equation*}
Together with \eqref{eq:bonus-coverage}, this verifies
\eqref{eq:batch-comparison-conditions} with $C_0=1$.
Thus \cref{lem:batch-regret-comparison} applies to every comparison
policy, in particular the deployed policy $\pi_m$. Combining it with
\eqref{eq:bonus-regret} gives
\begin{equation*}
 \mathcal R(\pi_m)
 \le\bigl(32+8\log(B_m+1)\bigr)\frac A{\gamma_m}
 \le\bigl(32+8\log(T+1)\bigr)\frac A{\gamma_m}.
\end{equation*}
Let $n_m=\min\{B_m,T-B_m+1\}$ be the number of rounds in
epoch $m$. Using $n_m\le B_m$, the first epoch's population regret
bound of one, and
\begin{equation*}
 \frac{AB_m}{\gamma_m}
 =\max\{A,\sqrt{2A\Err_{\rm orc}^{\rm batch}(T,\delta)B_m}\}
 \le A+\sqrt{2A\Err_{\rm orc}^{\rm batch}(T,\delta)B_m},
\end{equation*}
we obtain
\begin{align*}
 \sum_{m=1}^M n_m\mathcal R(\pi_m)
 &\le1+\bigl(32+8\log(T+1)\bigr)
          \sum_{m=2}^M\frac{AB_m}{\gamma_m}\\
 &\le1+\bigl(32+8\log(T+1)\bigr)
          \bigl(AM+5\sqrt{AT\,\Err_{\rm orc}^{\rm batch}(T,\delta)}\bigr).
\end{align*}
The last step uses \eqref{eq:batch-length-sum} and $3.5\sqrt2<5$.
This bounds the regret averaged over contexts on the accuracy event.

Each deployed policy is determined before its epoch begins, so
\cref{cor:mean-regret}, applied with failure probability $\delta/2$,
gives
\begin{equation*}
 \Reg\le\sum_{m=1}^M n_m\mathcal R(\pi_m)
       +\sqrt{\frac T2\log\frac2\delta}
\end{equation*}
with probability at least $1-\delta/2$.
This concentration statement holds without conditioning on the
oracle's accuracy event. A union bound with that event proves
\eqref{eq:bonus-total-regret} with probability at least $1-\delta$.
Combining this with $\Reg\le T$ gives the stated $\widetilde O$
bound, as in the proof of \cref{thm:batched-vanilla}; the additional
concentration term is absorbed up to logarithmic factors in
$1/\delta$, because
$A\ge2$ and $\Err_{\rm orc}^{\rm batch}(T,\delta)\ge1$.

For privacy, the full policy construction
\eqref{eq:bonus-count}--\eqref{eq:virtual-average} is a deterministic
function of $\widetilde f_m$ and public parameters. In particular,
its auxiliary counts use no raw batch observations. Each batch is
collected under a fixed map and used in only one private fit.
The conditions of \cref{lem:private-dp} therefore hold, proving
privacy of the full transcript for arbitrary fixed streams,
independently of the accuracy event.
\end{proof}

\begin{corollary}[Finite-class batched bonus PO]
\label{cor:bonus-finite}
Let $\F$ be finite and realizable, and fix $\delta\in(0,1/2]$.
With least-squares ERM and the
error bound \eqref{eq:batch-erm-error},
batched bonus PO satisfies
\begin{equation*}
 \Reg=\widetilde O\!\left(\sqrt{AT\log(|\F|/\delta)}\right)
\end{equation*}
with probability at least $1-\delta$.
For any $\varepsilon_{\rm priv}>0$, with the exponential mechanism
in \eqref{eq:em} and $\Err_{\rm orc}^{\rm batch}(T,\delta)$ from
\eqref{eq:private-batch-error}, its transcript is pure
$\varepsilon_{\rm priv}$-DP and, with probability at least $1-\delta$,
\begin{equation*}
 \Reg=\widetilde O\!\left(
   \sqrt{AT\log(|\F|/\delta)}
   \bigl(1+\varepsilon_{\rm priv}^{-1/2}\bigr)
 \right).
\end{equation*}
\end{corollary}

\begin{proof}
The oracle bounds in \cref{app:regression}, applied across epochs
by \cref{lem:batch-accuracy-event}, verify the accuracy condition
of \cref{thm:batched-bonus} with the stated error bounds.
Use each bound in the bonus parameters \eqref{eq:bonus-parameters}.
Substituting the corresponding
$\Err_{\rm orc}^{\rm batch}(T,\delta)$ proves the regret bounds.
For the exponential mechanism,
$\delta_{\rm priv}=0$, so the transcript guarantee is pure DP.
\end{proof}

\section{Experiments}
\label{app:experiments}

Following the contextual-bandit evaluation protocol of \citet{bietti2021bakeoff}, we evaluate the algorithms on randomly permuted multiclass classification datasets from \textsc{OpenML}. The following sections describe the regression oracles and their private counterparts, the algorithm implementations and hyperparameter tuning, the experimental protocol, and the resulting empirical comparisons. Code for reproducing the experiments is publicly available at \url{https://github.com/alex-rbch/private-contextual-bandits}.

\subsection{Function classes and regression objectives}
\label{app:exp-oracles}

We first describe the feature representation and regression oracles used throughout the experiments. In the quadratic-loss setting, we use linear predictors fitted by least squares, while in the logistic-loss setting, we use logistic predictors fitted by logistic regression. All methods use the same feature construction and regression objectives; they differ only in the observations supplied to the oracle and, for the private variants, how the fit is privatized.

\paragraph{Action-dependent features.}
We represent each context--action pair using a block-structured feature map. Given a context $c\in\mathbb R^{d_0}$ and an action $a\in \A=[A]$, we first project the context onto the unit Euclidean ball, augment it with a bias coordinate, and then place the resulting vector in the block corresponding to action $a$:
\begin{equation*}
    \phi(c,a)
    = e_a\otimes b(\bar c)\in\mathbb R^d,
    \qquad
    \bar c=\frac{c}{\max\{1,\|c\|_2\}},
    \qquad
    b(\bar c)=\frac{(1,\bar c)}{\sqrt{2}},
    \qquad
    d=A(d_0+1),
\end{equation*}
where $e_a\in\mathbb R^A$ is the $a$th standard basis vector and $\otimes$ denotes the Kronecker product.\footnote{The map $b$ incorporates the constant bias coordinate used in the implementation.}
Thus, $\phi(c,a)$ has $b(\bar c)$ in its $a$th block and zeros elsewhere, and in particular $\|\phi(c,a)\|_2\le1$.

\paragraph{Function classes.}
We consider linear and logistic predictors over these features:
\begin{equation*}
    \mathcal F_{\rm lin}
    =\left\{(c,a)\mapsto\phi(c,a)^\top\theta:\theta\in\mathbb R^d\right\},
    \qquad
    \mathcal F_{\rm log}
    =\left\{(c,a)\mapsto\sigma\!\left(\phi(c,a)^\top\theta\right):
    \theta\in\mathbb R^d\right\},
\end{equation*}
where $\sigma(s)=(1+e^{-s})^{-1}$ is the sigmoid function.

\paragraph{Batching and regression objectives.}
All algorithms follow the same doubling schedule, with batch $j$ having size $B_j=2^{j-1}$, up to truncation of the final batch. Let $n_j=B_j$ for the bandit algorithms and $n_j=AB_j$ for the supervised reference. For each completed batch, let
\[
    X_j\in\mathbb R^{n_j\times d},
    \qquad
    y_j\in\mathbb R^{n_j},
\]
denote the corresponding feature matrix and vector of observed $0$--$1$ classification losses, where the loss is $0$ for a correct prediction and $1$ otherwise. For a bandit algorithm, each round contributes the single observation $(\phi(c_t,a_t),\ell_t(a_t))$, whereas the supervised reference contributes $A$ observations.

For the linear class, define the regularized quadratic objective
\begin{equation}
    L_j^{\rm lin}(\theta)
    =
    \frac{1}{2n_j}\|X_j\theta-y_j\|_2^2
    +\frac{\lambda}{2}\|\theta\|_2^2.
    \label{eq:exp-square-objective}
\end{equation}
For the logistic class, define the regularized logistic objective
\begin{equation}
    L_j^{\rm log}(\theta)
    =
    \frac1{n_j}\sum_{i=1}^{n_j}
    \left[
        \log\!\left(1+e^{x_{ji}^\top\theta}\right)
        -y_{ji}x_{ji}^\top\theta
    \right]
    +\frac{\lambda}{2}\|\theta\|_2^2,
    \label{eq:exp-log-objective}
\end{equation}
where $x_{ji}^\top$ denotes row $i$ of $X_j$.
Both objectives are averaged over the $n_j$ observations available to the oracle, so the same definitions apply directly to the supervised reference.

\subsection{Private regression mechanisms}
\label{app:private-oracle}

We use two standard mechanisms to privatize the regression oracles. For the quadratic objective, we perturb the sufficient statistics to obtain an $(\varepsilon,\delta)$-DP predictor, while for the logistic objective, we use objective perturbation to obtain a pure $\varepsilon$-DP predictor. The corresponding non-private oracles use the same regression procedures with all perturbations set to zero.

Privacy is defined at the level of one original context--label record. A bandit record contributes one fitting row, whereas a supervised record contributes $A$ rows. Since the batches are disjoint, releases across batches compose in parallel, and all subsequent policy computations are post-processing of the private oracle outputs.

\paragraph{Private quadratic oracle.}
For batch $j$, let
\begin{equation*}
    G_j=X_j^\top X_j,
    \qquad
    q_j=X_j^\top y_j.
\end{equation*}
Their record-level sensitivities in Frobenius and Euclidean norm are
\begin{equation*}
    (\Delta_G,\Delta_q)
    =
    \begin{cases}
        (\sqrt{2},\,\sqrt{2}), & \text{bandit feedback},\\
        (\sqrt{2A},\,\sqrt{2(A-1)}), & \text{supervised feedback},
    \end{cases}
\end{equation*}
where the supervised response bound uses the single-label $0$--$1$ loss structure.

To calibrate the Gaussian noise, define
\begin{equation*}
    \rho(\varepsilon,\delta)
    =
    \left(
        \sqrt{\log(1/\delta)+\varepsilon}
        -\sqrt{\log(1/\delta)}
    \right)^2.
\end{equation*}
We split this zCDP budget equally between the two sufficient statistics and release
\begin{equation}
    \widetilde G_j=G_j+E_j,
    \qquad
    \widetilde q_j=q_j+Z_j,
    \qquad
    \widehat\theta_{j+1}^{\rm lin}
    =
    (\widetilde G_j+n_j\lambda I)^{-1}\widetilde q_j,
    \label{eq:exp-private-ridge}
\end{equation}
where
\begin{equation*}
    Z_j\sim
    \mathcal N\!\left(0,\tfrac{\Delta_q^2}{\rho(\varepsilon,\delta)}I_d\right),
    \qquad
    E_j=\tfrac{W_j+W_j^\top}{2},
\end{equation*}
and the entries of $W_j$ are independent
$\mathcal N(0,\Delta_G^2/\rho(\varepsilon,\delta))$ variables. In the Frobenius geometry of symmetric matrices, this assigns half of the total zCDP budget to each statistic. Their joint release is therefore $(\varepsilon,\delta)$-DP \citep{bun2016concentrated}, and the resulting predictor is private by post-processing.

\paragraph{Private logistic oracle.}
Let $\varepsilon_P$ denote the privacy budget allocated to the predictor and set
\begin{equation*}
    h=
    \begin{cases}
        1, & \text{bandit feedback},\\
        A, & \text{supervised feedback},
    \end{cases}
    \qquad
    \varepsilon_{\rm fit}=\frac{\varepsilon_P}{h}.
\end{equation*}
The logistic loss has gradient norm at most one and scalar second derivative at most $c_0=1/4$. Following \citet{chaudhuri2011private}, define
\begin{equation*}
    u_j
    =
    \varepsilon_{\rm fit}
    -2\log\!\left(1+\tfrac{c_0}{n_j\lambda}\right),
\end{equation*}
and set
\begin{equation*}
    (\kappa_j,\varepsilon'_j)
    =
    \begin{cases}
        (0,u_j), & u_j>0,\\[1mm]
        \left(
        \displaystyle
        \frac{c_0}{n_j(e^{\varepsilon_{\rm fit}/4}-1)}-\lambda,
        \displaystyle\frac{\varepsilon_{\rm fit}}{2}
        \right), & u_j\le0.
    \end{cases}
\end{equation*}
We draw $b_j\in\mathbb R^d$ with density
\begin{equation*}
    p(b_j)\propto
    \exp\!\left(-\frac{\varepsilon'_j}{2}\|b_j\|_2\right)
\end{equation*}
and return the minimizer
\begin{equation}
    \widehat\theta_{j+1}^{\rm log}
    =
    \arg\min_{\theta\in\mathbb R^d}
    \left\{
        L_j^{\rm log}(\theta)
        +\frac{\kappa_j}{2}\|\theta\|_2^2
        +\frac{b_j^\top\theta}{n_j}
    \right\}.
    \label{eq:exp-objective-perturbation}
\end{equation}
The resulting estimator is $\varepsilon_{\rm fit}$-DP under replacement of one fitting row; group privacy therefore gives pure $\varepsilon_P$-DP at the original-record level. We solve the perturbed objective numerically using L-BFGS, a limited-memory quasi-Newton method \citep{liu1989limited}.\footnote{The stated privacy guarantee corresponds to the exact minimizer; L-BFGS approximates this solution.}

\paragraph{Private confidence radii.}
Our implementations of several algorithms rely on confidence radii constructed from empirical Gram matrices. Given privatized Gram matrices $(\widetilde G_i)_{i=1}^j$, we use the radius
\begin{equation*}
    r_j(c,a)
    =
    \sqrt{
        \max\left\{
            \phi(c,a)^\top
            (\textstyle\sum_{i=1}^j\widetilde G_i+I)^{-1}
            \phi(c,a),
            0
        \right\}
    }.
\end{equation*}
For the quadratic oracle, privatized $\widetilde G_j$ are already released as part of \eqref{eq:exp-private-ridge} and can therefore be reused without any additional privacy cost. For the logistic oracle, algorithms requiring confidence radii allocate $\varepsilon_P=\varepsilon/2$ to the predictor and $(\varepsilon/2,\delta)$ to a separate Gaussian release of $G_j$. Specifically, we release $\widetilde G_j=G_j+E_j$ with symmetric Gaussian noise of Frobenius-coordinate standard deviation
$\sigma_G=\Delta_G/\sqrt{2\rho(\varepsilon/2,\delta)}$.
Here $\Delta_G=\sqrt2$ for the bandit algorithms considered. The predictor and Gram matrix are therefore jointly $(\varepsilon,\delta)$-DP. Logistic methods that do not require confidence radii instead allocate the full budget $\varepsilon_P=\varepsilon$ to the predictor, yielding pure $\varepsilon$-DP.

\subsection{Algorithm implementations and hyperparameter tuning}
\label{app:algorithm-implementations}

Here we describe the implementation choices and algorithm-specific hyperparameters used in the experiments. All algorithms share the regression regularization parameter $\lambda$; the remaining hyperparameters are specific to the policy. At each policy update, the algorithms use the released loss predictions $\widehat f_a=\widehat f(c,a)$ and, when required, the confidence radius $r_a=r_j(c,a)$ defined in the previous section. All subsequent policy computations are therefore post-processing.

\paragraph{New algorithms.}
\algname{VanillaPO} (\cref{alg:batched-vanilla-po}) uses only the predicted losses $\widehat f_a$. We parameterize its learning rate as $\eta_m = \eta_0 \min\{1/A, 1/\sqrt{AB_m}\}$, where $\eta_0$ is the tunable hyperparameter.

A direct implementation of \algname{BonusPO} (\cref{alg:private-BonusPO}) would require $O(T^2)$ computation. We therefore use a historical-replay modification with $O(T\log T)$ computation, defined as follows. For a context $c$ in batch $m$, let $f_{j,a}=\widehat f_{j+1}(c,a)$ denote the prediction released after batch $j$. Initialize $q_0(a)=1/A$ and $D_0(a)=0$. At replay step $j$, let $T_j=\sum_{i=1}^j B_i$ and define
\begin{equation*}
    b_{j,a}
    =
    \min\left\{
        1,\,
        \tfrac{\gamma\sqrt{T_j/A}}{1+D_{j-1}(a)}
    \right\},
    \qquad
    \widetilde f_{j,a}
    =
    \max\left\{
        0,\,
        f_{j,a}-b_{j,a}
    \right\}.
\end{equation*}
The policy and cumulative action mass are then updated as
\begin{equation*}
    q_j(a)
    \propto
    q_{j-1}(a)
    \exp\!\left(-\eta B_j\widetilde f_{j,a}\right),
    \qquad
    D_j(a)
    =
    D_{j-1}(a)+B_jq_{j-1}(a).
\end{equation*}
The policy used in batch $m$ is $q_{m-1}$. The hyperparameters are the learning rate $\eta$ and bonus scale $\gamma$.

\paragraph{Prediction-only baselines.}
\algname{SquareCB} and \algname{FastCB} use only the predicted losses $\widehat f_a$. For both algorithms, we use a batch-wise schedule
$\gamma_m=\gamma_0(\sum_{j<m} B_j)^\rho$, where $\gamma_0$ controls the overall exploration scale and $\rho$ controls how quickly the policy becomes more concentrated as additional data are collected. Larger values of $\gamma_m$ place less probability on actions with larger predicted loss gaps. We evaluate \algname{SquareCB} with both quadratic and logistic oracles. \algname{FastCB}, whose allocation rule is designed for log-loss regression, is evaluated only with the logistic oracle.

\paragraph{Exploration-radius baselines.}
Confidence-based contextual bandit algorithms use action-dependent uncertainty estimates to guide exploration. In our implementations of \algname{LinUCB}, \algname{RegCB}, and \algname{AdaCB}, we use the common elliptical radius $r_a$ as a proxy for this uncertainty. \algname{LinUCB} selects actions by minimizing the optimistic loss score $\widehat f_a-\beta r_a$, where $\beta$ scales the exploration bonus; we use this rule with both quadratic and logistic oracles. \algname{RegCB} forms raw scores $g_a=\phi(c,a)^\top\widehat\theta$ and retains the survivor set $\mathcal S=\{a:g_a-\beta r_a\le\min_b(g_b+\beta r_b)\}$, sampling uniformly over $\mathcal S$. Here $\beta$ scales the confidence widths, so, for fixed predictions and radii, larger values can retain more actions. For the logistic oracle, these comparisons are still performed in raw-score space; monotonicity of the sigmoid preserves the corresponding ordering after applying the link function. \algname{AdaCB} uses the same survivor set, then applies the \algname{SquareCB} allocation to the predicted losses $\widehat f_a$ restricted to $\mathcal S$, with zero probability outside the set. Its parameters $\gamma_0$ and $\rho$ have the same roles as for \algname{SquareCB}, while $\beta$ controls the elimination step.

\paragraph{Supervised reference.}
\algname{Supervised} serves as a full-feedback reference, observing the losses of all $A$ actions at each context and fitting the same regression oracles as above. It selects uniformly among actions minimizing the predicted loss and has no additional policy hyperparameters.

\subsection{Experimental protocol}
\label{app:exp-design}

We describe the datasets, privacy settings, evaluation metric, and model-selection procedure used throughout the experiments.

\paragraph{Datasets.}
We evaluate the algorithms on 100 unique \textsc{OpenML} classification datasets,\footnote{\textsc{OpenML} (\url{https://www.openml.org/}) contains multiple versions or copies of some underlying datasets. We screened the collection to avoid repeated datasets and ensure diversity across the benchmark.} spanning a broad range of problem sizes and dimensions: the datasets contain between 1{,}000 and 89{,}640 examples and between 2 and 10 classes. We interpret each example as one contextual-bandit round and each class as an action. For each run, we use the full dataset, shuffled without replacement.

\paragraph{Evaluation settings.}
We evaluate each method in the non-private setting and in the $(\varepsilon,\delta)$-DP setting with $\varepsilon\in\{8,4,2,1\}$ and $\delta=10^{-5}$. All methods follow the same doubling batch schedule and regression-objective normalization described above. We run 10 random seeds for every configuration. All algorithms except \algname{FastCB} are evaluated with both quadratic and logistic oracles, while \algname{FastCB} uses only the logistic oracle. Across 100 datasets, five privacy settings, and 10 seeds, the full hyperparameter sweep comprises $100\times5\times10\times800=4{,}000{,}000$ runs.

\paragraph{Evaluation metric.}
We report progressive-validation loss
$L_{\rm PV}(t)=t^{-1}\sum_{s=1}^t\ell_s(a_s)$. For algorithms evaluated on the same sequence of examples, differences in final progressive-validation loss equal differences in realized regret relative to a common comparator, divided by $T$.

\paragraph{Model selection.}
We tune the common regression parameter $\lambda$ and the algorithm-specific hyperparameters over the grids in \cref{tab:exp-hyperparameters}. Here, $\ell_{\rm sq}$ and $\ell_{\rm log}$ denote the quadratic and logistic regression oracle losses, respectively. For algorithms evaluated with both oracles, the oracle type is included in model selection. For each algorithm, dataset, and privacy level, we report the configuration with the lowest mean final progressive-validation loss across seeds.

\begin{table*}[t]
\centering
\setlength{\tabcolsep}{6pt}
\renewcommand{\arraystretch}{1.15}
\resizebox{\textwidth}{!}{\begin{tabular}{@{}llll@{}}
\toprule
\textbf{Algorithm}
& \textbf{Policy hyperparameters}
& $\boldsymbol{\lambda}$
& \textbf{Oracle loss} \\
\midrule

\algname{VanillaPO}
& $\eta_0\in 2^{\{1,3,5,7,9\}}$
& $10^{\{-6,-4,-2,-1,0\}}$
& $\ell_{\rm sq},\ell_{\rm log}$ \\

\algname{BonusPO}
& $\eta\in10^{\{-3,-2,-1,0\}},\;
   \gamma\in10^{\{-3,-2,-1\}}$
& $10^{\{-6,-4,-2,-1,0\}}$
& $\ell_{\rm sq},\ell_{\rm log}$ \\

\algname{SquareCB}
& $\gamma_0\in2^{\{0,2,4,6,8\}},\;
   \rho\in\{0.25,0.5\}$
& $10^{\{-6,-4,-2,-1,0\}}$
& $\ell_{\rm sq},\ell_{\rm log}$ \\

\algname{FastCB}
& $\gamma_0\in10^{\{-1,0,1,2,3\}},\;
   \rho\in\{0.25,0.5\}$
& $10^{\{-6,-4,-2,-1,0\}}$
& $\ell_{\rm log}$ \\

\algname{LinUCB}
& $\beta\in2^{\{-12,-8,-7,-5,-2,0,1,2,3\}}$
& $10^{\{-6,-4,-2,-1,0\}}$
& $\ell_{\rm sq},\ell_{\rm log}$ \\

\algname{RegCB}
& $\beta\in10^{\{-5,-3,-2,-1,0,1\}}$
& $10^{\{-6,-4,-2,-1,0\}}$
& $\ell_{\rm sq},\ell_{\rm log}$ \\

\algname{AdaCB}
& $\gamma_0\in10^{\{-1,0,1,2\}},\;
   \rho\in\{0.25,0.5\},\;
   \beta\in10^{\{-2,-1,0,1\}}$
& $10^{\{-6,-4,-2,-1,0\}}$
& $\ell_{\rm sq},\ell_{\rm log}$ \\

\algname{Supervised}
& --
& $10^{\{-6,-4,-2,-1,0\}}$
& $\ell_{\rm sq},\ell_{\rm log}$ \\

\bottomrule
\end{tabular}}
\caption{Hyperparameter grids used in the experiments.}
\label{tab:exp-hyperparameters}
\end{table*}

\subsection{Results and discussion}
\label{app:exp-results}

We report results at two levels. We first aggregate the final progressive-validation losses over all 100 \textsc{OpenML} datasets, after selecting the best hyperparameter configuration separately for each algorithm, dataset, and privacy regime. We then examine three representative datasets in more detail to show how the learning dynamics change with privacy.

\subsubsection{Aggregate results across datasets}
\label{app:agg-results}

\paragraph{Pairwise win rates.}
Figure~\ref{fig:heatmaps} summarizes pairwise win rates across the 100 datasets. For each ordered pair of algorithms, the corresponding entry gives the fraction of datasets on which the first algorithm attains lower final progressive-validation loss than the second. The same pattern appears consistently across private regimes and agrees with the trends seen in the main text. In particular, \algname{BonusPO} outperforms every competing bandit method on roughly 70\%--95\% of datasets, making it the clear empirical winner in this comparison. The next tier consists of \algname{VanillaPO}, \algname{FastCB}, and \algname{SquareCB}, whose performance is broadly comparable. The final tier consists of \algname{LinUCB}, \algname{RegCB}, and \algname{AdaCB}, with \algname{AdaCB} typically performing slightly better than the other two. In the non-private setting, \algname{BonusPO} remains dominant, while the other bandit algorithms perform comparably, with \algname{AdaCB} holding a slight edge.

This separation under privacy is consistent with the structure of the algorithms. The stronger performance of \algname{BonusPO} indicates that the bonus remains useful under private batched regression, as it incorporates uncertainty information that is absent from the raw predicted losses alone. By contrast, \algname{LinUCB}, \algname{RegCB}, and \algname{AdaCB} rely directly on a privatized exploration radius, which appears to make them more sensitive to the noise introduced by the privacy mechanism.

\begin{figure}[t]
    \centering
    \includegraphics[width=0.77\linewidth]{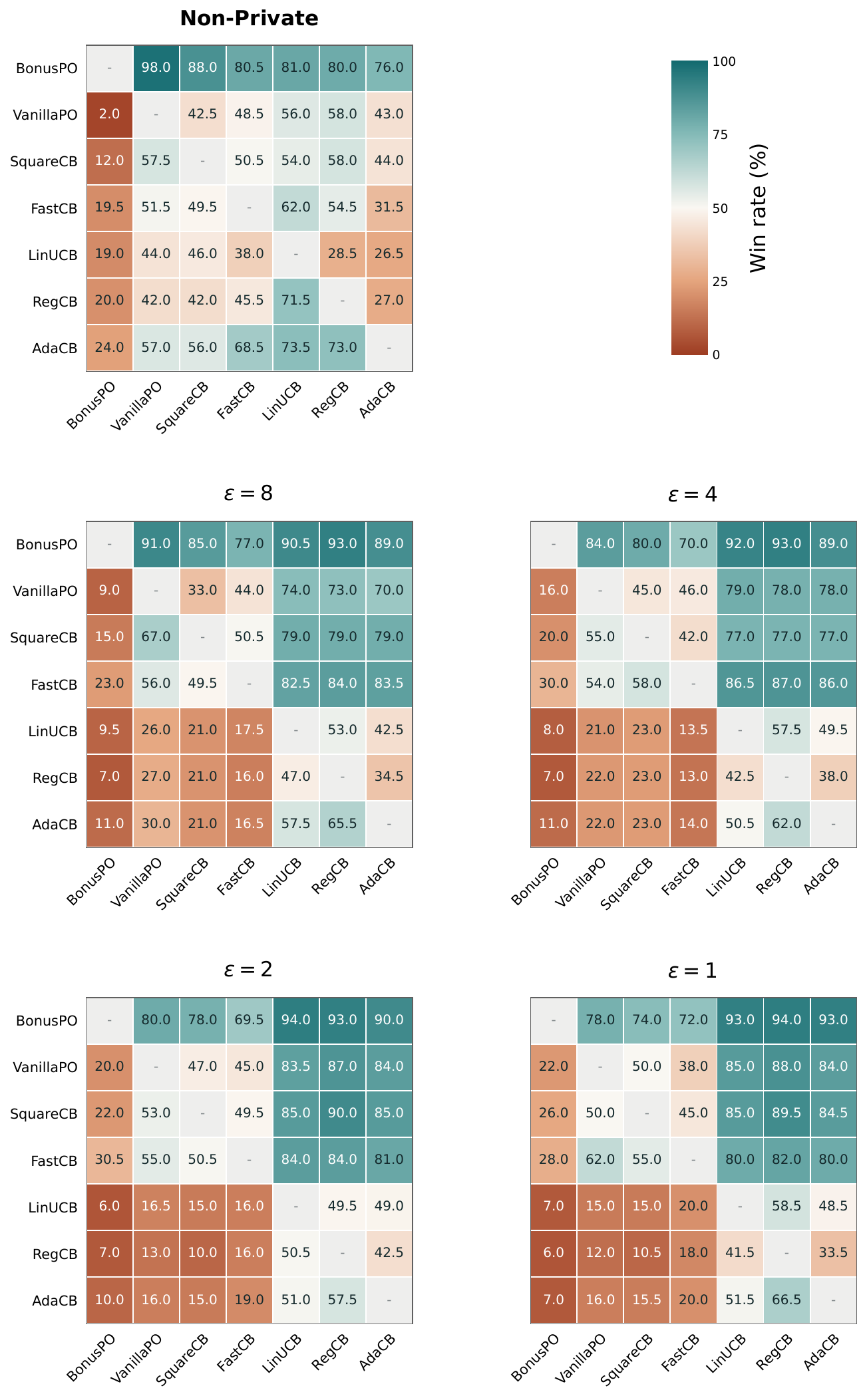}
    \caption{Pairwise win rates based on final progressive-validation loss over 100 \textsc{OpenML} datasets. Each entry reports the percentage of datasets on which the row algorithm achieves lower mean final progressive-validation loss than the column algorithm, with ties receiving half credit. We show the non-private setting and each private regime.}
    \label{fig:heatmaps}
\end{figure}

\paragraph{Utility--privacy trade-off.}
Figure~\ref{fig:loss-increase} shows the increase in final progressive-validation loss relative to $(8, 10^{-5})$-DP baseline as privacy becomes stronger. This gives a direct empirical view of the privacy--utility trade-off. As expected, stronger privacy leads to larger loss for all methods. The supervised baseline degrades the least. Among the bandit algorithms, the effect of privatization is much more similar.

\begin{figure}[h]
    \centering
    \includegraphics[width=0.9\linewidth]{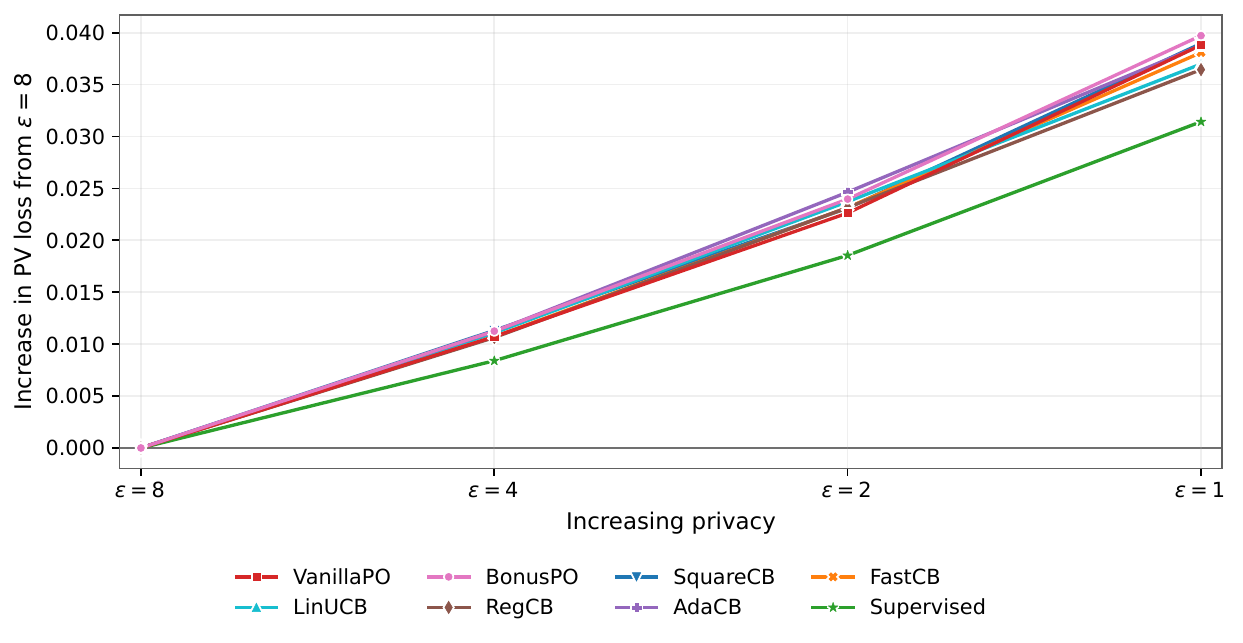}
    \caption{Increase in final progressive-validation loss relative to the $(8, 10^{-5})$-DP baseline, aggregated over 100 \textsc{OpenML} datasets.}
    \label{fig:loss-increase}
\end{figure}

Overall, the aggregate results support the main empirical conclusion of the appendix: \algname{BonusPO} is the strongest bandit method across privacy regimes, followed by a middle group consisting of \algname{VanillaPO}, \algname{FastCB}, and \algname{SquareCB}, while \algname{LinUCB}, \algname{RegCB}, and \algname{AdaCB} are consistently less competitive. Importantly, strengthening the privacy requirement does not substantially alter this relative ordering, consistent with the patterns in Figures~\ref{fig:heatmaps} and \ref{fig:loss-increase}.

\subsubsection{Results on selected datasets}
\label{app:selected-results}

Finally, we examine in greater detail the three \textsc{OpenML} datasets introduced in Section~\ref{sec:experiments}, 901, 1489, and 41027, together with three additional datasets, 1462, 1560, and 44743. Figures~\ref{fig:selected-datasets}, \ref{fig:more-datasets}, and \ref{fig:additional-datasets} show progressive-validation classification loss over rounds, averaged over 10 seeds. For each algorithm and privacy regime, we report the best configuration from the hyperparameter grid.

Several consistent trends emerge. First, performance degrades gracefully as privacy becomes stronger: on all six datasets, increasing privacy raises the loss but does not qualitatively change the learning behavior. Second, the supervised baseline performs best overall, while \algname{BonusPO} is the strongest bandit method and often comes closest to supervised performance.

The separation between methods is clearest on datasets 901, 1462, 41027, 44743. In the non-private regime, most algorithms are relatively close, but under stronger privacy they split into three groups: the supervised baseline, with \algname{BonusPO} closest to it; a middle group consisting of \algname{VanillaPO}, \algname{SquareCB}, and \algname{FastCB}; and a lower-performing group consisting of \algname{LinUCB}, \algname{RegCB}, and \algname{AdaCB}. Datasets 1489 and 1560 show the same pattern more weakly: the algorithms remain closer together across privacy levels, but the same ordering begins to emerge as privacy becomes more stringent.

\clearpage
\begin{figure}[p]
    \centering

    \begin{subfigure}{\linewidth}
        \centering
        \includegraphics[width=\linewidth]{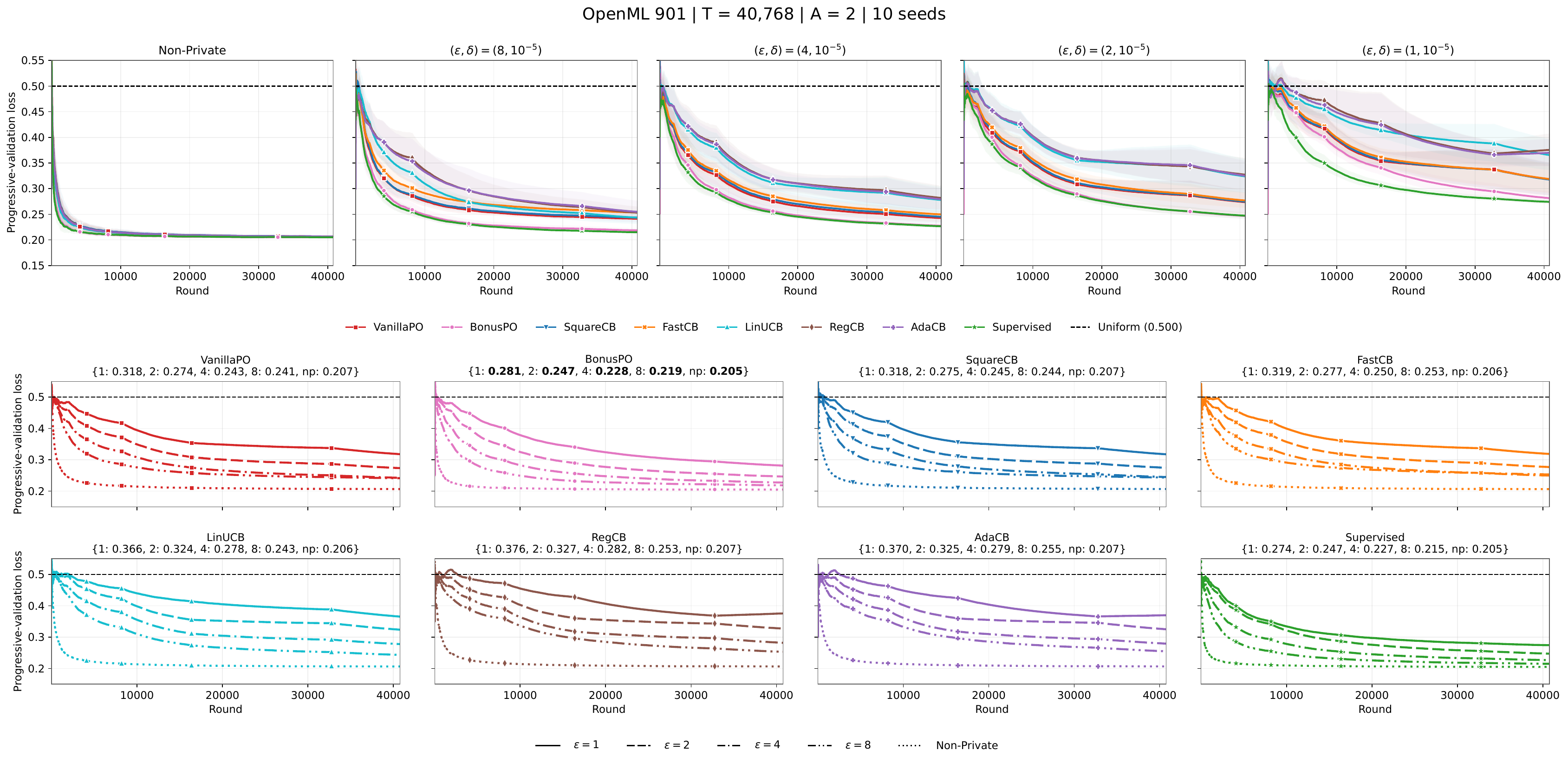}
        \caption{\textsc{OpenML} 901.}
        \label{fig:selected-901}
    \end{subfigure}

    \vspace{0.75em}

    \begin{subfigure}{\linewidth}
        \centering
        \includegraphics[width=\linewidth]{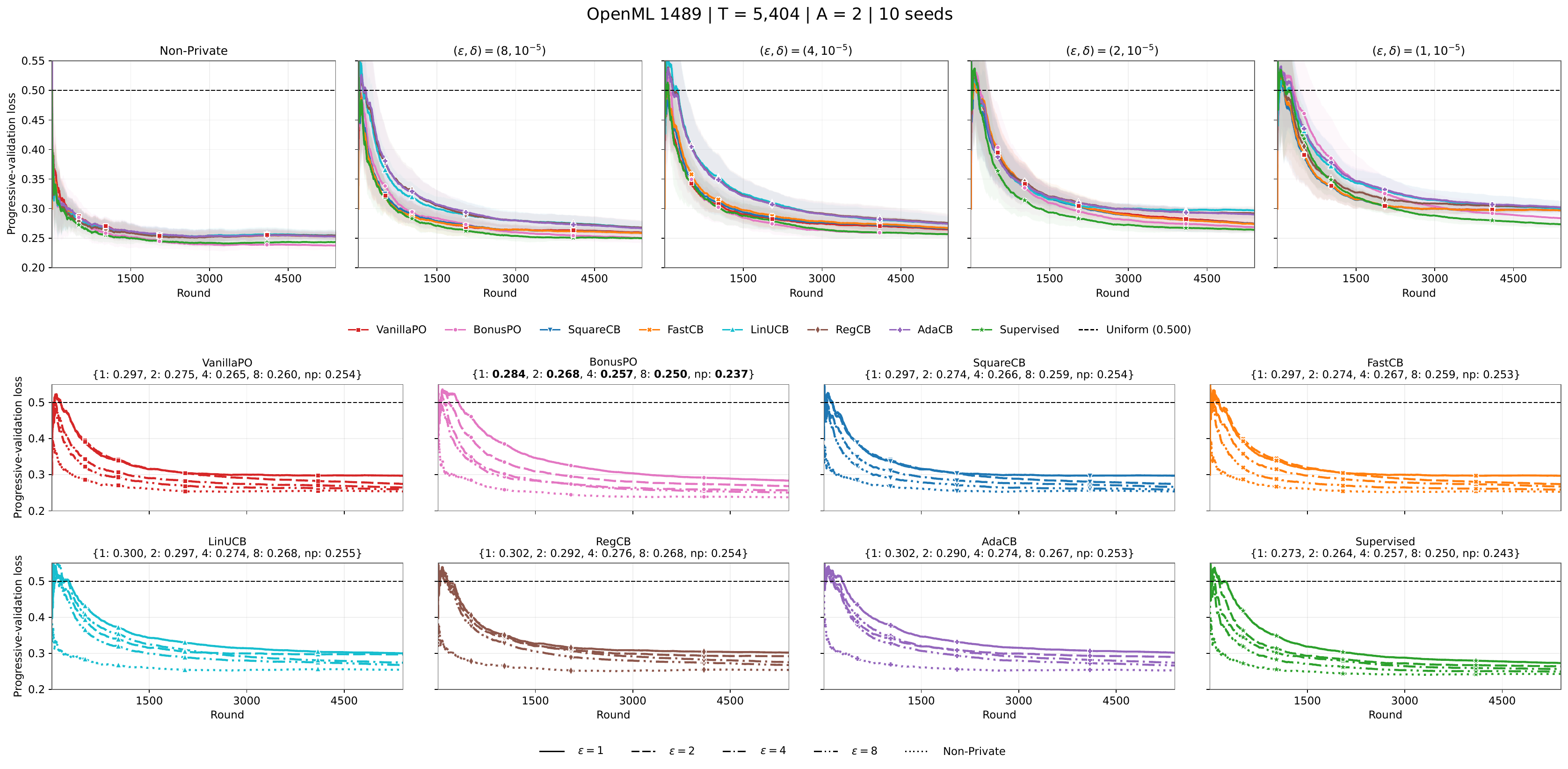}
        \caption{\textsc{OpenML} 1489.}
        \label{fig:selected-1489}
    \end{subfigure}

    \caption{Detailed comparison on \textsc{OpenML} datasets 901 and 1489 from Section~\ref{sec:experiments}. Curves show progressive-validation loss over rounds, averaged over 10 seeds, using the best hyperparameter configuration for each algorithm and privacy regime. Within each subfigure, the top row compares algorithms at fixed privacy regimes, while the remaining rows show the effect of privacy on each algorithm.}
    \label{fig:selected-datasets}
\end{figure}

\clearpage
\begin{figure}[p]
    \centering

    \begin{subfigure}{\linewidth}
        \centering
        \includegraphics[width=\linewidth]{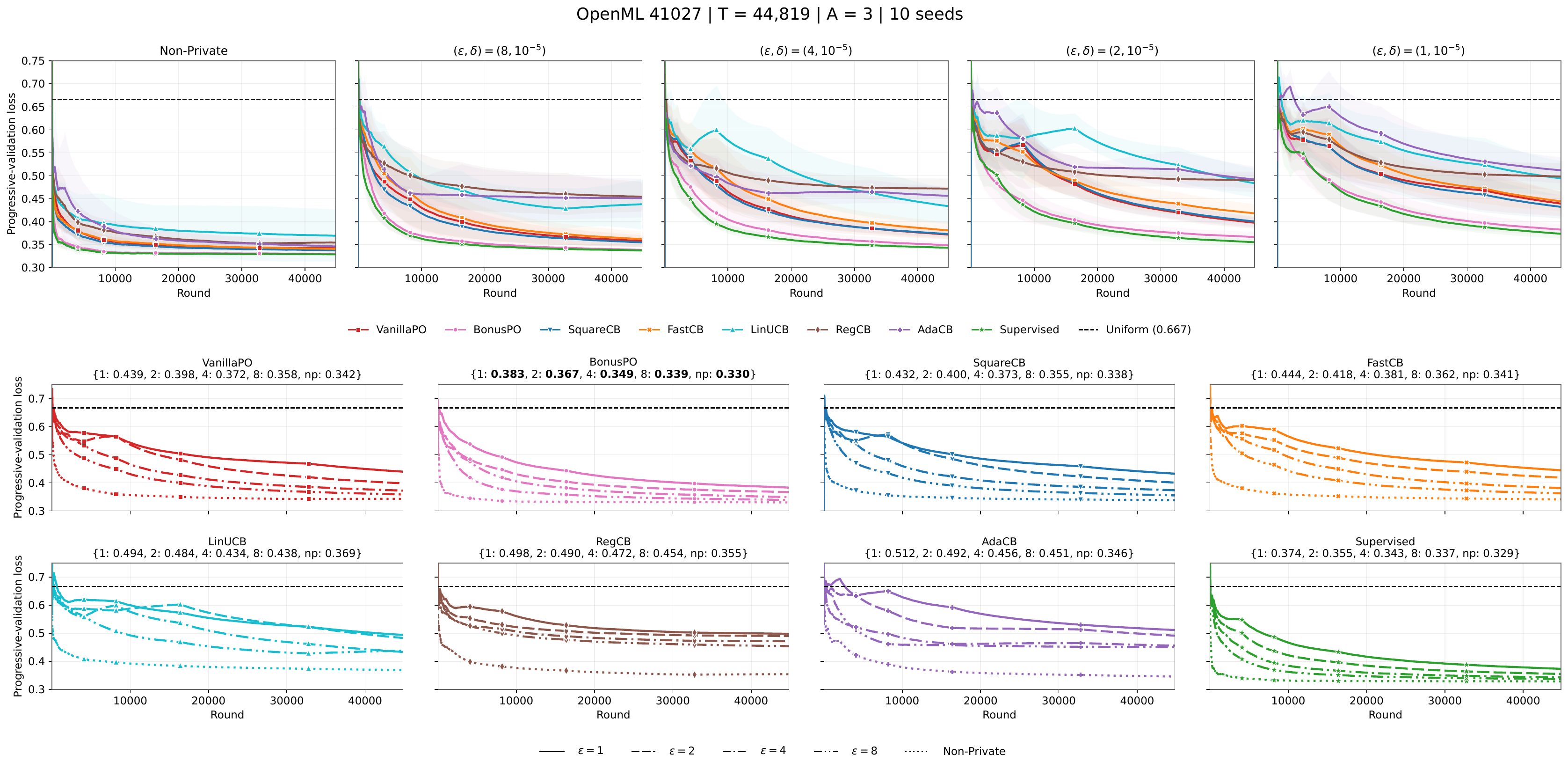}
        \caption{\textsc{OpenML} 41027.}
        \label{fig:selected-41027}
    \end{subfigure}

    \vspace{0.75em}

    \begin{subfigure}{\linewidth}
        \centering
        \includegraphics[width=\linewidth]{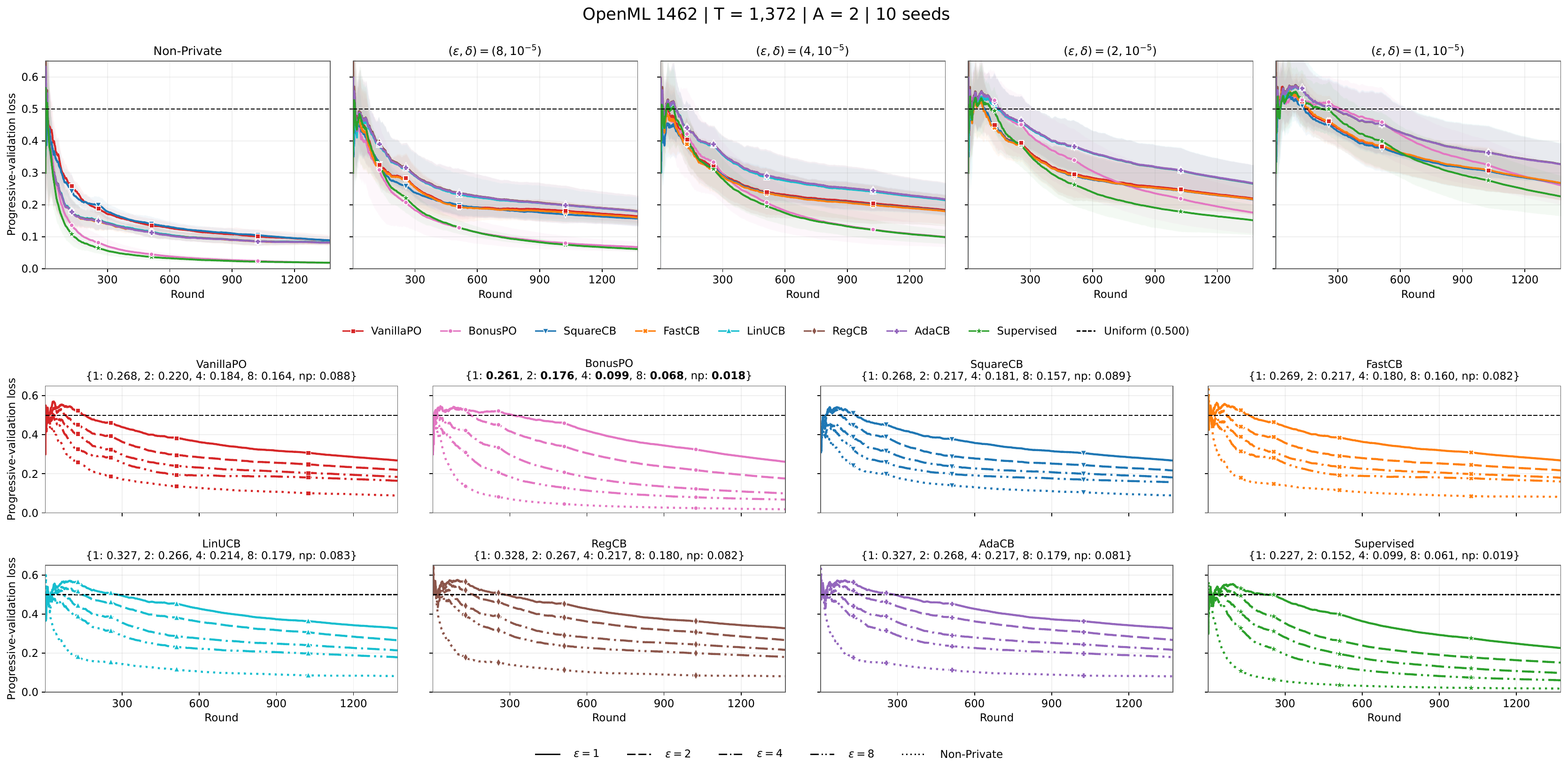}
        \caption{\textsc{OpenML} 1462.}
        \label{fig:selected-1462}
    \end{subfigure}

    \caption{Detailed comparison on \textsc{OpenML} dataset 41027 from Section~\ref{sec:experiments} and the additional \textsc{OpenML} dataset 1462. Curves show progressive-validation loss over rounds, averaged over 10 seeds, using the best hyperparameter configuration for each algorithm and privacy regime. The layout mirrors Figure~\ref{fig:selected-datasets}.}
    \label{fig:more-datasets}
\end{figure}

\clearpage
\begin{figure}[p]
    \centering

    \begin{subfigure}{\linewidth}
        \centering
        \includegraphics[width=\linewidth]{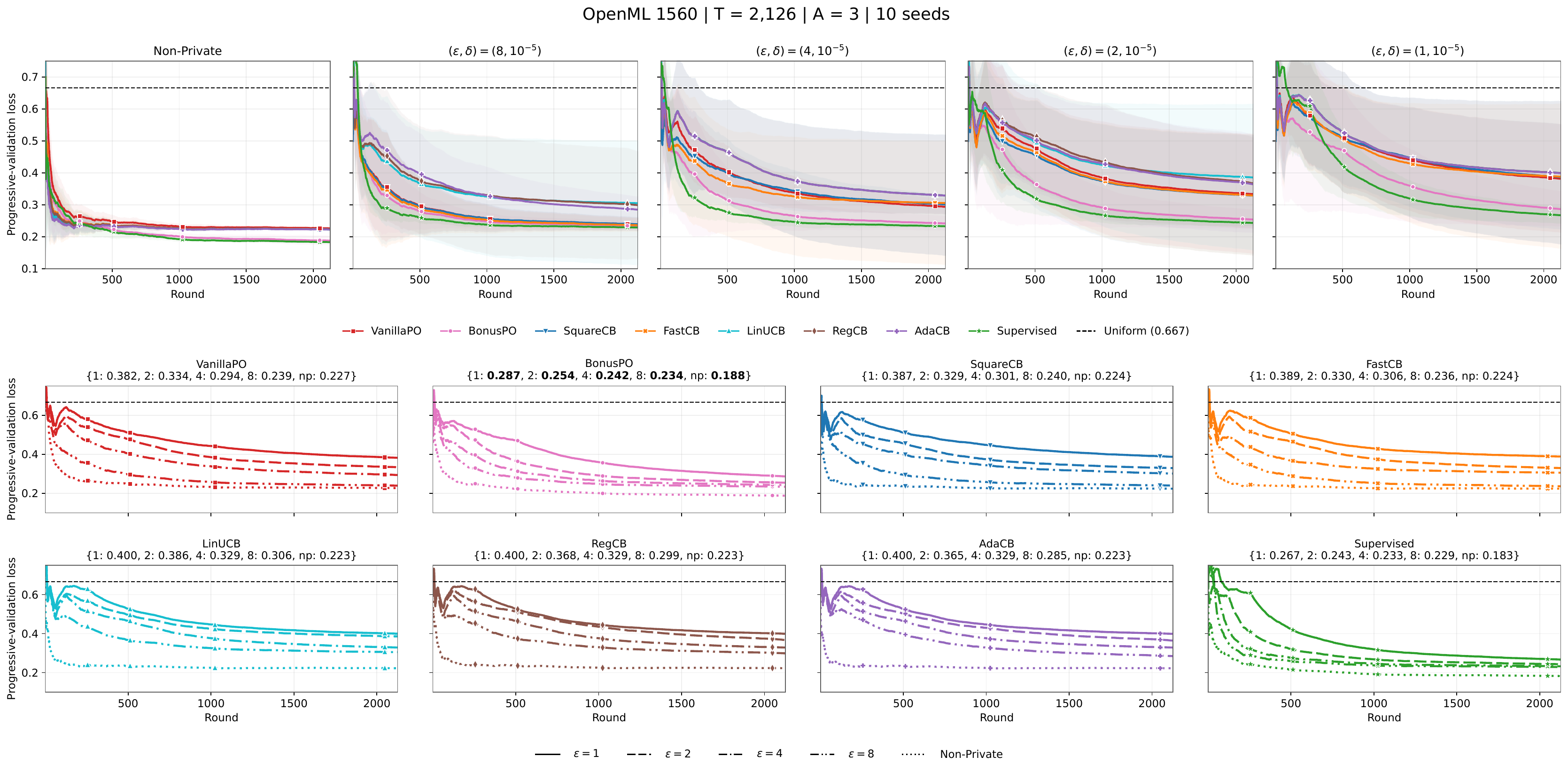}
        \caption{\textsc{OpenML} 1560.}
        \label{fig:selected-1560}
    \end{subfigure}

    \vspace{0.75em}

    \begin{subfigure}{\linewidth}
        \centering
        \includegraphics[width=\linewidth]{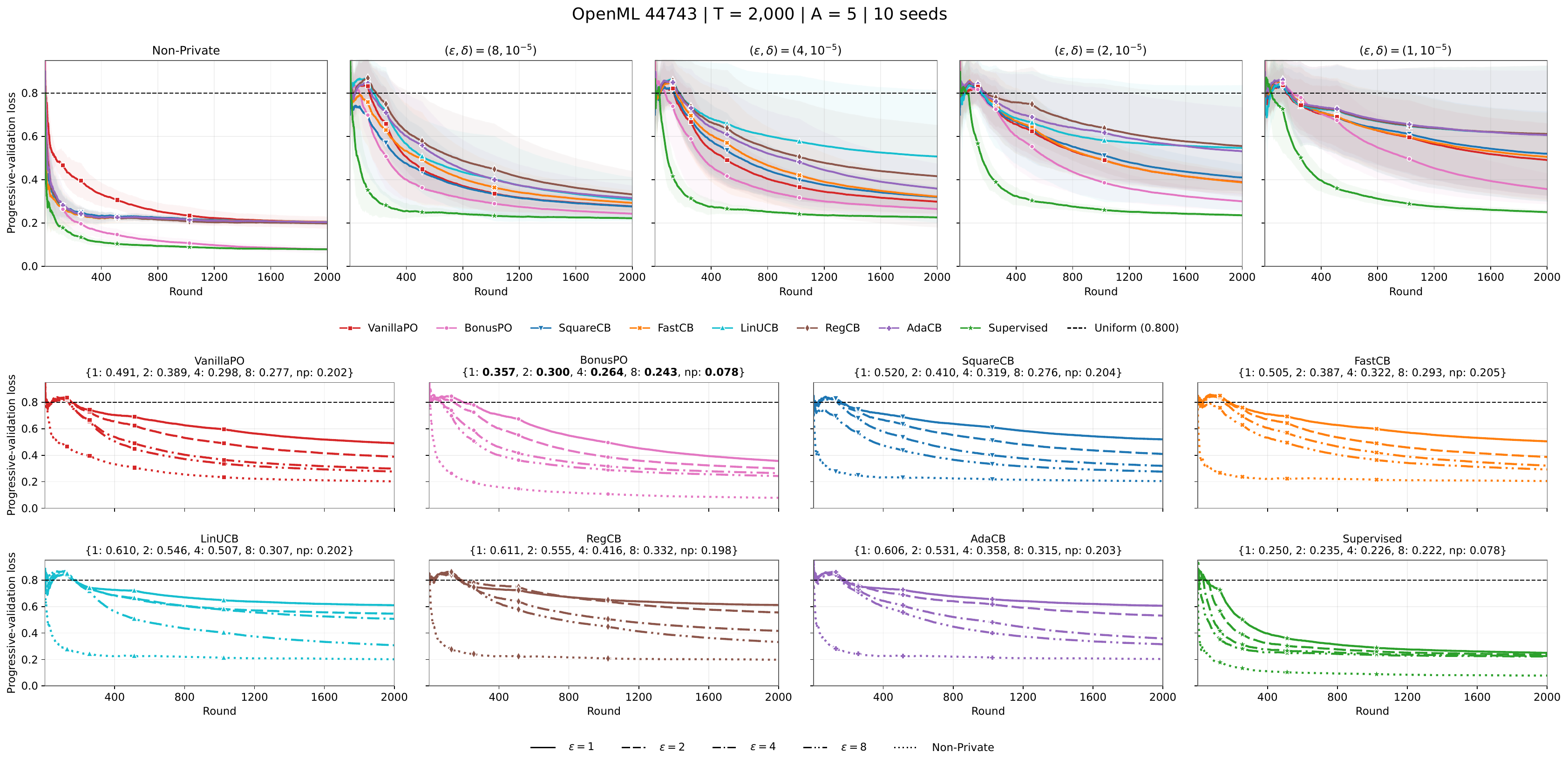}
        \caption{\textsc{OpenML} 44743.}
        \label{fig:selected-44743}
    \end{subfigure}

    \caption{Detailed comparison on two additional \textsc{OpenML} datasets, 1560 and 44743. Curves show progressive-validation loss over rounds, averaged over 10 seeds, using the best hyperparameter configuration for each algorithm and privacy regime. The layout mirrors Figure~\ref{fig:selected-datasets}.}
    \label{fig:additional-datasets}
\end{figure}
\clearpage

\end{document}